\documentclass[12pt]{article}

\usepackage{amsmath, amssymb, amsthm}
\usepackage{geometry}
\usepackage{fullpage}
\usepackage{graphicx}
\usepackage{booktabs}
\usepackage{multirow}
\usepackage{enumitem}
\usepackage{algorithm}
\usepackage{algpseudocode}
\usepackage{natbib}
\usepackage{hyperref}
\usepackage{bbm}
\usepackage{xcolor}
\usepackage{microtype}
\usepackage{placeins}
\hypersetup{
    hypertexnames=false,
    colorlinks=true,
    citecolor=blue,
    linkcolor=red,
    urlcolor=blue
}

\newtheorem{assumption}{Assumption}[section]
\newtheorem{theorem}{Theorem}[section]
\newtheorem{lemma}{Lemma}[section]
\newtheorem{remark}{Remark}[section]

\newtheorem{proposition}{Proposition}[section]
\newtheorem{example}{Example}[section]

\usepackage{amsmath,amsfonts,bm}

\def\1{\bm{1}}

\DeclareMathAlphabet{\mathsfit}{\encodingdefault}{\sfdefault}{m}{sl}
\SetMathAlphabet{\mathsfit}{bold}{\encodingdefault}{\sfdefault}{bx}{n}

\def\gA{{\mathcal{A}}}
\def\gB{{\mathcal{B}}}
\def\gC{{\mathcal{C}}}
\def\gD{{\mathcal{D}}}

\def\gH{{\mathcal{H}}}
\def\gI{{\mathcal{I}}}
\def\gJ{{\mathcal{J}}}

\def\gL{{\mathcal{L}}}
\def\gM{{\mathcal{M}}}

\def\gP{{\mathcal{P}}}
\def\gQ{{\mathcal{Q}}}
\def\gR{{\mathcal{R}}}

\def\gX{{\mathcal{X}}}
\def\gY{{\mathcal{Y}}}
\def\gZ{{\mathcal{Z}}}

\def\sB{{\mathbb{B}}}

\def\sP{{\mathbb{P}}}

\def\sR{{\mathbb{R}}}

\newcommand{\E}{\mathop{\mathbb{E}}}

\newcommand{\R}{\mathbb{R}}

\DeclareMathOperator*{\argmin}{arg\,min}

\newcommand{\LRs}[1]{\left(#1\right)}
\newcommand{\LRm}[1]{\left[#1\right]}
\newcommand{\LRl}[1]{\left\{#1\right\}}

\def\one{{\mathbbm{1}}}

\usepackage{url}

\title{Minimum Volume Optimal Transport for Multivariate Conformal Prediction}
\author{Chuchen Zhang\textsuperscript{1}, Hajie Ren\textsuperscript{2}, Zhaojun Wang\textsuperscript{1}, Yajie Bao\textsuperscript{3} and Changliang Zou\textsuperscript{1}\\[0.5em]
\small \textsuperscript{1}School of Statistics and Data Science, Nankai University\\
\small \textsuperscript{2}School of Mathematical Sciences,  Shanghai Jiao Tong University\\
\small \textsuperscript{3}School of Management, Fudan University}
\date{\today}

\begin{document}

\maketitle

\begin{abstract}
Optimal transport (OT) enables multivariate conformal prediction (CP) to construct flexible prediction sets that reflect the underlying complex distributional geometry while retaining distribution-free coverage guarantees. Existing OT-based methods typically rely on a fixed Euclidean transport cost, which can be suboptimal in terms of the volume of prediction sets. We propose Minimum-Volume OTCP, a framework that learns a Mahalanobis ground cost to directly minimize prediction-region volume over a class of positive-definite metrics. Using a metric Brenier representation, we derive a volume-based objective that depends explicitly on the  Mahalanobis cost matrix and optimize it jointly with the transport map. We establish finite-sample marginal coverage of the output prediction sets and optimality of volume metrics. We further develop an individualized extension that learns covariate-dependent transport metrics, adapting the shape and orientation of prediction sets to heterogeneous conditional distributions. For this extension, we also show the asymptotic conditional coverage of our method. Simulations and real-data experiments demonstrate substantial reductions in prediction-set volume.
\end{abstract}

\section{Introduction}
\label{sec:introduction}
Conformal prediction (CP) constructs finite-sample, distribution-free prediction sets under exchangeability \citep{vovk2005algorithmic}. Suppose we have the labeled data $\{(X_i,Y_i)\}_{i=1}^n \subset \gX \times \gY$ and a new covariate $X_{n+1}\in \gX$, whose label $Y_{n+1} \in \gY$ is unobserved. The goal is to construct a prediction set $\gC(X_{n+1})$ based on the labeled data, satisfying the following coverage guarantee
\begin{align*}
    \sP\{Y_{n+1}\in \gC(X_{n+1})\}
    \ge
    1-\alpha,
\end{align*}
where $\alpha \in (0,1)$ is the nominal coverage level.

A central challenge in multivariate CP is to construct prediction sets that adapt to the geometry of the response distribution. Structured conformity scores offer simple and interpretable constructions, but their geometric flexibility is limited by the chosen score class \citep{johnstone2021conformal,messoudi2022ellipsoidal,braun2025multivariate}. Capturing more complex multivariate uncertainty therefore calls for flexible distributional representations \citep{feldman2023calibrated,Fang2025contra,dheur2025unified}. Optimal transport (OT) offers such a representation by connecting the response distribution to a simple reference distribution through transport maps. These maps provide a basis for defining multivariate ranks and quantiles \citep{chernozhukov2017monge,hallin2021distribution}, allowing prediction sets to be constructed according to the geometry of the response distribution.

Generally, conformal prediction for a multivariate response is often equipped with a multivariate score function $s(\cdot,\cdot)$. Considering the example $\gY\subseteq \sR^d$, given a pretrained predictor $\widehat\mu(\cdot):\gX \to \sR^d$, we use the multivariate residual score $s(x,y) = y - \widehat\mu(x)$. Let $\gP$ denote the distribution of the multivariate score $s(X_i,Y_i)$, and let $\gQ$ be a radial reference distribution, e.g., the uniform distribution on the unit ball of $\sR^d$. Existing OT-based conformal methods typically transport the score distribution $\gP$ to $\gQ$ under the squared Euclidean ground cost, yielding the transport map
\begin{equation}
    T_I
    \in
    \argmin_{T_\#\gP=\gQ}
    \int \|s-T(s)\|^2\,d\gP(s).
    \label{eq:intro_euclidean_transport}
\end{equation}
The OT conformal prediction set takes the form $\gC_{I}(X_{n+1}) = \{y\in \gY: \|T_{I}(s(X_{n+1},y))\| \le \hat{q}_{I}\}$, where $\hat{q}_I$ is the empirical quantile of $\{\|T_I(s(X_i,Y_i))\|\}_{i=1}^n$.
However, although OT can produce highly flexible prediction sets, the transport geometry induced by the squared Euclidean ground cost need not be well suited to achieving a small-volume prediction set.

To address this limitation, we introduce \emph{Minimum-Volume Optimal Transport Conformal Prediction} (MV-OTCP), which learns the OT ground metric to minimize prediction-set volume. 
For a positive-definite matrix $M$, we replace the squared Euclidean ground cost in \eqref{eq:intro_euclidean_transport} with $c_M(s,u)=(s-u)^\top M(s-u)$, obtaining a transport map $T_M$. Applying the same conformal calibration with this map yields the corresponding prediction set, denoted by $\gC_M(X_{n+1})$.
Our target metric minimizes the population volume objective:
\[
M^\star \in \argmin_{M} \E\left[\mathsf{Leb}\{\gC_M(X_{n+1})\}\right],
\]
where $\mathsf{Leb}$ denotes the Lebesgue measure.
Because $M$ enters the transport cost, learning $M$ changes the transport map itself and hence the geometry of the resulting prediction region.
Our contributions are threefold:
\begin{itemize}[leftmargin=2em,itemsep=2pt,topsep=2pt]

\item We introduce MV-OTCP for learning the ground metric of OT-based conformal prediction sets. To account for heterogeneous conditional distributions, we develop an individualized extension, IMV-OTCP, which learns a covariate-dependent metric so that prediction-set shape and orientation can adapt locally to the covariate. 

\item Using a metric Brenier representation, we derive a Jacobian-based expression for prediction-set volume, yielding a tractable objective for metric learning. We establish finite-sample marginal coverage, consistency of the learned metric, and asymptotic agreement between the learned and oracle prediction regions. Under additional high-level conditions, we also establish asymptotic conditional coverage for IMV-OTCP.

\item Experiments on synthetic and real data demonstrate improved prediction-set volume efficiency relative to existing conformal baselines while maintaining the target coverage.

\end{itemize}

\section{Related work}
\label{sec:related}
We focus here on the two directions most closely related to our work, and refer readers to \citet{vovk2005algorithmic,angelopoulos2024theoretical} for broader reviews of conformal prediction.

\textbf{Multivariate conformal prediction.}
A growing literature develops conformal prediction methods for multivariate responses using different scores and representation-based constructions \citep{johnstone2021conformal,messoudi2022ellipsoidal,izbicki2022cd,feldman2023calibrated,Fang2025contra,dheur2025unified,braun2025multivariate}.
Within the OTCP framework, SOTCP \citep{thurin2025optimal,klein2025multivariate} uses a transport map learned from an independent sample within split conformal calibration, whereas FOTCP \citep{ndiaye2025beyond} incorporates each candidate test score into the transport assignment to obtain a full-conformal construction. Conditional neural OT \citep{kondratyev2026neural} learns covariate-dependent vector-rank maps whose radial norms yield split-conformal pullback regions. Our work retains this OT-based ranking principle but treats the ground geometry as a learnable component rather than fixing the squared Euclidean cost.

\textbf{Minimum-volume conformal geometry.}
Conformal efficiency has been improved through score selection or aggregation and through direct adaptation of prediction-set shape \citep{yang2024selection,liang2024conformal,xu2026aggregating,bao2026shape}. \citet{kiyani2024length} study length-optimal prediction intervals in the univariate setting through a minimax formulation. Most directly related to our multivariate minimum-volume formulation, the MVCS method of \citet{braun2025minimum} explicitly minimizes volume within a structured norm-based parametric class. Its multi-norm extension provides additional flexibility by allowing different norm-based regions over a prespecified partition. However, the partition must be specified in advance, and the resulting construction remains restricted to norm-based geometries within each partition element. Our prediction sets are inverse images of a radial reference ball under a generally nonlinear optimal transport map, allowing asymmetric and nonconvex geometries without requiring a prespecified partition or restricting the regions to norm-based shapes. Our extension to more general downstream task losses is also related to work that directly optimizes decision losses of conformal prediction sets \citep{yeh2024end,bao2025optimal}.



\section{Minimum-volume OTCP}
\label{sec:tacp}


\subsection{Recap of OT-based conformal prediction}
\label{subsec:otcp_recap}

Let $\gD_{\rm cal}=\{(X_i,Y_i)\}_{i=1}^n$ be a calibration sample independent of the pretrained predictor, and let $U_1,\ldots,U_{n+1}$ be i.i.d. draws from reference distribution $\gQ$. For a trial response $y$ at $X_{n+1}$, write $S_i^y=s(X_i,Y_i)$ for $i\in[n]$ and $S_{n+1}^y=s(X_{n+1},y)$. At the empirical level, the optimal transport plan \eqref{eq:intro_euclidean_transport} between these two discrete measures is a classic assignment problem. FOTCP then solves
\begin{align*}
    \widehat\sigma^y
    \in
    \argmin_{\sigma\in\mathfrak S_{n+1}}
    \sum_{i=1}^{n+1}\|S_i^y-U_{\sigma(i)}\|^2,
\end{align*}
where $\mathfrak S_{n+1}$ is the set of all permutations of $[n+1]$.  Let $\widehat q_{1-\alpha}(\{a_i\}_{i=1}^n)$ be the $(1-\alpha)$-quantile of the empirical distribution of $\{a_i\}_{i=1}^n$. Then the FOTCP prediction set is 
\begin{align*}
    \widehat\gC^{\rm FOTCP}(X_{n+1})=\LRl{y\in\gY:\|U_{\widehat\sigma^y(n+1)}\|\le \widehat q_{1-\alpha}\LRs{\{\|U_i\|\}_{i=1}^{n+1}}}.
\end{align*}
Under exchangeability of $\{(X_{i},Y_{i})\}_{i=1}^{n+1}$, we have $\sP\{Y_{n+1}\in \widehat\gC^{\rm FOTCP}(X_{n+1})\} \ge 1-\alpha$.

\subsection{Population volume metric of OTCP}
\label{subsec:volume_objective}

The standard squared Euclidean transport cost treats all directions in the score space equally, which need not be aligned with the criterion used to evaluate the resulting conformal set.
We therefore introduce a volume-adaptive geometry into the transport problem. 
Specifically, we replace the Euclidean cost by the Mahalanobis-type quadratic cost
\begin{align}
\label{eq:metric_cost}
  c_M(s,u)=(s-u)^\top M(s-u),\qquad M=M^\top\succ0,
\end{align}
The matrix $M$ allows the transport geometry to reweight and rotate different score directions according to the downstream objective, while preserving the quadratic structure of the transport problem.
For each fixed $M$, let $T_M$ denote the population optimal transport map from the score distribution $\gP$ to the reference distribution $\gQ$ under the cost $c_M$ , namely,
\begin{align}
\label{eq:metric_monge_problem}
T_M\in\argmin_{T_\#\gP=\gQ}\int c_M\{s,T(s)\}\,d\gP(s),
\end{align}
Let $F_Q(t)=\gQ\{u:\|u\|\le t\}$, $q_\star=F_Q^{-1}(1-\alpha)$ and consider residual score $s(x,y) = y - \widehat\mu(x)$.
The {population oracle prediction set} can be written as $\gC_M(x)=\widehat\mu(x)+T_M^{-1}(\sB_{1-\alpha})$, where $\sB_{1-\alpha}=\{u:\|u\|\le q_\star\}$. Since translation does not change Lebesgue measure, the volume of $\gC_M(x)$ is determined entirely by the geometry induced by $M$. Thus we define the population volume target by
\begin{equation}
 M^\star\in\argmin_{M\in\gM_0}\gJ(M),\qquad
 \gJ(M)=\mathsf{Leb}\{\gC_M(x)\}=\mathsf{Leb}\{T_M^{-1}(\sB_{1-\alpha})\},\nonumber
\end{equation}
where $\gM_0=
    \{
        M=M^\top\succ0:
        \lambda_{\min}(M)\ge\underline\lambda,\ 
        \lambda_{\max}(M)\le\overline\lambda,\
        \det(M)=1
    \}$.
The constraint $\det(M)=1$ resolves the scale non-identifiability of $M$, while the eigenvalue bounds ensure that the resulting parameter space is compact.
The dependence of $\gJ(M)$ on $M$ is implicit through the transport map. The following metric Brenier representation makes the geometry explicit.

\begin{lemma}[Metric Brenier representation]
\label{lem:metric_brenier}
Suppose $\gP$ and $\gQ$ have finite second moments and $\gP$ is absolutely continuous. Then \eqref{eq:metric_monge_problem} has a unique solution up to $\gP$-null sets, with $M^{1/2}T_M(s)=\nabla\psi(M^{1/2}s)$ for a convex $\psi$. If $\gQ$ is also absolutely continuous, the reverse map $G_M$ exists, is the almost-everywhere inverse of $T_M$, and has the representation
$G_M(u)=M^{-1/2}\nabla g_M(M^{1/2}u)$
for a convex potential $g_M$.
\end{lemma}

Under the condition of Lemma~\ref{lem:metric_brenier}, we can rewrite $\gC_M(x)=\widehat\mu(x)+G_M(\sB_{1-\alpha})$. Assuming that $g_M$ is twice differentiable, Lemma~\ref{lem:metric_brenier} and a change-of-variables argument yield
\begin{align}
\label{eq:task_loss_hessian_form}
    \gJ(M)
    =
    \E_{U}\LRm{\{f_Q(U)\}^{-1}\one\{U\in\sB_{1-\alpha}\}\det\{\nabla^2g_M(M^{\frac{1}{2}}U)\}}.
\end{align}
Here, $f_Q$ denotes the density of the reference distribution. 
This representation is useful because it converts the volume loss into an optimization problem over a metric and a convex transport potential. The coverage level is handled by integrating over the fixed reference region $\sB_{1-\alpha}$, while the geometry of the prediction set enters through the Hessian determinant.
In computation, the outer expectation over $U$ can be approximated using independent reference samples, while the transport map can be learned using any suitable parametric model, including neural parameterizations of Brenier maps \citep{makkuva2020optimal,chaudhari2025gradnetot}.

\begin{example}[Explicit volume-optimal metric]
\label{ex:elliptical_volume}
Let $\gQ$ be a radial reference distribution on the unit ball, represented as $U=R\Theta$, where $R\sim\operatorname{Unif}(0,1)$ and $\Theta$ is uniformly distributed on the unit sphere. Let $\gP$ have density
$f_P(s)=|\Sigma|^{-1/2}g(\|\Sigma^{-1/2}(s-m)\|)$, where $g$ is nonincreasing. Writing $r(s)=\|\Sigma^{-1/2}(s-m)\|$ and $F_\rho$ for the distribution of $r(S)$, the map
$T_{M^\star}(s)=F_\rho\{r(s)\}\frac{\Sigma^{-1/2}(s-m)}{r(s)}$
induces the prediction set $T_{M^\star}^{-1}(\sB_{1-\alpha}) =\{s:r(s)\le F_\rho^{-1}(1-\alpha)\}$. This set is the central density-superlevel ellipsoid and is therefore the minimum-volume measurable set with probability $1-\alpha$. Moreover, by Lemma~\ref{lem:metric_brenier}, $T_{M^\star}$ is generated by the metric $M_{\rm vol}^\star\propto \Sigma^{-1/2}$. See Figure~\ref{fig:metric_ellipsoid_polyhedral} for an illustration and Appendix~\ref{app:elliptical_derivation} for derivation details.
\end{example}

\begin{figure}[ht]
    \centering
    \includegraphics[width=\linewidth]{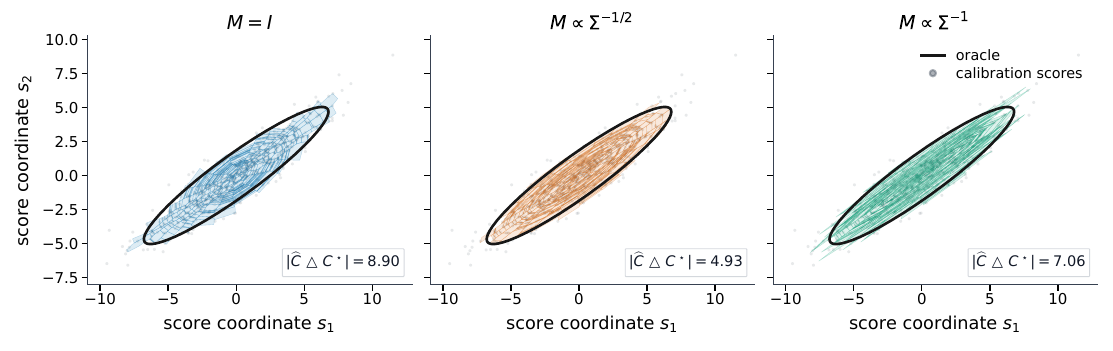}
    \caption{Metric choices on an elliptical score distribution. The accepted assignment regions are finite unions of polyhedral cells (Appendix~\ref{subsec:polyhedral_represent}) under $M=I$, $M\propto\Sigma^{-1/2}$, and $M\propto\Sigma^{-1}$. Among the three choices, $M\propto\Sigma^{-1/2}$ yields the smallest symmetric difference from the oracle region.}
    \label{fig:metric_ellipsoid_polyhedral}
\end{figure}

\paragraph{General task-aware OTCP.}
The same transport representation also permits a downstream task to determine which parts of the response space are costly to include. 
We consider an averaged robust optimization problem. For a decision $z\in\gZ$ and an integrable loss density $\ell(z,y)$, define the decision loss over a prediction set $\gC(x)$ as $\gL(z,\gC(x))=\int_{\gC(x)}\ell(z,y)\,dy$. This formulation encompasses the optimization problems considered in \citet{bertsimas2023robust,cortes-gomez2025utilitydirected}.
The corresponding task-aware target for metric learning is
\begin{equation}
 M_{\rm task}^\star\in\argmin_{M\in\gM_0}\gJ_{\rm task}(M),\qquad
 \gJ_{\rm task}(M)=\E_X\left[\gL\{z_M^\star(X),\gC_M(X)\}\right],\nonumber
\end{equation}
where $z_M^\star(x)\in\argmin_{z\in\gZ}\gL(z,\gC_M(x))$. In line \eqref{eq:task_loss_hessian_form}, the general task target can be rewritten as
\begin{equation}
\label{eq:general_task_hessian_form}
 \gJ_{\rm task}(M)=\E_{X,U}\left[
 \ell\{z_M^\star(X),\widehat\mu(X)+G_M(U)\}\cdot \{f_Q(U)\}^{-1}\one\{U\in\sB_{1-\alpha}\}\det\{\nabla G_M(U)\}\right].
\end{equation}
Volume is recovered as the special case $\ell\equiv1$, whereas a general task loss additionally accounts for where the accepted volume is placed in the response space; several examples are given in Appendix~\ref{sec:general_task_loss}.
The resulting metric is then used to construct the final conformal prediction set. Further implementation details are provided in Appendix~\ref{sec:general_task_loss}.

\subsection{Metric estimation and conformal calibration}
\label{subsec:empirical_construction}

Let $\gD_{\rm est}=\{(X_i,Y_i)\}_{i=1}^N$ be an independent estimation sample, and let $U_1,\ldots,U_N\sim\gQ$ be independent reference draws. 
For a fixed $M\in\gM_0$, let $\Psi_\theta$ be a parameterized approximation of the Brenier map $\nabla g_M(\cdot)$.
Then the corresponding empirical version of \eqref{eq:task_loss_hessian_form} is
\begin{align}
\label{eq:empirical_volume_objective}
    \widehat \gJ(M,\theta)
    =
    \frac{1}{N}\sum_{i=1}^N \{f_Q(U_i)\}^{-1}\one\{U_i\in\sB_{1-\alpha}\}
        \det\{\nabla\Psi_\theta(M^{1/2}U_i)+\delta_N I_{{d}}\},
\end{align}
where $\delta_N>0$ is a small numerical stabilizer. The empirical minimum-volume metric-learning problem is
\begin{align*}
    {\widehat M}\in\argmin_{M\in\gM_0}
    \widehat\gJ(M,\widehat\theta(M)),
\end{align*}
where $\widehat\theta(M)$ is the output of an OT-fitting subroutine at metric $M$. Since the fitted transport parameters $\widehat\theta(M)$ themselves depend on $M$, the gradient of the profiled objective must account for this dependence. Therefore we do so via implicit differentiation through the inner OT-fitting problem. Additional implementation details are provided in Appendix~\ref{subsec:tacp_implementation}.

To form the final conformal set, hold $\widehat M$ fixed and use a fresh calibration sample $\gD_{\rm cal}=\{(X_i,Y_i)\}_{i=1}^n$. For a trial response $y$ at a new covariate $x$, compute the assignment
\begin{align}
\label{eq:final_assignment}
    \widehat\sigma_{\widehat M}^{\,y}
    \in
    \argmin_{\sigma\in\mathfrak S_{n+1}}
    \sum_{i=1}^{n+1}
    (S_i^y-U_{\sigma(i)})^\top
    \widehat M
    (S_i^y-U_{\sigma(i)}).
\end{align}
The resulting Minimum Volume OT Conformal Prediction (MV-OTCP) set is
\begin{align*}
    {\widehat\gC}^{\rm MV}(x)
    =
    \left\{
        y\in\gY:
        \|U_{\widehat\sigma_{\widehat M}^{\,y}(n+1)}\|
        \le
        \widehat q_{1-\alpha}\LRs{\{\|U_i\|\}_{i=1}^{n+1}}
    \right\}.
\end{align*}

\begin{theorem}[Marginal coverage for MV-OTCP]
\label{thm:tacp_validity}
Assume that $\{(X_i,Y_i)\}_{i=1}^{n+1}$ are exchangeable, that $\widehat M$ is learned independently of the final calibration sample and test response, and that ties in \eqref{eq:final_assignment} are resolved permutation-equivariantly. Then $\sP\{Y_{n+1}\in {\widehat\gC}^{\rm MV}(X_{n+1})\}
    \ge
    1-\alpha .$
\end{theorem}

Similar to FOTCP \citep{ndiaye2025beyond}, for every fixed metric $M=M^\top\succ0$, the empirical set has a polyhedral representation. See Appendix~\ref{subsec:polyhedral_represent} for additional details. 

\subsection{Theoretical results of MV-OTCP}
For notational simplicity, define
\[
\zeta_M(u)
    =
    {
        {\{f_Q(u)\}^{-1}}\one\{u\in\sB_{1-\alpha}\}
        \det \{\nabla^2g_M(M^{\frac{1}{2}}u)\}
    }.
\]
Then $\gJ(M)=\E\{\zeta_M(U)\}$.
Write $\gJ_N(M)=N^{-1}\sum_{i=1}^N\zeta_M(U_i)$, and $\widehat\gJ(M)=\widehat\gJ\{M,\widehat\theta(M)\}$.  Thus $\gJ_N$ is the infeasible empirical average of the population integrand, while $\widehat\gJ$ is the actual plug-in criterion used for metric learning. For $M_c\in\gM_0$, let $\gM_0(M_c,\epsilon)=\{M\in\gM_0:\|M-M_c\|_F\le\epsilon\}$.

\begin{assumption}
\label{ass:tacp_metric_learning}
(i) There exist $B_\zeta,L_\zeta\in(0,\infty)$ such that $|\zeta_M(U)|\le B_\zeta$ a.s. for all $M\in\gM_0$. Moreover, $\sup_{M_c\in\gM_0}\E_U\left[\sup_{M\in\gM_0(M_c,\epsilon)}|\zeta_M(U)-\zeta_{M_c}(U)|\right]\le L_\zeta\epsilon$, for any $\epsilon>0$. (ii) $\mathcal E_N^{\rm obj}:=\sup_{M\in\gM_0}|\widehat\gJ(M)-\gJ_N(M)|$
satisfies $\mathcal E_N^{\rm obj}=o_{\sP}(1)$ as $N\to\infty$. (iii) For every $\epsilon>0$,
\(
    \inf_{\substack{M\in\gM_0:\ \|M-M^\star\|_F\ge\epsilon}}
    \{\gJ(M)-\gJ(M^\star)\}>0.
\)
\end{assumption}

Assumption~\ref{ass:tacp_metric_learning}(i) imposes boundedness and local Lipschitz regularity on the population objective integrand. Part~(ii) isolates the plug-in transport error induced by $\Psi_{\widehat\theta}$ entering metric learning. Part~(iii) is a standard separation condition ensuring that metrics bounded away from $M^\star$ have strictly larger population objective values.
Appendix~\ref{app:proof_tacp_uniform_deviation} gives primitive sufficient conditions for Assumption~\ref{ass:tacp_metric_learning}(i).

\begin{theorem}
\label{thm:tacp_uniform_deviation}
Suppose Assumption~\ref{ass:tacp_metric_learning}(i) holds. For fixed ${d}$, there exists a universal constant $C>0$ such that, for every $\delta\in(0,1)$, with probability at least $1-\delta$,
\begin{align*}
    \sup_{M\in\gM_0}
    |\widehat\gJ(M)-\gJ(M)|
    \le
    \mathcal E_N^{\rm obj}
    +
    C B_\zeta
    \sqrt{
        \frac{
            {d}^2\log\{e+2\sqrt{{d}}\,\overline\lambda L_\zeta N/B_\zeta\}
            +\log(2/\delta)
        }{N}
    }.
\end{align*}
In addition, if parts~(ii)--(iii) of Assumption~\ref{ass:tacp_metric_learning} hold, then $\E\|\widehat M-M^\star\|_F\to0$
\end{theorem}
Theorem~\ref{thm:tacp_uniform_deviation} decomposes the error of the learned volume objective into the transport approximation error $\mathcal E_N^{\rm obj}$ and a uniform empirical fluctuation term. The latter vanishes at the usual $N^{-1/2}$ rate up to logarithmic factors. Hence, if the fitted transport is uniformly consistent and the population objective is well separated around $M^\star$, the learned metric $\widehat M$ converges to the population-optimal metric.

The final conformal geometry depends on both $\widehat M$ and the empirical assignment map. Define the oracle set $\gC_\star^{\rm MV}(x)=\gC_{M^\star}(x)$. We now investigate the asymptotic agreement between the oracle set $\gC_\star^{\rm MV}(x)$ and empirical set ${\widehat\gC}^{\rm MV}(x)$.

\begin{assumption}
\label{ass:tacp_assignment_stability}
Let $S\sim\gP$ be independent of the calibration and reference data. (i) $\gQ$ is absolutely continuous with compact convex support, and $F_Q$ is continuously differentiable near $q_\star$ with $F_Q'(q_\star)>0$. (ii) $\gP$ is absolutely continuous with compact support, and there exists $L_{\rm OT}<\infty$ such that $\|T_M(s)-T_M(s')\|\le L_{\rm OT}\|s-s'\|$ for all $M\in\gM_0$ and $s,s'\in\operatorname{supp}(\gP)$. (iii) There exists $L_T<\infty$ such that $\E_S\|T_M(S)-T_{M^\star}(S)\|\le L_T\|M-M^\star\|_F$ for all $M\in\gM_0$.




\end{assumption}
Assumption~\ref{ass:tacp_assignment_stability}~(i)--(ii) impose standard smoothness and compactness conditions on the reference and score distributions, together with a uniform Lipschitz condition on the population transport map; conditions of this type are used in \citet{deb2021rates}. Assumption~\ref{ass:tacp_assignment_stability}~(iii) further requires the transport map to vary smoothly with the ground metric, ensuring that perturbations of $M$ induce only proportional changes in the transported scores.

\begin{theorem}
\label{thm:tacp_assignment_set_error}
Suppose Assumptions~\ref{ass:tacp_metric_learning} and~\ref{ass:tacp_assignment_stability} hold. Assume that the observations in $\gD_{\rm est}$, $\gD_{\rm cal}$, and $(X_{n+1},Y_{n+1})$ are i.i.d. from the data distribution, and assignment ties are resolved permutation-equivariantly. Then
\begin{align}
    {
    \sP\{Y_{n+1}\in\gC_\star^{\rm MV}(X_{n+1})\triangle
    {\widehat\gC}^{\rm MV}(X_{n+1})\}
    \longrightarrow0
    \qquad\text{as }N,n\to\infty.}\nonumber
\end{align}
\end{theorem}
Theorem~\ref{thm:tacp_assignment_set_error} shows that, as both estimation and calibration samples grow, the probability of the disagreement between the MV-OTCP region and its population oracle vanishes. Thus the learned construction asymptotically recovers the oracle acceptance region in probability.

\section{Individualized minimum-volume OTCP}
\label{sec:itacp}
MV-OTCP uses one ground metric for the marginal score law. When the conditional score distribution varies with the covariate, a single geometry can be inefficient even if it remains marginally valid. We therefore first localize the transport map and then learn a local minimum-volume metric. 
Let $\gP_x$ denote the conditional law of $S=s(X,Y)$ given $X=x$. For a fixed localization point $x$ and a candidate metric $M\in{\gM_0}$, define the local population Monge map
\begin{align*}
    T_{M,x}
    \in
    \argmin_{T_\#\gP_x=\gQ}
    \int
    c_M\{s,T(s)\}
    \,d\gP_x(s).
\end{align*}
By Lemma~\ref{lem:metric_brenier}, for a convex potential $g_{M,x}$, ${G_{M,x}(u):=T_{M,x}^{-1}(u)} = M^{-1/2}\nabla g_{M,x}(M^{1/2}u)$.
The pointwise task-aware metric is defined through the local volume objective
\begin{align}
\label{eq:localized_metric_target}
    \gJ_x(M)
    =
    \mathsf{Leb}\{T_{M,x}^{-1}(\sB_{1-\alpha})\}
    =
    \E_U\left[
            \{f_Q(U)\}^{-1}\one\{U\in\sB_{1-\alpha}\}
            \det\{\nabla^2g_{M,x}(M^{1/2}U)\}
    \right].
\end{align}
Thus, for each fixed localization point $x$, the oracle metric is $ M_x^\star\in\argmin_{M\in\gM_0}\gJ_x(M).$
Therefore, define the population oracle set as $\gC_\star^{\rm IMV}(x)=\left\{y\in\gY:\|T_{M_x^\star,x}\{s(x,y)\}\|\le q_\star\right\}.$
Since $(T_{M_x^\star,x})_\#\gP_x=\gQ$, it satisfies $\sP\{Y\in\gC_\star^{\rm IMV}(x)\mid X=x\}=\gQ\{u:\|u\|\le q_\star\} = 1-\alpha$.

At the empirical level, for each fixed $M$ and covariate value $x$, we let $\Psi_{\theta,x}$ denote a parameterized approximation to the Brenier map $\nabla g_{M,x}(\cdot)$ using any suitable parametric optimal transport model, in line with the MV-OTCP construction.
Given estimation sample $\gD_{\rm est}$, the empirical counterpart of \eqref{eq:localized_metric_target} is
\[
    \widehat\gJ_x(M,\theta)
    =
    \frac{1}{N}\sum_{j=1}^N \{f_Q(U_j)\}^{-1}
        \one\{U_j\in\sB_{1-\alpha}\}
        \det\left\{
            \nabla \Psi_{\theta,x}(M^{1/2}U_j)
            +\delta_N I_{{d}}
        \right\}.
\]
The pointwise metric estimate is
\[
    \widehat M_x
    \in
    \argmin_{M\in\gM_0}
    \widehat\gJ_x
    \{M,\widehat\theta_x(M)\}.
\]
After estimating $\widehat M_x$, we use the same estimation sample to fit a separate score-to-reference map $\widehat T_{\widehat M_x,x}$ for split calibration.
Implementation details of learning $\widehat M_x$ and $\widehat T_{\widehat M_x,x}$ are given in Appendix~\ref{sec:itacp_implementation}.
Let $\gD_{\rm cal}=\{(X_i,Y_i)\}_{i=1}^n$ be independent of the estimation sample. {For any $(x,y)$, define $\widehat R_N(x,y)=\|\widehat T_{\widehat M_x,x}\{s(x,y)\}\|$.}
The resulting Individualized Minimum-Volume OT Conformal Prediction (IMV-OTCP) set is
\begin{align}
\label{eq:itacp_set}
    {\widehat\gC}^{\rm IMV}(x)
    =
    \left\{
        y\in\gY:
        {\widehat R_N(x,y)}
        \le
         \widehat q_{1-\alpha}(\{{\widehat R_N(X_i,Y_i)}\}_{i=1}^n\cup\{\infty\})
    \right\}.
\end{align}

\begin{theorem}[Marginal coverage for IMV-OTCP]
\label{thm:itacp_validity}
Assume that $(X_i,Y_i)_{i=1}^{n+1}$ are exchangeable. Suppose all quantities defining the score rule ${(x,y)\mapsto\widehat R_N(x,y)}$ are fitted using data independent of the calibration sample and test point, and that this fixed rule is applied to all calibration and test observations. Then $\sP\{
        Y_{n+1}\in
        {\widehat\gC}^{\rm IMV}(X_{n+1})\}
    \ge
    1-\alpha.$
\end{theorem}

\begin{remark}[Anchor-point implementation]
\label{rem:itacp_anchor_implementation}
IMV-OTCP fit local transport objects at every covariate value. To reduce computation, we select a finite set of anchor covariates $\{a_1,\ldots,a_L\}$ using only the training and estimation samples, and fit ${(\widehat M_{a_j},\widehat T_{\widehat M_{a_j},a_j})}$ once for each anchor. For a generic covariate value $x$, we use the fit associated with its nearest anchor,
${j(x)=\argmin_{j\in[L]}\|x-a_j\|}$, so that
$(\widehat M_x,\widehat T_{\widehat M_x,x})={(\widehat M_{a_{j(x)}},\widehat T_{\widehat M_{a_{j(x)}},a_{j(x)}})}$. This amortizes the expensive pointwise fitting step while preserving the split-conformal validity argument.
\end{remark}


We next state the consistency results under high-level approximation assumptions on the reverse-map objective and the final forward-map.
 Define
\begin{align*}
    \zeta_{M,x}(u)
    :=
    {\{f_Q(u)\}^{-1}}{
        \one\{u\in\sB_{1-\alpha}\}
        |\det\{\nabla_uG_{M,x}(u)\}|
    }.
\end{align*}
Then $\gJ_x(M)=\E_U\{\zeta_{M,x}(U)\}$. Let
$\gJ_{N,x}(M):=N^{-1}\sum_{j=1}^N\zeta_{M,x}(U_j)$
denote the infeasible empirical objective based on the same reference draws as $\widehat\gJ_x$. Let $\gH_N$ be the $\sigma$-algebra generated by $\gD_{\rm est}$, the reference draws, and the pretrained predictor $\widehat\mu$, and define $R^\star(x,y)=\|T_{M_x^\star,x}\{s(x,y)\}\|$.

\begin{assumption}
\label{ass:imv_metric_learning}
Fix an interior covariate value $x$.
(i) There exist constants $0<B_x<\infty$ and $0\le L_x<\infty$ such that $|\zeta_{M,x}(U)|\le B_x$ almost surely for all $M\in\gM_0$. Moreover, for every $\epsilon>0$,
$\sup_{M_c\in\gM_0}\E_U\left[\sup_{M\in\gM_0(M_c,\epsilon)}|\zeta_{M,x}(U)-\zeta_{M_c,x}(U)|\right]\le L_x\epsilon$. 
(ii) The fitted reverse-map objective satisfies $\mathcal E_N^{\rm obj}(x)=o_{\sP}(1)$ as $N\to\infty$, where $\mathcal E_N^{\rm obj}(x):=\sup_{M\in\gM_0}|\widehat\gJ_x\{M,\widehat\theta_x(M)\}-\gJ_{N,x}(M)|$.
(iii) For every $\epsilon>0$,
\(
    \inf_{\substack{M\in\gM_0:\ \|M-M_x^\star\|_F\ge\epsilon}}
    \{\gJ_x(M)-\gJ_x(M_x^\star)\}
    >0.
\)
\end{assumption}

Assumption~\ref{ass:imv_metric_learning}(i) imposes regularity on the local objective. Part~(ii) is the localized counterpart of Assumption~\ref{ass:tacp_metric_learning}(ii), controlling the reverse-map approximation error that enters metric learning. Part~(iii) parallels Assumption~\ref{ass:tacp_metric_learning}(iii) and provides the standard separation condition needed for local metric consistency.

\begin{theorem}
\label{thm:localized_metric_consistency}
Suppose Assumption~\ref{ass:imv_metric_learning}(i) holds. For fixed ${d}$, there is a universal constant $C>0$ such that, for every $\delta\in(0,1)$, with probability at least $1-\delta$,
\begin{align*}
    &\sup_{M\in\gM_0}
    |\widehat\gJ_x\{M,\widehat\theta_x(M)\}-\gJ_x(M)|
    \le \mathcal E_N^{\rm obj}(x)+C B_x
    \sqrt{
        \frac{
            {d}^2\log\left\{e+2\sqrt{{d}}\,\overline\lambda L_xN/B_x\right\}
            +\log(2/\delta)
        }{N}
    }.
\end{align*}
If, in addition, parts~(ii)--(iii) of Assumption~\ref{ass:imv_metric_learning} hold, then
\(
    \|\widehat M_x-M_x^\star\|_F
    \overset{\sP}{\longrightarrow}0.
\)
\end{theorem}

\begin{assumption}
\label{ass:imv_radius_nuisance}
    At every fixed interior $x\in\gX$, $\mathcal E_N^{\rm rad}(x)=o_{\sP}(1)$, where $\mathcal E_N^{\rm rad}(x):=\E[|\widehat R_N(x,Y)-R^\star(x,Y)|\mid X=x,\gH_N]$.
    Moreover, there exists a nonnegative $P_X$-integrable function $B_{\rm rad}$ such that $\mathcal E_N^{\rm rad}(x)\le B_{\rm rad}(x)$ almost surely for all $N$.
\end{assumption}
\begin{theorem}
\label{thm:itacp_conditional_consistency}
Suppose Assumption~\ref{ass:imv_radius_nuisance} holds, $F_Q$ is continuous and strictly increasing at $q_\star$, and the calibration samples and test observation are i.i.d. and independent of $\gH_N$. For any fixed interior $x$,
\begin{align*}
    \lim_{(N,n)\to \infty}\sP\left\{
        Y_{n+1}\in
        \gC_\star^{\rm IMV}(x)
        \mathbin\triangle
        {\widehat\gC}^{\rm IMV}(x)
        \,\middle|\,
        X_{n+1}=x
    \right\}
    = 0,
\end{align*}
and therefore $\lim_{(N,n)\to \infty}\sP\{Y_{n+1}\in{\widehat\gC}^{\rm IMV}(x)\,|\,X_{n+1}=x\}= 1-\alpha$.
\end{theorem}

\section{Experiments}
\label{sec:experiments}
\label{sec:simulation}

We conduct experiments on synthetic data and five real datasets to compare with {SCP} \citep{vovk2005algorithmic}, {MVCS} \citep{braun2025minimum}, {FOTCP} \citep{ndiaye2025beyond}, and {RLCP} \citep{hore2024conformal}.
The nominal coverage is $1-\alpha=0.9$, using $100$ replications with common samples across methods. 
Predictors and metric parameters are fitted before calibration, and test data are reserved for evaluation. All methods use the residual score $s(x,y)=y-\widehat\mu(x)$ with a ridge-regression predictor fitted on the training data. Additional experiments and experimental details appear in Appendices~\ref{app:additional_exp} and~\ref{app:simulation_protocols}.


\textbf{Evaluation metrics.}
{For a prediction set $\widehat C(X)$ and test observations indexed by $i=1,\ldots,n_{\rm test}$, let $C_i=\one\{Y_i\in\widehat C(X_i)\}$ and $V_i=\mathsf{Leb}\{\widehat C(X_i)\}$.} The following criteria are used for comparison: (1) \textit{Marginal coverage}: ${n_{\rm test}^{-1}}\sum_i C_i$; (2) \textit{Set volume}: $\operatorname{Median}_{i} V_i$ (root volume $\operatorname{Median}_i V_i^{1/{d}}$ for real data); (3) \textit{Symmetric difference}: $\operatorname{Median}_i\mathsf{Leb}\{\widehat C(X_i)\triangle C^\star(X_i)\}$, where $C^\star(x)$ denotes the oracle $(1-\alpha)$ highest-density region.; (4) \textit{Worst-case coverage}: $\min_{B}\sum_i\mathbbm{1}(X_i\in B)C_i/\sum_i\mathbbm{1}(X_i\in B)$, with 100 random balls $B$ on $\gX$; (5) $\sqrt{\mathrm{MSCE}}$: $\{\E[\{\sP\bigl(Y\in\widehat\gC(X)\mid X\bigr)-(1-\alpha)\}^2]\}^{1/2}$, the empirical estimation of this quantity differs between the simulation and real-data experiments; detailed definitions are provided in Appendix~\ref{app:eval_metrics}.
Error bars show mean $\pm1.96$ standard errors.

\subsection{Homogeneous simulations}

We generate $Y=m(X)+\varepsilon$ with $X\sim\operatorname{Unif}[-2,2]$ and three bivariate residual laws. Let $e(\theta)=(\cos\theta,\sin\theta)^\top$ and let $R_0=\{-2\log V\}^{1/2}$ for $V\sim\operatorname{Unif}(0,1)$.
We draw $\Theta$ with density proportional to $r_j(\theta)^2$, where $j=1,2,3$ indexes the setting. For the first two settings, we then generate
$\varepsilon=A_j\{R_0r_j(\Theta)e(\Theta)\}+\xi$,
where $r_j(\theta)=c_j[1+d_j\cos\{k_j(\theta-\phi_j)\}]$, $A_j$ is a rotated anisotropic linear transformation, and $\xi\sim N(0,0.015^2I_2)$. Implementation details are deferred in Appendix~\ref{app:homogeneous_implementation}.  We consider the following three settings: 

\begin{itemize}[leftmargin=2.3em]

\item[(i)] \textit{Five-star.} Set $k_1=5$ and use the corresponding parameters $(c_1,d_1,\phi_1,A_1)$. 

\item[(ii)] \textit{Clover.} Set $k_2=4$ with a different choice of $(c_2,d_2,\phi_2,A_2)$. 

\item[(iii)] \textit{Exp.} Let $R_3\sim\operatorname{Gamma}(2,1)$, $\mu_3=\mathbb{E}\{R_3r_3(\Theta)e(\Theta)\}$ and  $\varepsilon=A_3\{R_3r_3(\Theta)e(\Theta)-\mu_3\}$.

\end{itemize}

\begin{figure}[h]
    \centering
    \includegraphics[width=\linewidth]{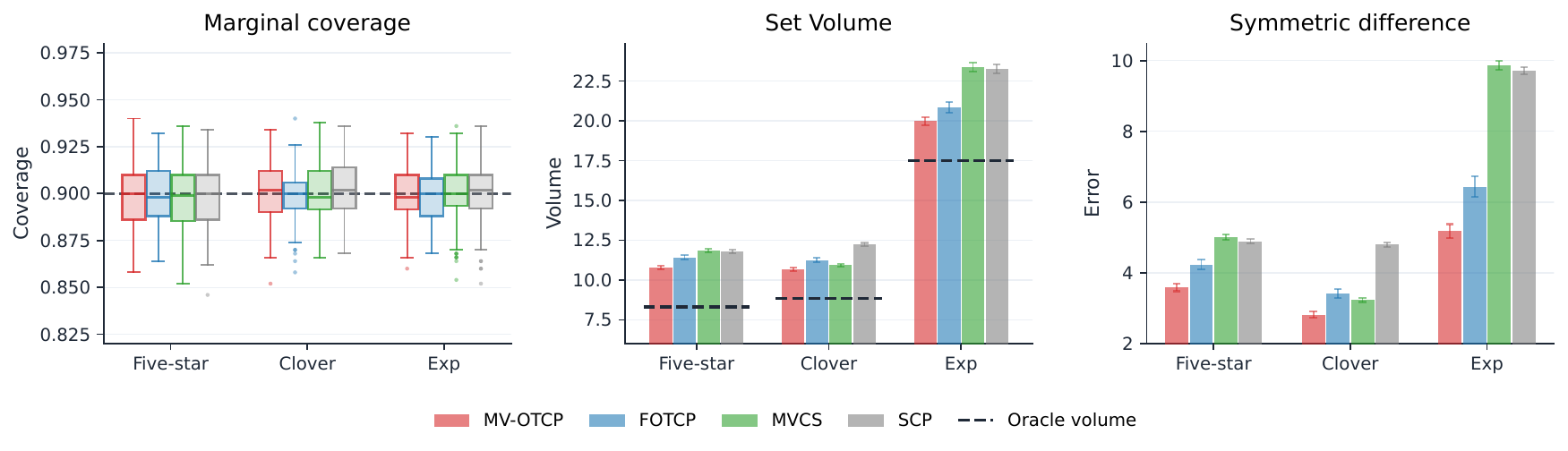}
    \caption{Homogeneous settings. Dashed lines mark nominal coverage and setting-specific oracle minimum volumes.}

    \label{fig:volume_target}
\end{figure}

Figure~\ref{fig:volume_target} shows that all methods achieve coverage close to the nominal level. Across all three settings, MV-OTCP consistently attains the smallest mean volume and the smallest symmetric difference from the oracle set. The improvement demonstrates the benefit of learning a volume-aware transport geometry.





To isolate the effect of the learned metric $M$, we focus on setting~(i)-(ii) and vary the anisotropy through {the family $A_j(\rho)$}, with larger $\rho$ producing more elongated residual distributions while preserving the overall area scale.
More precisely, let $R(\gamma)$ denote rotation through angle $\gamma$ and write each base transform as $A_j^{(0)}=R(\gamma_j)\operatorname{diag}(a_j,b_j)$, where $j\in\{1,2\}$ indexes setting (i) and (ii), respectively, and define the aspect-ratio family
\begin{align*}
    A_j(\rho)
    =
    R(\gamma_j)
    \operatorname{diag}\left(\sqrt{a_j b_j\rho},\sqrt{a_j b_j/\rho}\right),
    \qquad \rho>0 .
\end{align*}
In this way, increasing $\rho$ makes the residual distribution more elongated without mechanically changing the overall area scale. We display $\rho\in\{2,4,6,8,10\}$. {The main simulations use $A_j=A_j(6)$ for $j=1,2$, so the $\rho=6$ records are exactly those used for Five-star and Clover in Figure~\ref{fig:volume_target}.}
The bars give the percentage improvement of {MV-OTCP} over {FOTCP}, so positive values favor {MV-OTCP}.

\begin{figure}[h]
    \centering
    \includegraphics[width=\linewidth]{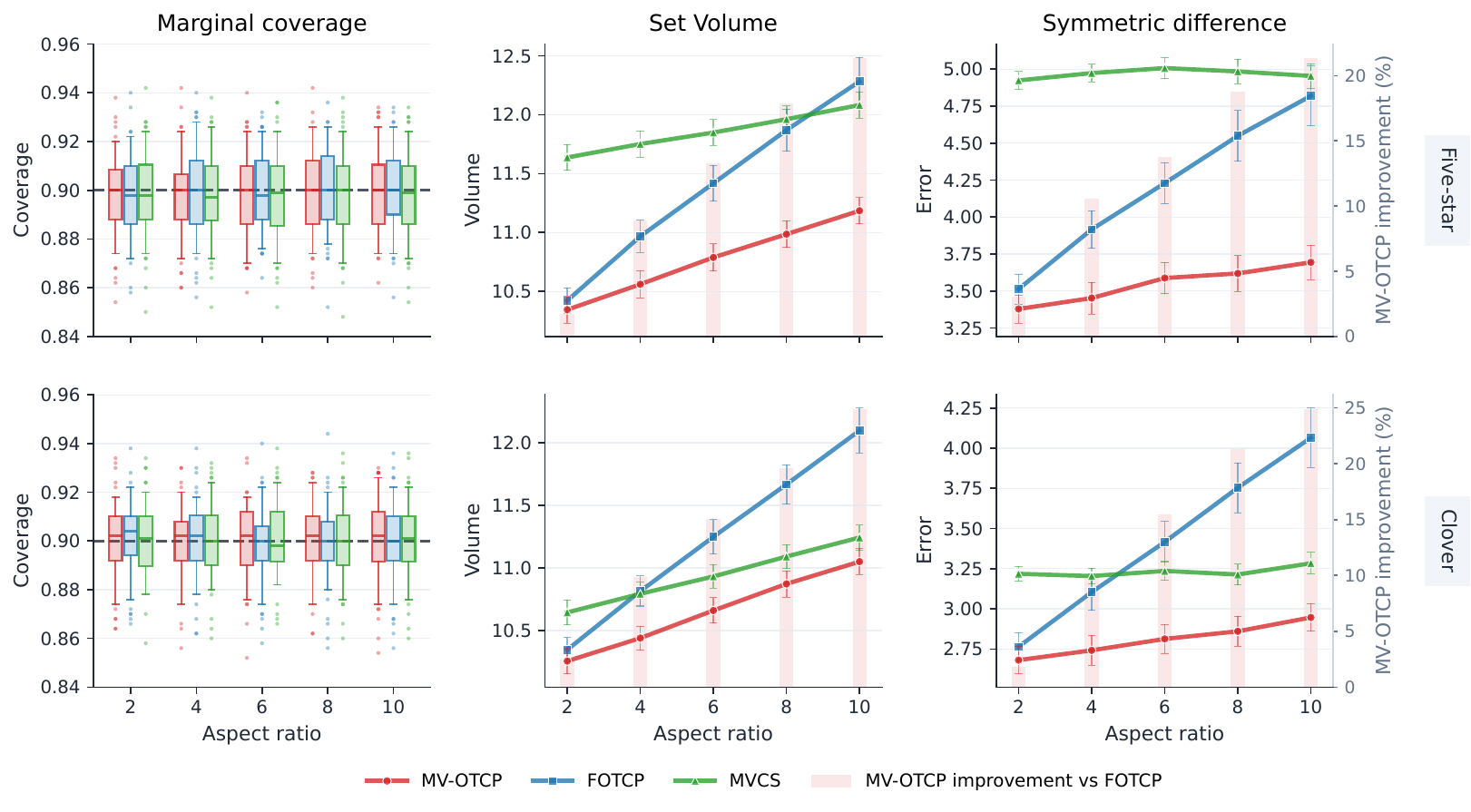}
    \caption{Aspect-ratio ablation for Five-star and Clover settings. Translucent bars report the mean paired improvement of {MV-OTCP} over {FOTCP} on the right axis. The dashed line in the marginal-coverage panels marks the nominal target.}
    \label{fig:aspect_ratio_ablation}
\end{figure}

Figure~\ref{fig:aspect_ratio_ablation} shows that marginal coverage remains near nominal coverage throughout. {MV-OTCP} uniformly improves on {MVCS} in both efficiency metrics and is close to {FOTCP} at the smallest displayed aspect ratio. Its advantage over {FOTCP} increases with $\rho$, especially for symmetric difference, showing that the learned cost captures global anisotropy without sacrificing the nonconvex geometry supplied by the OT assignment.

\subsection{Heterogeneous simulations}
\label{subsec:heterogeneous_simulation}
We generate $Y=m(X)+\varepsilon_X$ with $X\sim\operatorname{Unif}[-2,2]$. Let $Z\sim N(0,I_2)$, $U\sim\operatorname{Unif}(0,1)$, and $\Theta\sim\operatorname{Unif}(0,2\pi)$ be mutually independent. Implementation details are deferred in Appendix~\ref{app:heterogeneous_implementation}. Consider the following three settings:

\begin{itemize}[leftmargin=2.3em]
\item[(i)] \textit{$X$-mixture.} With
$C\mid X=x\sim\operatorname{Categorical}\{p(x)\}$ and $\bar v(x)=\sum_{k=1}^4p_k(x)v_k$, define
$\varepsilon_x=R\{\theta_{\rm m}(x)\}[{a_{\rm m}(x)}\{v_C-\bar v(x)\}+\tau(x)Z]$.

\item[(ii)] \textit{$X$-clover.} Let $\varepsilon_x=A_{\rm c}(x)[\sqrt U\,q_x(\Theta)e(\Theta)]+0.020Z$.
The radial boundary $q_x$ varies the lobe strength and phase, while $A_{\rm c}(x)$ changes orientation, scale, and anisotropy.


\item[(iii)] \textit{$X$-Exp.} For independent $E_1,E_2\sim\operatorname{Exp}(1)$, define $\varepsilon_X=r(X)R\{\theta_{\rm e}(X)\}E$ with $E=(E_1,E_2)^\top$.
The positive scale $r(x)$ and orientation $\theta_{\rm e}(x)$ both vary with $x$.
\end{itemize}

\begin{figure}[ht]
    \centering
    \includegraphics[width=\linewidth]{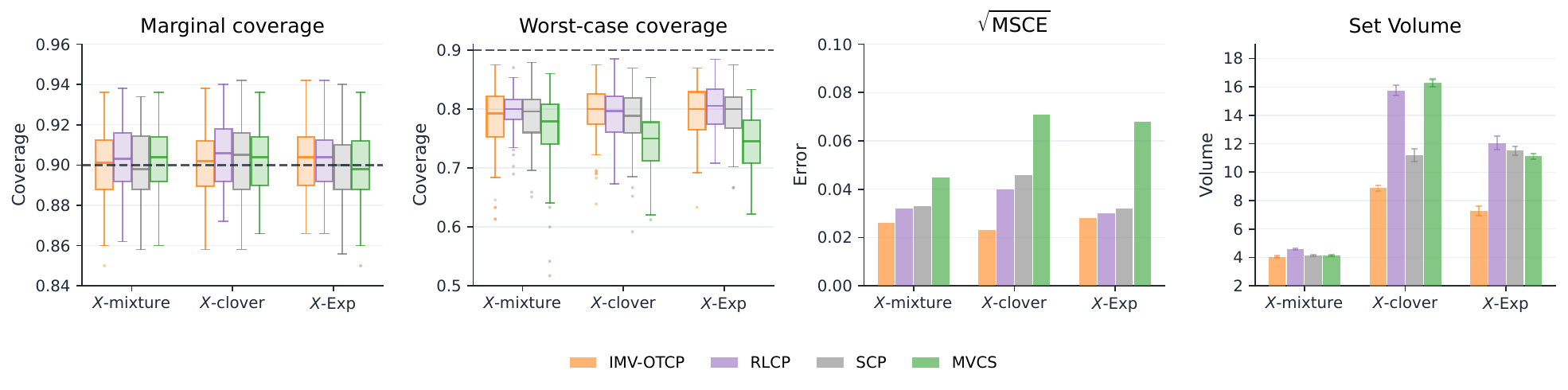}
    \caption{Heterogeneous settings. Boxplots show marginal and worst-case coverage; bars show root MSCE and set volume. Dashed lines mark nominal coverage. }
    \label{tab:itacp_heterogeneous}
\end{figure}

In Figure~\ref{tab:itacp_heterogeneous}, all methods maintain near-nominal marginal coverage. {IMV-OTCP} attains the smallest volume and lowest root MSCE across all settings, with particularly substantial improvements in the \textit{$X$-Mixture} and \textit{$X$-Exp} settings, while also achieving competitive worst-case coverage.

We next conduct an ablation study that jointly examines the effect of the learned metric and the number of anchors by comparing \texttt{IMV-OTCP} with its $M=I$ counterpart on the \textit{$X$-Mixture} and \textit{$X$-Exp} settings. We compare the reported method with {IMV-OTCP($M=I$)} over $100$ paired replications and $L\in\{3,5,10,15,20\}$ empirical-quantile anchors. The two variants use the same localized fitting and calibration pipeline, differing only in the ground-cost metric: {IMV-OTCP} learns $M$ at each anchor, whereas the baseline fixes $M=I$. This comparison isolates the contribution of metric learning in the individualized setting.

\begin{figure}[h]
    \centering
    \includegraphics[width=\linewidth]{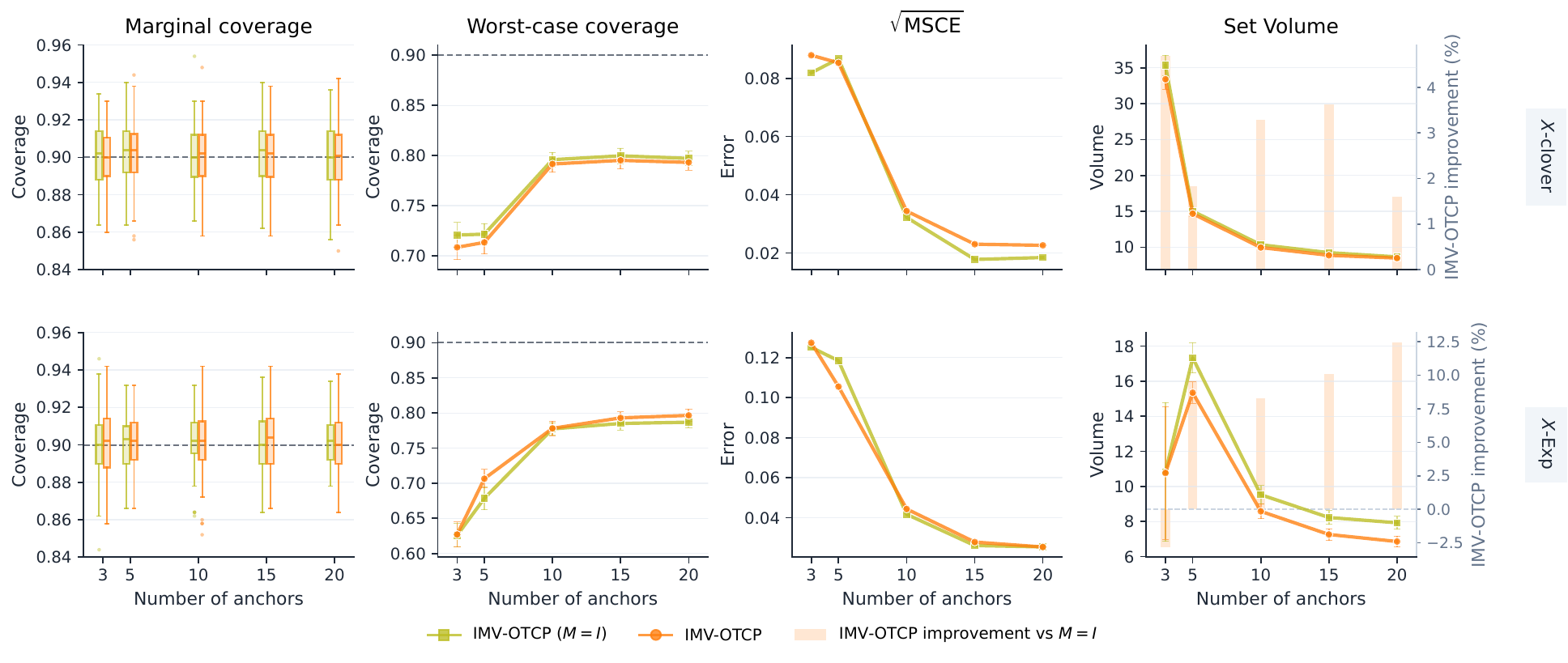}
    \caption{{Localized ground-cost and anchor ablation. Translucent bars report the mean paired volume improvement of learned-metric {IMV-OTCP} over {IMV-OTCP($M=I$)}.}}
    \label{fig:imv_anchor_ablation}
\end{figure}

In Figure~\ref{fig:imv_anchor_ablation}, both branches retain marginal coverage near $0.9$. Increasing the anchor count substantially improves the two local-coverage diagnostics, indicating that very coarse anchor grids are not adequate approximations of the varying conditional geometry. With too few anchors, the procedure may fail to capture the conditional structure adequately and can produce less stable prediction sets.
Across anchor counts, {IMV-OTCP} consistently reduces volume relative to the $M=I$ baseline, while maintaining comparable marginal and local-coverage diagnostics. This provides evidence that learning the ground metric contributes additional efficiency beyond the localized construction itself.

\subsection{Real-data regression benchmarks}
\label{sec:real_data}
\label{subsec:real_volume}
We use the \texttt{Energy}, \texttt{Bias}, \texttt{Taxi}, \texttt{rf1}, and \texttt{scm1d} datasets, with response dimensions $2,2,2,8,$ and $16$, respectively, {and follow the preprocessing protocol of \citet{dheur2025unified,braun2025minimum}.} We further project the transformed covariates onto a whitened PCA representation with at most eight components, which reduces the cost of localized computation and yields a better-conditioned representation for neighborhood distances. All transformations are fitted only on the training split and then held fixed. Additional implementation details can be found in Appendix~\ref{app:real_volume}.

Table~\ref{tab:real_volume_results} shows that all methods remain close to the nominal marginal coverage level. {IMV-OTCP} achieves the smallest volume on all datasets except \texttt{rf1}, while also attaining competitive $\sqrt{\mathrm{MSCE}}$ relative to the localized methods. These results indicate that {IMV-OTCP} improves prediction-set efficiency without sacrificing marginal coverage.
Its performance on the conditional diagnostic further suggests that the individualized construction captures meaningful local variation.

\begin{table}[h]
\centering
\footnotesize
\setlength{\tabcolsep}{1.4pt}
\caption{Regression benchmarks over $100$ splits. Entries are means (standard errors). Root volume is the median of $V_i^{1/{d}}$. Bold marks the lowest root volume in each dataset.}
\label{tab:real_volume_results}
\begin{tabular}{llcccc}
\toprule
Metric & Dataset & SCP & MVCS & RLCP & IMV-OTCP \\
\midrule
\multirow{5}{*}{\shortstack{Marginal \\coverage}}
 & \texttt{Energy} & $0.909(0.004)$ & $0.910(0.004)$ & $0.912(0.003)$ & $0.917(0.004)$ \\
 & \texttt{Bias}   & $0.901(0.002)$ & $0.901(0.002)$ & $0.903(0.002)$ & $0.901(0.002)$ \\
 & \texttt{Taxi}   & $0.901(0.001)$ & $0.902(0.001)$ & $0.903(0.001)$ & $0.900(0.001)$ \\
 & \texttt{rf1}    & $0.900(0.001)$ & $0.901(0.001)$ & $0.901(0.001)$ & $0.900(0.002)$ \\
 & \texttt{scm1d}  & $0.898(0.001)$ & $0.899(0.001)$ & $0.901(0.001)$ & $0.897(0.002)$ \\
\addlinespace
\multirow{5}{*}{$\sqrt{\mathrm{MSCE}}$}
 & \texttt{Energy} & $0.115(0.0039)$ & $0.101(0.0028)$ & $0.104(0.0032)$ & $0.091(0.0022)$ \\
 & \texttt{Bias}   & $0.044(0.0008)$ & $0.042(0.0008)$ & $0.042(0.0008)$ & $0.042(0.0008)$ \\
 & \texttt{Taxi}   & $0.033(0.0004)$ & $0.029(0.0010)$ & $0.025(0.0004)$ & $0.029(0.0006)$ \\
 & \texttt{rf1}    & $0.065(0.0007)$ & $0.061(0.0008)$ & $0.051(0.0006)$ & $0.060(0.0007)$ \\
 & \texttt{scm1d}  & $0.070(0.0005)$ & $0.073(0.0007)$ & $0.054(0.0007)$ & $0.065(0.0006)$ \\
\addlinespace
\multirow{5}{*}{\shortstack{Root\\volume}}
 & \texttt{Energy} & $1.474(0.021)$ & $1.337(0.030)$ & $1.867(0.025)$ & $\mathbf{1.107(0.019)}$ \\
 & \texttt{Bias}   & $2.159(0.007)$ & $2.177(0.008)$ & $2.149(0.007)$ & $\mathbf{1.779(0.007)}$ \\
 & \texttt{Taxi}   & $3.462(0.010)$ & $3.492(0.090)$ & $3.563(0.009)$ & $\mathbf{2.703(0.008)}$ \\
 & \texttt{rf1}    & $1.720(0.006)$ & $\mathbf{1.286(0.062)}$ & $2.088(0.006)$ & $1.309(0.005)$ \\
 & \texttt{scm1d}  & $2.025(0.007)$ & $4.266(0.616)$ & $2.878(0.008)$ & $\mathbf{1.759(0.004)}$ \\
\bottomrule
\end{tabular}
\end{table}

\section{Discussion}
\label{sec:discussion}
Learning the OT ground cost provides a unified framework for shaping multivariate conformal regions toward minimum volume or downstream task objectives. MV-OTCP uses a global metric in a full-conformal assignment, while IMV-OTCP fits local transport radii and uses split calibration; both separate conformal calibration from geometric optimization, yielding finite-sample marginal validity and, under consistency conditions, oracle-region recovery and asymptotic conditional coverage.
Experiments demonstrate the benefits of global metric learning, local adaptation under heterogeneous uncertainty, and task-aware optimization.


\bibliographystyle{plainnat}
\bibliography{references}

\clearpage
\appendix

\section{Implementation details}

\subsection{Elliptical volume-optimal geometry}
\label{app:elliptical_derivation}
Suppose that the score density is elliptically symmetric, with
\[
    f_P(s)=|\Sigma|^{-1/2}{g(\|\Sigma^{-1/2}(s-m)\|)},
\]
where $\Sigma\succ0$, $m\in\R^{{d}}$, and $g$ is nonincreasing, and that the reference distribution $\gQ$ is radial, with $U=R\Theta$, $R\sim\operatorname{Unif}(0,1)$, and $\Theta$ uniformly distributed on $\{\theta\in\R^{{d}}:\|\theta\|=1\}$ independently of $R$. Let $r(s)=\|\Sigma^{-1/2}(s-m)\|$, $\rho=r(S)$, and let $F_\rho$ be the distribution function of $\rho$. For a radial reference law with uniform radial component, define
\begin{align*}
    T_{M^\star}(s)
    =
    F_\rho(r(s))
    \frac{\Sigma^{-1/2}(s-m)}{r(s)}
\end{align*}
for $r(s)>0$. Then $(T_{M^\star})_\#\gP=\gQ$, and
\begin{align*}
	    T_{M^\star}^{-1}(\sB_{1-\alpha})
    =
    \{s:r(s)\le F_\rho^{-1}(1-\alpha)\}.
\end{align*}
This is the central density superlevel ellipsoid and is therefore the minimum-volume measurable set with probability $1-\alpha$. Moreover, by Lemma~\ref{lem:metric_brenier}, $T_{M^\star}$ is generated by the metric $M_{\rm vol}^\star\propto \Sigma^{-1/2}$. Indeed, set $A=\Sigma^{-1/4}$, $\tilde m=Am$, and $h(r)=\int_0^r F_\rho(t)\,dt$. For $\tilde s\in\R^{{d}}$, define $\tilde r(\tilde s)=\|A(\tilde s-\tilde m)\|$ and $\psi(\tilde s)=h\{\tilde r(\tilde s)\}$. Since $F_\rho$ is nondecreasing, $h$ is convex and nondecreasing; hence $\psi$ is convex. Differentiating with respect to $\tilde s$ and evaluating at $\tilde s=As$ gives
\begin{align*}
    \nabla\psi(As)
    =
    F_\rho(r(s))
    \frac{A^3(s-m)}{r(s)}
    =
    A T_{M^\star}(s).
\end{align*}
Since $A=M^{1/2}$ when $M=\Sigma^{-1/2}$, Lemma~\ref{lem:metric_brenier} implies that $T_{M^\star}$ is optimal for $c_M$ with $M=\Sigma^{-1/2}$.

\subsection{Implementation details of MV-OTCP}
\label{subsec:tacp_implementation}
A direct empirical approach would be to replace the population objective in \eqref{eq:task_loss_hessian_form} by its sample analogue and optimize it over $M$. This is generally infeasible because, apart from the metric $M$, the Brenier representation only ensures the existence of a convex potential $g_M$, which is unknown and depends on $M$ implicitly. Thus the objective cannot be evaluated directly unless the associated transport map is estimated. We represent its gradient map directly with GradNet \citep{chaudhari2025gradnetot} and differentiate the empirical task loss through the fitted transport problem.

Let $\gD_{\rm est}=\{(X_i,Y_i)\}_{i=1}^N$ be an independent estimation sample with scores $S_i=Y_i-\widehat\mu(X_i)$, and let $U_1,\ldots,U_N\sim\gQ$ be independent reference draws. For a current metric $M$, transform the samples by
\[
    \bar S_i=M^{1/2}S_i,\qquad \bar U_i=M^{1/2}U_i .
\]
Because the empirical score law is discrete, define the KDE of $(M^{1/2})_\#\gP$
\[
    \widehat f_{P,M}(z)
    =
    \frac{1}{N}\sum_{i=1}^N
    \kappa_b(z-\bar S_i),
\]
where $\kappa_b$ is a score-space kernel.
We use an isotropic Gaussian kernel.  At each metric iterate, its bandwidth is
chosen by the weighted Scott rule
\[
    b
    =
    \max\left\{
        10^{-3},
        \left(\frac{1}{{d}}\sum_{i=1}^N
        \widetilde w_i\|\bar S_i-\bar S_w\|^2\right)^{1/2}
        n_{\rm eff}^{-1/({d}+4)}
    \right\},
    \qquad
    n_{\rm eff}=\left(\sum_{i=1}^N\widetilde w_i^2\right)^{-1},
\]
where $\widetilde w_i$ are normalized localization weights and
$\bar S_w=\sum_i\widetilde w_i\bar S_i$; global fits use
$\widetilde w_i=1/N$.  The same rule is used for the source and target KDEs
when fitting the score-to-reference map.
Let $\Psi_\theta$ be a GradNet mapping the transformed reference distribution to the transformed score distribution. Following \citet{chaudhari2025gradnetot}, we fit $\Psi_\theta$ through the Monge--Amp\`ere density identity:
\begin{align}
\label{eq:global_reverse_gradnet_loss}
    \widehat\gL^{\leftarrow}_{\rm GN}(M,\theta)
    =
    \frac{1}{2N}\sum_{j=1}^N
    \Big[
        \log\det\{
            \nabla \Psi_\theta(M^{1/2}U_j)
        \}
        -
        \log f_Q(U_j)
        +
        \log\widehat f_{P,M}
        \{\Psi_\theta(M^{1/2}U_j)\}
    \Big]^2 .
\end{align}
Our parameterization adds a positive linear skip connection and a bias to a
positive combination of $J$ GradNet modules.  Each module has one hidden
projection of width $H$, a weighted-softmax activation, and the tied transpose
projection required by the gradient-map construction.  The experiment-specific
values of $(J,H)$ are reported below.
For a fixed $M$, let $\widehat\theta(M)$ minimize \eqref{eq:global_reverse_gradnet_loss}. We rewrite the empirical volume objective \eqref{eq:empirical_volume_objective} here for convenience:
\begin{align}
    \widehat \gJ(M,\theta)
    =
    \frac{1}{N}\sum_{i=1}^N
    \frac{
        \one\{U_i\in\sB_{1-\alpha}\}
        \det\{\nabla \Psi_\theta(M^{1/2}U_i)+{\delta_N} I_{{d}}\}
    }{f_Q(U_i)},\nonumber
\end{align}
where ${\delta_N}>0$ is a small numerical stabilizer. The empirical metric-learning problem is
\[
    \widehat M
    \in
    \argmin_{M\in\gM_0}
    \widehat\gJ\{M,\widehat\theta(M)\}.
\]
Importantly, this is a profiled objective because the fitted transport map depends on the metric. Its total derivative is
\begin{align*}
    \nabla_M
    \widehat\gJ\{M,\widehat\theta(M)\}
    ={}&
    \underbrace{
        \left.
        \partial_M\widehat\gJ(M,\theta)
        \right|_{\theta=\widehat\theta(M)}
    }_{\text{direct metric derivative}}
    +
        \underbrace{
        \{\nabla_M\widehat\theta(M)\}^{\top}
        \left.
        \nabla_\theta\widehat\gJ(M,\theta)
        \right|_{\theta=\widehat\theta(M)}
    }_{\text{transport-response derivative}}.
\end{align*}
The second term is generally nonzero because the GradNet is fitted by the OT loss rather than the volume loss. If $\widehat\theta(M)$ is a locally stationary minimizer of \eqref{eq:global_reverse_gradnet_loss} and the parameter Hessian is nonsingular, implicit differentiation gives
\[
    \nabla_M\widehat\theta(M)
    =
    -\left[
        \nabla_{\theta\theta}^2
        \widehat\gL^{\leftarrow}_{\rm GN}
        \{M,\widehat\theta(M)\}
    \right]^{-1}
    \nabla_{\theta M}^2
    \widehat\gL^{\leftarrow}_{\rm GN}
    \{M,\widehat\theta(M)\}.
\]
In computation, the linear system can be damped and solved by conjugate gradients; differentiating through a finite number of inner updates is an alternative.

To enforce the metric constraints, define
\[
    \gB_0
    =
    \left\{
        B=B^\top:\operatorname{tr}(B)=0,\ 
        \log(\underline\lambda)I_{{d}}
        \preceq B\preceq
        \log(\overline\lambda)I_{{d}}
    \right\},
    \qquad
    M(B)=\exp(B).
\]
Then $M(B)\in\gM_0$. The spectral projection may be replaced in practice by a condition-number penalty, provided that the returned metric satisfies the localization condition used in the theory.
The resulting bilevel metric-learning procedure is summarized in Algorithm~\ref{alg:metric_learning}.
\begin{algorithm}[H]
\caption{Volume-based metric learning for MV-OTCP}
\label{alg:metric_learning}
\begin{algorithmic}[1]
\Require Estimation scores $S_1,\ldots,S_N$, reference draws
$U_1,\ldots,U_N$, kernel bandwidth $b$, initial parameter $B_0\in\gB_0$,
outer iterations $T$, inner steps $K$
\For{$t=0,\ldots,T-1$}
    \State Set $M_t=\exp(B_t)$.
    \State Transform samples: $\bar S_i=M_t^{1/2}S_i$ and $\bar U_i=M_t^{1/2}U_i$.
    \State Form the transformed-score KDE $\widehat f_{P,M_t}$.
    \State Starting from warm-started $\theta$, take $K$ steps minimizing
    $\widehat\gL^{\leftarrow}_{\rm GN}(M_t,\theta)$ in
    \eqref{eq:global_reverse_gradnet_loss}; denote the result by $\theta_t$.
    \State Evaluate $\widehat\gJ(M_t,\theta_t)$ using
    \eqref{eq:empirical_volume_objective}.
    \State Compute the transport-response derivative by implicit
    differentiation through the fitted GradNet.
    \State Update $B_{t+1}=\Pi_{\gB_0}\big[B_t-\gamma_t\nabla_B
    \widehat\gJ\{M(B),\widehat\theta(M(B))\}|_{B=B_t}\big]$.
\EndFor
\Ensure Learned determinant-one metric $\widehat M=M(B_T)$
\end{algorithmic}
\end{algorithm}




\subsection{General task-aware OTCP}
\label{sec:general_task_loss}

We rewrite objective function here for convinence.
For a decision $z\in\gZ$ and an integrable loss density $\ell(z,y)$, define
\begin{equation}
\label{eq:integral_task_loss}
 \gL(z,\gC(x))=\int_{\gC(x)}\ell(z,y)\,dy,\qquad
 z_M^\star(x)\in\argmin_{z\in\gZ}\gL(z,\gC_M(x)).
\end{equation}
The corresponding task-aware target for metric learning is
\begin{equation}
 M_{\rm task}^\star\in\argmin_{M\in\gM_0}\gJ_{\rm task}(M),\qquad
 \gJ_{\rm task}(M)=\E_X\gL\{z_M^\star(X),\gC_M(X)\}.\nonumber
\end{equation}
Using the same change-of-variables argument as in Section~\ref{sec:tacp}, the task loss can be rewritten as
\begin{equation}
 \gJ_{\rm task}(M)=\E_{X,U}\left[
 \ell\{z_M^\star(X),\widehat\mu(X)+G_M(U)\}\frac{\one\{U\in\sB_{1-\alpha}\}\det\{\nabla G_M(U)\}}{f_Q(U)}\right].\nonumber
\end{equation}
Volume is recovered as the special case $\ell\equiv1$, whereas a general task loss additionally accounts for where the accepted volume is placed in the response space.

\begin{example}[Integrated decision loss]
\label{example:average_RO}
The criterion \eqref{eq:integral_task_loss} aggregates loss over the uncertainty set, as in averaged robust optimization \citep{bertsimas2023robust,cortes-gomez2025utilitydirected}. It is an unnormalized integral, so both region size and location affect metric selection. 
\end{example}

\begin{example}[Smoothed contextual robust optimization]
\label{example:smooth_CRO}
The contextual robust optimization \citep{chenreddy2022data} objective 
$\inf_{z\in\gZ}\sup_{y\in \gC(x)}\ell(z,y)$
can be smoothed by the log-exponential transform
\begin{align}
    \gL_\tau(z,\gC(x))
    =
    \tau\log
    \int_{y\in \gC(x)}
    \exp\{\ell(z,y)/\tau\}\,dy ,
    \qquad \tau>0 .\nonumber
\end{align}
For fixed $\tau$, minimizing above equation over $z$ is equivalent to minimizing an integral of the form \eqref{eq:integral_task_loss} with transformed loss density $\ell_\tau^{\exp}(z,y)=\exp\{\ell(z,y)/\tau\}$. As $\tau\downarrow0$, under mild regularity conditions—e.g., $\gC(x)$ is compact and $\ell(z,\cdot)$ is continuous, the log-exponential objective approaches the worst-case loss.
\end{example}

\subsubsection{Empirical task-aware criterion}
\label{subsec:empirical_general_task}

Let $\widehat G_M$ be a reference-to-score transport-map estimate obtained from the estimation sample, and define
\[
    \widehat w_M(u)
    =
    \frac{\one\{u\in\sB_{1-\alpha}\}
    |\det\nabla_u\widehat G_M(u)|}{f_Q(u)}.
\]
For estimation covariates $X_1,\ldots,X_N$ and a common independent reference sample $U_1,\ldots,U_K\sim\gQ$, define
\begin{align}
    \widehat\gL_i(z;M)
    :=\frac{1}{K}\sum_{j=1}^K
    \widehat w_M(U_j)
    \ell\left\{
        z,\widehat\mu(X_i)+\widehat G_M(U_j)
    \right\}.\nonumber
\end{align}
All $K$ transported points are integrated jointly for each decision; optimizing a separate decision at every quadrature point would define a different, generally degenerate objective. Let $\widehat z_i^\star(M)\in\argmin_{z\in\gZ}\widehat\gL_i(z;M)$, then
\begin{align}
    \widehat M_{\rm task}
    \in\argmin_{M\in\gM_0}
    \widehat\gJ_{\rm task}(M),\quad
    \widehat\gJ_{\rm task}(M)
    :=\frac{1}{N}\sum_{i=1}^N
    \widehat\gL_i\{\widehat z_i^\star(M);M\}.\nonumber
\end{align}
The estimated metric $\widehat M_{\rm task}$ can be used in the full-conformal assignment \eqref{eq:final_assignment}, with the coverage guarantee of Theorem~\ref{thm:tacp_validity}. For computational efficiency in the task experiments, we instead fit a score-to-reference map $\widehat T_{\widehat M_{\rm task}}$ before calibration and calibrate its radius by the ordinary split-conformal rule, as in IMV-OTCP. This gives Task-Aware OT Conformal Prediction (TA-OTCP):
\begin{equation}
 {\widehat\gC}^{\rm TA}(x)=\LRl{y:\|\widehat T_{\widehat M_{\rm task}}\{s(x,y)\}\|
 \le\widehat q_{1-\alpha}\LRs{\LRl{\|\widehat T_{\widehat M_{\rm task}}\{s(X_i,Y_i)\}\|}_{i=1}^n\cup\{\infty\}}}.\nonumber
\end{equation}
Both the metric and forward map are fixed before calibration, so exchangeability of the calibration and test scalar scores gives the usual split-conformal marginal guarantee.

Now consider the computation details. For an arbitrary pair $(M,\theta)$, define the transported point in the original score coordinates and its stabilized Jacobian weight by
\begin{align}
\label{eq:general_task_transport_point}
    \widehat G_{M,\theta}(u)
    &:={}
    M^{-1/2}\Psi_\theta(M^{1/2}u),\\
\label{eq:jacobian_surrogate}
    w_{M,\theta}(u)
    &:={}
    \frac{\one\{u\in\sB_{1-\alpha}\}}{f_Q(u)}
    \det\{\nabla\Psi_\theta(M^{1/2}u)+\delta_N I_{{d}}\}.
\end{align}
With $X_1,\ldots,X_N$ and $U_1,\ldots,U_K$, the parameterized version is
\[
    \widehat\gL_i(z;M,\theta)
    =
    \frac1K\sum_{j=1}^K
    w_{M,\theta}(U_j)
    \ell\{z,\widehat\mu(X_i)+\widehat G_{M,\theta}(U_j)\}.
\]
Let $\widehat z_i^\star(M,\theta)$ minimize this criterion over $z\in\gZ$, and set
\[
    \widehat\gJ_{\rm task}(M,\theta)
    =
    \frac1N\sum_{i=1}^N
    \widehat\gL_i\{\widehat z_i^\star(M,\theta);M,\theta\}.
\]
The implemented outer problem is
\[
    \widehat M_{\rm task}
    \in
    \argmin_{M\in\gM_0}
    \widehat\gJ_{\rm task}\{M,\widehat\theta(M)\},
\]
where $\widehat\theta(M)$ is obtained from \eqref{eq:global_reverse_gradnet_loss}. This is a bilevel problem because both the transported points and their Jacobian weights depend on the fitted map.

The metric derivative must retain this dependence. In particular,
\begin{align}
\label{eq:total_transport_derivative}
    {\nabla_M \widehat G_{M,\widehat\theta(M)}(u)}
    =
    \left.
    \left\{
        \partial_M \widehat G_{M,\theta}(u)
        +\partial_\theta \widehat G_{M,\theta}(u)
        \nabla_M\widehat\theta(M)
    \right\}
    \right|_{\theta=\widehat\theta(M)},
\end{align}
and the same decomposition applies to $w_{M,\widehat\theta(M)}(u)$. The direct term includes the derivatives of $M^{-1/2}$ and of the input $M^{1/2}u$; the response term accounts for the change in the fitted map. It can be evaluated by the implicit system in Appendix~\ref{subsec:tacp_implementation} or by differentiating through the inner optimization.

Let $\widehat y_{ij}=\widehat\mu(X_i)+\widehat G_{M,\widehat\theta(M)}(U_j)$, $\widehat w_j=w_{M,\widehat\theta(M)}(U_j)$, and $\widehat z_i^\star=\widehat z_i^\star\{M,\widehat\theta(M)\}$. {If the required derivatives exist, the chain rule gives}
\begin{align}
\label{eq:general_task_full_gradient}
    \nabla_M\widehat\gJ_{\rm task}
    \{M,\widehat\theta(M)\}
    ={}&\frac{1}{NK}\sum_{i=1}^N\sum_{j=1}^K
    \Bigl[
        \ell(\widehat z_i^\star,\widehat y_{ij})
        \nabla_M\widehat w_j+\widehat w_j
        \nabla_y\ell(\widehat z_i^\star,\widehat y_{ij})^\top
        \nabla_M\widehat y_{ij}+\widehat w_j
        \nabla_z\ell(\widehat z_i^\star,\widehat y_{ij})^\top
        \nabla_M\widehat z_i^\star
    \Bigr].
\end{align}
If $\widehat z_i^\star$ is an exact interior minimizer of $\widehat\gL_i\{z;M,\widehat\theta(M)\}$, its first-order optimality condition is
\begin{align*}
    \nabla_z\widehat\gL_i
    \{\widehat z_i^\star;M,\widehat\theta(M)\}
    =\frac1K\sum_{j=1}^K
        \widehat w_j
        \nabla_z\ell(\widehat z_i^\star,\widehat y_{ij})
    =0.
\end{align*}
Since all $K$ quadrature points at $X_i$ share the same decision $\widehat z_i^\star$, the decision-response terms vanish after summing over $j$:
\begin{align*}
    &\frac1K\sum_{j=1}^K
        \widehat w_j
        \nabla_z\ell(\widehat z_i^\star,\widehat y_{ij})^\top
        \nabla_M\widehat z_i^\star=
    \left[
        \frac1K\sum_{j=1}^K
        \widehat w_j
        \nabla_z\ell(\widehat z_i^\star,\widehat y_{ij})
    \right]^\top
    \nabla_M\widehat z_i^\star
    =0.
\end{align*}
Thus no derivative of the optimal decision needs to be evaluated when differentiating the profiled task objective. The envelope theorem yields

\begin{align}
\label{eq:general_task_total_gradient}
    \nabla_M\widehat\gJ_{\rm task}
    \{M,\widehat\theta(M)\}
    =\frac{1}{NK}\sum_{i=1}^N\sum_{j=1}^K
    \left[
        \ell(\widehat z_i^\star,\widehat y_{ij})
        \nabla_M\widehat w_j
        +\widehat w_j
        \nabla_y\ell(\widehat z_i^\star,\widehat y_{ij})^\top
        \nabla_M\widehat y_{ij}
    \right].
\end{align}

We now use this gradient to update the metric. As in MV-OTCP, parameterize $M(B)=\exp(B)$ with $B\in\gB_0$, and apply the chain rule through the matrix exponential. At outer iteration $t$, set
\begin{align}
\label{eq:task_metric_update}
    M_t=\exp(B_t),\quad
    g_t=\left.\nabla_B\widehat\gJ_{\rm task}
    \{M(B),\widehat\theta(M(B))\}\right|_{B=B_t},\quad
    B_{t+1}=\Pi_{\gB_0}\{B_t-\gamma_t g_t\}.
\end{align}
Each iteration fits the reverse transport at $M_t$, solves one decision problem for each estimation covariate using all $K$ quadrature points, and computes the total metric derivative in \eqref{eq:general_task_total_gradient}. The decisions are held fixed when evaluating this derivative, while the fitted transport response is retained. For OT fitting, the reference average in \eqref{eq:global_reverse_gradnet_loss} uses the same $K$ reference draws in place of $N$, with normalization $1/(2K)$.

After $T$ outer updates, set $\widehat M_{\rm task}=\exp(B_T)$ and refit the reverse map at this fixed metric. The learned metric can be inserted into the full-conformal assignment \eqref{eq:final_assignment}. For the split TA-OTCP construction used in the task experiments, fit $\widehat T_{\widehat M_{\rm task}}$ on the estimation scores by the score-to-reference procedure in Appendix~\ref{sec:itacp_implementation}, using uniform source weights $1/N$ and the global metric $\widehat M_{\rm task}$. Freeze the predictor, metric, and forward map before using the independent calibration sample. With $R_i=\|\widehat T_{\widehat M_{\rm task}}\{s(X_i,Y_i)\}\|$ for calibration observations, 
\begin{align}
\label{eq:task_split_calibration}
    \widehat\gC^{\rm TA}(x)
    =\{y\in\gY:\|\widehat T_{\widehat M_{\rm task}}\{s(x,y)\}\|
    \le\widehat q_{1-\alpha}(\{R_i\}_{i=1}^n\cup\{\infty\})\}.
\end{align}
Algorithm~\ref{alg:task_metric_learning} summarizes the metric updates.

\begin{algorithm}[H]
\caption{Task-aware metric learning for TA-OTCP}
\label{alg:task_metric_learning}
\begin{algorithmic}[1]
\Require Pretrained predictor $\widehat\mu$; independent estimation and calibration samples; reference draws $U_1,\ldots,U_K\sim\gQ$; loss $\ell$ and decision set $\gZ$; KDE bandwidth $b$; $B_0\in\gB_0$; outer iterations $T$, inner steps $K_{\rm in}$, step sizes $\gamma_t$; level $\alpha$
\State Compute estimation scores $S_i=s(X_i,Y_i)$, $i\in[N]$.
\For{$t=0,\ldots,T-1$}
    \State Set $M_t=\exp(B_t)$ and form the transformed-score KDE $\widehat f_{P,M_t}$.
    \State Take $K_{\rm in}$ warm-started steps for the reverse OT loss in \eqref{eq:global_reverse_gradnet_loss}, averaged over $U_1,\ldots,U_K$; denote the fitted parameters by $\theta_t$.
    \State Evaluate $\widehat G_{M_t,\theta_t}(U_j)$ and $w_{M_t,\theta_t}(U_j)$ for $j\in[K]$.
    \For{$i=1,\ldots,N$}
        \State Solve $\widehat z_{i,t}^\star\in\argmin_{z\in\gZ}\widehat\gL_i(z;M_t,\theta_t)$.
    \EndFor
    \State Compute the total metric derivative using \eqref{eq:general_task_total_gradient}, retaining the fitted transport response and holding the solved decisions fixed.
    \State Apply the chain rule through $M(B)=\exp(B)$ to obtain $g_t$.
    \State Update $B_{t+1}=\Pi_{\gB_0}\{B_t-\gamma_t g_t\}$.
\EndFor
\State Set $\widehat M_{\rm task}=\exp(B_T)$.
\Ensure Learned metric $\widehat M_{\rm task}$.
\end{algorithmic}
\end{algorithm}

\begin{example}[Quadratic task with a closed-form decision]
\label{ex:quadratic_general_task}
Let $H=H^\top\succ0$, $\gZ=\R^{{d}}$, and
\[
    \ell_H(z,y)=\frac12(z-y)^\top H(z-y).
\]
Writing $\widehat y_{ij}(M)=\widehat\mu(X_i)+\widehat G_M(U_j)$ and $\widehat w_j(M)=\widehat w_M(U_j)$, the empirical decision is the weighted centroid
\begin{align*}
    \widehat z_i^\star(M)
    =
    \frac{\sum_{j=1}^K\widehat w_j(M)
    \widehat y_{ij}(M)}
    {\sum_{j=1}^K\widehat w_j(M)}.
\end{align*}
The decision is well defined whenever the denominator is positive. Thus the downstream optimization is available in closed form, although the task criterion still depends on the metric through the fitted transport map and its Jacobian.
\end{example}


\subsection{Implementation details of IMV-OTCP}
\label{sec:itacp_implementation}
IMV-OTCP differs at the map-fitting stage: for each fixed $x$, it first parameterizes a determinant-one metric by an unconstrained matrix $A_x$,
\[
\begin{aligned}
    {B_{x}}
    &=
    \frac{A_x+A_x^\top}{2}
    -
    \frac{\operatorname{tr}(A_x)}{{d}}I_{{d}},
    &
    M_x
    &=
    \exp({B_{x}}).
\end{aligned}
\]
Then ${B_{x}}={B_{x}}^\top$, $\operatorname{tr}({B_{x}})=0$, $M_x\succ0$, and $\det(M_x)=1$. {For a nonnegative covariate-localization kernel $K_{\tilde h}$ with bandwidth $\tilde h$, set $w_{i,\tilde h}(x)=K_{\tilde h}(X_i,x)/\sum_{k=1}^N K_{\tilde h}(X_k,x)$.} The local empirical law is discrete, whereas GradNetOT is formulated for continuous source and target densities. We therefore use the weighted KDE of $(M^{1/2})_\#\gP_x$:
\begin{align}
    \widehat f_{P,x,M}(z)
    =
    \sum_{i=1}^N
    w_{i,\tilde h}(x)
    \kappa_b(z-M^{1/2}S_i),
    \qquad
    \kappa_b(v)=b^{-{d}}\kappa(v/b).\nonumber
\end{align}
Its score-space bandwidth $b$ is distinct from the covariate-localization bandwidth $\tilde h$. Let $\Psi_{\theta,x}$ be a GradNet that maps the transformed reference distribution to this transformed local score distribution. Following GradNetOT \citep{chaudhari2025gradnetot}, we fit it through the Monge--Amp\`ere density identity:
\begin{align}
\label{eq:pointwise_reverse_gradnet_loss}
    \widehat\gL_x^{\leftarrow}(M,{\theta})
    =
    \frac{1}{2N}\sum_{j=1}^N
    \Big[
        \log\det\{
            \nabla \Psi_{\theta,x}(M^{1/2}U_j)
        \}
        -
        \log f_Q(U_j)
        +
        \log\widehat f_{P,x,M}
        \{\Psi_{\theta,x}(M^{1/2}U_j)\}
    \Big]^2 .
\end{align}
This loss directly parameterizes the monotone gradient map and matches its log-Jacobian determinant to the log density ratio. For a fixed $M$, let $\widehat\theta_x(M)$ minimize \eqref{eq:pointwise_reverse_gradnet_loss}. The fitted map in the original coordinates is
\begin{align}
    \label{eq:fitted_ref_source_map}
    \widehat G_{M,x}(u)
    =
    M^{-1/2}
    \Psi_{\widehat\theta_x(M),x}(M^{1/2}u).
\end{align}

The empirical counterpart of \eqref{eq:localized_metric_target} is
\[
    \widehat\gJ_x(M,{\theta})
    =
    \frac{1}{N}\sum_{j=1}^N
    \frac{
        \one\{U_j\in\sB_{1-\alpha}\}
        \det\left\{
            \nabla \Psi_{\theta,x}(M^{1/2}U_j)
            +\delta_N I_{{d}}
        \right\}
    }{f_Q(U_j)}.
\]
The pointwise metric estimate is
\[
    \widehat M_x
    \in
    \argmin_{M\in\gM_0}
    \widehat\gJ_x
    \{M,\widehat\theta_x(M)\}.
\]
Because the fitted reference-to-score map depends on $M_x$ through the profiled GradNetOT objective, we compute the total derivative of the empirical volume by implicitly differentiating through the first-order optimality conditions defining $\widehat\theta_x(M_x)$, following the implementation used for MV-OTCP in Appendix~\ref{subsec:tacp_implementation}. Thus, the metric update accounts for the response of the fitted transport map to a change in $M_x$.

After selecting $\widehat M_x$, fit a second GradNet $F_{{\psi}}$ directly from the transformed local score law to the transformed reference law by minimizing
\begin{align}
\widehat\gL_{x,{\rm MA}}^{\rightarrow}({\psi})
=
\frac{1}{2}
\sum_{i=1}^N w_{i,\tilde h}(x)
\Biggl[
    &\log\det\left\{
        \nabla F_{{\psi}}(\widehat M_x^{1/2}S_i)
    \right\}
    -\log\widehat f_{P,x,\widehat M_x}
    (\widehat M_x^{1/2}S_i)
    \notag\\
&+\log f_Q\left\{
        \widehat M_x^{-1/2}
        F_{{\psi}}(\widehat M_x^{1/2}S_i)
    \right\}
\Biggr]^2 .
\end{align}
The original Monge--Amp\`ere objective does not explicitly enforce the learned reference-to-score and score-to-reference maps to be mutual inverses. We therefore introduce inverse- and cycle-consistency losses to encourage this property. In transformed coordinates, write {$\bar G_x=\Psi_{\widehat\theta_x(\widehat M_x),x}$} for the frozen reference-to-score GradNet obtained during metric learning, {$\bar T_{\psi,x}=F_\psi$} for the score-to-reference GradNet being fitted, $\bar U_j=\widehat M_x^{1/2}U_j$, and $\bar S_i=\widehat M_x^{1/2}S_i$. We use
\begin{align*}
\widehat\gL_{x,{\rm inv}}(\psi)
&=\frac1N\sum_{j=1}^N
\|\bar T_{\psi,x}\{\bar G_x(\bar U_j)\}-\bar U_j\|_2^2,\\
\widehat\gL_{x,{\rm cyc}}(\psi)
&=\sum_{i=1}^N w_{i,\tilde h}(x)
\|\bar G_x\{\bar T_{\psi,x}(\bar S_i)\}-\bar S_i\|_2^2,
\end{align*}
and fit the final map with
\begin{align}
\label{eq:pointwise_forward_gradnet_loss}
\widehat\gL_{x,T}(\psi)
=
\widehat\gL_{x,{\rm inv}}(\psi)
+\widehat\gL_{x,{\rm cyc}}(\psi)
+\widehat\gL_{x,{\rm MA}}^{\rightarrow}(\psi),
\end{align}
Let $\widehat\psi_x$ minimize \eqref{eq:pointwise_forward_gradnet_loss}. The fitted score-to-reference map is
\[
    \widehat T_{\widehat M_x,x}(s)
    =
    \widehat M_x^{-1/2}
    F_{\widehat\psi_x}(\widehat M_x^{1/2}s).
\]
The empirical procedure is summarized in Algorithm~\ref{alg:itacp}.
\begin{algorithm}[H]
\caption{Individualized volume-based metric learning for IMV-OTCP}
\label{alg:itacp}
\begin{algorithmic}[1]
\Require Independent training, estimation, and calibration samples; common
reference draws; kernel $K_{\tilde h}$; level $\alpha$
\State Fit the score function $s$ on the training sample and compute the
estimation scores.
\For{each calibration covariate $X_i$}
    \State Form the weighted local estimation law and KDE at $X_i$.
    \For{each outer metric-learning iteration}
        \State Fit the reference-to-score GradNet by minimizing
        \eqref{eq:pointwise_reverse_gradnet_loss}.
        \State Update $M$ through the profiled local task objective.
    \EndFor
    \State Fix $\widehat M_{X_i}$ and fit the score-to-reference GradNet by
    minimizing \eqref{eq:pointwise_forward_gradnet_loss}.
    \State {Compute $\widehat R_N(X_i,Y_i)$ using the fitted map.}
\EndFor
\State Compute the global split-conformal cutoff in \eqref{eq:itacp_set}.
\State At a test covariate $x$, learn $\widehat M_x$ and construct
$\widehat T_{\widehat M_x,x}$ by the same pointwise rule.
\Ensure The set \eqref{eq:itacp_set}; reuse $(\widehat M_x,\widehat T_{\widehat M_x,x})$
across all trial responses at that $x$.
\end{algorithmic}
\end{algorithm}

\begin{remark}
A full-conformal extension of the MV-OTCP construction to localized transport is provided in Appendix~\ref{app:weighted_full_conformal}. We use the split-conformal version in \eqref{eq:itacp_set} for computational efficiency, since the full construction requires repeated sample-dependent Kantorovich solves across calibration points, test points, and trial responses. The split construction fixes the localized score rule before calibration,  enabling anchor-based amortization of the pointwise transport fits.
\end{remark}

\section{Polyhedral representation of MV-OTCP}
\label{subsec:polyhedral_represent}
Fix the calibration scores and regard the trial score as a query $s\in\R^{{d}}$. If the query is assigned to $U_j$, the remaining calibration scores are optimally matched to the other reference points with value
\begin{align*}
    \kappa_{j,M}
    =
    \min_{\substack{\sigma:[n]\to[n+1]\setminus\{j\}\\\sigma\ \mathrm{bijective}}}
    \sum_{i=1}^n
    (S_i-U_{\sigma(i)})^\top M(S_i-U_{\sigma(i)}).
\end{align*}
Let $f_{j,M}(s)=(s-U_j)^\top M(s-U_j)+\kappa_{j,M}$. The index $j$ is an optimal destination for the query if and only if $s$ belongs to the candidate cell
\[
    \overline\gR_{j,M}
    =
    \{s:f_{j,M}(s)\le f_{k,M}(s),\ k\ne j\}.
\]
When several indices are optimal, the fixed tie rule partitions the overlap into disjoint cells $\gR_{j,M}^{\rm tie}\subseteq\overline\gR_{j,M}$.

\begin{proposition}[Polyhedral representation of MV-OTCP]
\label{prop:metric_polyhedra}
For any fixed $M=M^\top\succ0$, the candidate cell $\overline\gR_{j,M}$ is the convex polyhedron
\begin{align}
    \overline\gR_{j,M}
    =
    \left\{
        s\in\R^{{d}}:
        \langle s,M(U_k-U_j)\rangle
        \le
        \beta_{j,k}(M),
        \quad k\ne j
    \right\},\nonumber
\end{align}
where $\beta_{j,k}(M)=\frac{1}{2} \left\{ U_k^\top M U_k - U_j^\top M U_j + \kappa_{k,M} - \kappa_{j,M} \right\}$. Let $\gI_\alpha$ denote the indices of the accepted reference points,
\(
    \gI_\alpha
    =
    \left\{j:\|U_j\|\le
    \widehat q_{1-\alpha}\LRs{\{\|U_i\|\}_{i=1}^{n+1}}\right\}.
\)
Under any deterministic priority rule for ties, $\gR_{j,M}^{\rm tie}$ is a possibly half-open polyhedral cell obtained by assigning the shared boundaries of the candidate cells. Consequently, the exact accepted score region is the finite union of polyhedral cells
\[
    {\widehat\gA}_{M}
    =
    \bigcup_{j\in\gI_\alpha}\gR_{j,M}^{\rm tie},
\]
For $M=\widehat M$, the resulting MV-OTCP set is ${\widehat\gC}^{\rm MV}(x)=\widehat\mu(x)+{\widehat\gA}_{\widehat M}$.
\end{proposition}


\begin{proof}
By definition, $\overline\gR_{j,M}$ is the set of query scores for which assigning the query to $U_j$ has no larger augmented OT objective than assigning it to any other reference point. Thus
\begin{align*}
    \overline\gR_{j,M}
    =
    \{s:f_{j,M}(s)\le f_{k,M}(s),\ k\ne j\}.
\end{align*}
For each $k\ne j$, the inequality $f_{j,M}(s)\le f_{k,M}(s)$ is
\begin{align*}
    (s-U_j)^\top M(s-U_j)+\kappa_{j,M}
    \le
    (s-U_k)^\top M(s-U_k)+\kappa_{k,M}.
\end{align*}
Expanding both quadratic terms gives
\begin{align*}
    s^\top Ms
    -2s^\top M U_j
    +U_j^\top M U_j
    +\kappa_{j,M}
    \le
    s^\top Ms
    -2s^\top M U_k
    +U_k^\top M U_k
    +\kappa_{k,M}.
\end{align*}
The common term $s^\top Ms$ cancels. Rearranging the remaining terms gives
\begin{align*}
    2s^\top M(U_k-U_j)
    \le
    U_k^\top M U_k
    -
    U_j^\top M U_j
    +
    \kappa_{k,M}
    -
    \kappa_{j,M},
\end{align*}
which is exactly the halfspace inequality in Proposition~\ref{prop:metric_polyhedra}. Intersecting these finitely many halfspaces over $k\ne j$ proves that $\overline\gR_{j,M}$ is a convex polyhedron.

The deterministic tie rule assigns every common boundary to exactly one minimizing index. It therefore replaces weak by strict inequalities on selected facets, producing disjoint, possibly half-open polyhedral cells $\gR_{j,M}^{\rm tie}$. {For the accepted index set $\gI_\alpha$, the exact accepted score region is $\bigcup_{j\in\gI_\alpha}\gR_{j,M}^{\rm tie}$.} Finally, the affine map $s\mapsto\widehat\mu(x)+s$ preserves polyhedral cells and finite unions.
\end{proof}

\section{Weighted full-conformal localized OTCP}
\label{app:weighted_full_conformal}

This appendix extends the assignment construction of Section~\ref{sec:tacp} to localized transport. Localization weights enter the source marginal of an augmented Kantorovich problem, and a learned local metric yields a full-conformal alternative to the split IMV-OTCP method in Section~\ref{sec:itacp}. We first describe the Euclidean-cost construction, then insert the learned metric.

\subsection{Euclidean-cost weighted construction}
\label{subsec:weighted_localized_otcp}

Directly centering a local procedure at the distinguished test covariate $X_{n+1}$ destroys the symmetry of the augmented sample. Inspired by \cite{hore2024conformal}, we instead let $H_h(\tilde x\mid x)$ be a localization density and draw a randomized center $ \widetilde X_{n+1}\mid X_{n+1}\sim H_h(\cdot\mid X_{n+1})$.
For a realized center $\tilde x$, define
\begin{align}
\label{eq:lotcp_weights}
    p_i(\tilde x)
    =
    \frac{H_h(\tilde x\mid X_i)}
    {\sum_{k=1}^{n+1}H_h(\tilde x\mid X_k)},
    \qquad i\in[n+1].
\end{align}
Conditional on the unordered augmented sample and $\widetilde X_{n+1}=\tilde x$, the posterior probability that observation $i$ occupies the test position is $p_i(\tilde x)$. The same weights therefore serve as the source masses in the localized transport problem.

For each trial response $y$, let $\widehat\pi_{\tilde x}^{\,y} =(\widehat\pi_{\tilde x,ij}^{\,y})_{i,j\in[n+1]}$ solve
\begin{align}
\label{eq:lotcp_transport}
\begin{aligned}
    \widehat\pi_{\tilde x}^{\,y}
    &\in
    \argmin_{\pi\ge0}
    \sum_{i=1}^{n+1}\sum_{j=1}^{n+1}
    \pi_{ij}\|S_i^y-U_j\|^2,\\
    \text{subject to}\qquad
    \sum_{j=1}^{n+1}\pi_{ij}
    &=p_i(\tilde x),\quad i\in[n+1],
    \qquad
    \sum_{i=1}^{n+1}\pi_{ij}=\frac{1}{n+1},\quad j\in[n+1].
\end{aligned}
\end{align}
When $p_i(\tilde x)=1/(n+1)$ for every $i$, an optimal plan can be chosen as the scaled FOTCP assignment. For nonuniform localization weights, \eqref{eq:lotcp_transport} is generally a transportation linear program and a source row may split its mass across several reference points. This leads to two natural conformal constructions.

\paragraph{Randomized conditional transport.} For a positive test-row mass, normalize the last row of the plan,
$a_{\tilde x,j}^y=\frac{\widehat\pi_{\tilde x,n+1,j}^{\,y}}{p_{n+1}(\tilde x)},\, j\in[n+1]$,
and draw $J_{\tilde x}^y\mid\widehat\pi_{\tilde x}^{\,y}\sim\operatorname{Categorical}\{a_{\tilde x,1}^y,\ldots,a_{\tilde x,n+1}^y\}$.
The randomized LOTCP set is
\begin{align*}
    {\widehat\gC}^{\rm LOTCP-rand}
    (X_{n+1},\widetilde X_{n+1})
    =
    \left\{
        y\in\gY:
        \|U_{J_{\widetilde X_{n+1}}^y}\|
        \le
        \widehat q_{1-\alpha}\LRs{\{\|U_i\|\}_{i=1}^{n+1}}
    \right\}.
\end{align*}
A single $V\sim\operatorname{Unif}(0,1)$ is used for inverse-cdf sampling across all trial responses, so the construction defines one coherent random set rather than unrelated pointwise decisions.

\paragraph{Barycentric transport.} Alternatively, for every positive-mass row define $\overline U_{\tilde x,i}^y=\frac{1}{p_i(\tilde x)}\sum_{j=1}^{n+1}\widehat\pi_{\tilde x,ij}^{\,y}U_j$, and $R_{\tilde x,i}^y=\|\overline U_{\tilde x,i}^y\|$.
Let
\begin{align*}
    \widehat q_{1-\alpha}^{\,y}(\tilde x)
    =
    \inf\left\{t:
        \sum_{i=1}^{n+1}p_i(\tilde x)
        \one\{R_{\tilde x,i}^y\le t\}
        \ge1-\alpha
    \right\}.
\end{align*}
The deterministic LOTCP set is
\begin{align*}
    {\widehat\gC}^{\rm LOTCP-bar}
    (X_{n+1},\widetilde X_{n+1})
    =
    \left\{
        y\in\gY:
        R_{\widetilde X_{n+1},n+1}^y
        \le
        \widehat q_{1-\alpha}^{\,y}(\widetilde X_{n+1})
    \right\}.
\end{align*}
The weighted quantile must be formed from the barycentric radii themselves. In general, the barycenters do not follow $\gQ$, so the empirical reference quantile used by the randomized construction is not valid here.

The quadratic cost also gives the localized construction a finite polyhedral representation. Fix $\tilde x$ and write $p_i=p_i(\tilde x)$. For the following result, let $M=M^\top\succ0$ be fixed and replace the Euclidean cost in \eqref{eq:lotcp_transport} by $c_M(s,u)=(s-u)^\top M(s-u)$. The ordinary LOTCP construction corresponds to $M=I$, while the same result applies to IMV-OTCP after its local metric has been fixed. Let $\Pi_{\tilde x}$ denote the transportation polytope defined by the marginal constraints in \eqref{eq:lotcp_transport}, and let $\mathfrak B_{\tilde x}$ be a finite index set for its basic feasible plans, written $\{\pi^b:b\in\mathfrak B_{\tilde x}\}$. 

For every $b\in\mathfrak B_{\tilde x}$, define the unnormalized first moment of its test row and its query-free objective term by
\begin{align*}
    m_b
    =\sum_{j=1}^{n+1}\pi^b_{n+1,j}U_j,\quad
    \kappa_{b,M}
    =\sum_{i=1}^{n}\sum_{j=1}^{n+1}
        \pi^b_{ij}c_M(S_i,U_j)
      +\sum_{j=1}^{n+1}\pi^b_{n+1,j}U_j^\top M U_j.
\end{align*}
The optimality cell associated with $b$ is
\begin{align*}
    \gR_{b,M}(\tilde x)
    =
    \bigcap_{b'\in\mathfrak B_{\tilde x}}
    \left\{
        s\in\R^{{d}}:
        2s^\top M(m_{b'}-m_b)
        \le \kappa_{b',M}-\kappa_{b,M}
    \right\}.
\end{align*}
For the randomized construction and the common draw $V\sim\operatorname{Unif}(0,1)$, set
\begin{align*}
    J_b(V)
    =
    \min\left\{j:
        \frac{1}{p_{n+1}}
        \sum_{\ell=1}^j\pi^b_{n+1,\ell}\ge V
    \right\}.
\end{align*}
For the barycentric construction, set $\overline U_i^b=\frac{1}{p_i}\sum_{j=1}^{n+1}\pi^b_{ij}U_j$, $R_i^b=\|\overline U_i^b\|$ and
\begin{align*}
    q_b
    &=\inf\left\{t:
        \sum_{i=1}^{n+1}p_i\one\{R_i^b\le t\}\ge1-\alpha
    \right\}.
\end{align*}

\begin{proposition}[Polyhedral representation of localized OTCP]
\label{prop:lotcp_polyhedra}
Suppose the optimizer of the localized transportation problem is selected from its basic feasible plans. Conditional on $\tilde x$, the reference cloud, and $V$ for the randomized construction, each $\gR_{b,M}(\tilde x)$ is a convex polyhedron. Apart from optimal-basis tie boundaries, the two accepted score regions are respectively
\begin{align*}
    {\widehat\gA}_{M}^{\rm LOTCP-rand}(\tilde x,V)
    &=
    \bigcup_{\substack{b\in\mathfrak B_{\tilde x}:\\
        \|U_{J_b(V)}\|\le
        \widehat q_{1-\alpha}\LRs{\{\|U_i\|\}_{i=1}^{n+1}}}}
        \gR_{b,M}(\tilde x),\\
    {\widehat\gA}_{M}^{\rm LOTCP-bar}(\tilde x)
    &=
    \bigcup_{\substack{b\in\mathfrak B_{\tilde x}:\\
        R_{n+1}^b\le q_b}}
        \gR_{b,M}(\tilde x).
\end{align*}
Under a deterministic tie-breaking rule, assigning each tie boundary to its selected adjacent cell gives the corresponding exact, possibly half-open, finite polyhedral union. The prediction sets are obtained by translating these score regions by $\widehat\mu(X_{n+1})$.
\end{proposition}

Problem \eqref{eq:lotcp_transport} is an unregularized discrete Kantorovich linear program and can be solved exactly by specialized transportation or minimum-cost-flow algorithms; in the equal-mass assignment special case, auction or Hungarian algorithms apply \citep{PeyreCuturi2019}. This does not, however, remove the dependence on the trial response. Changing $y$ changes the last row of the cost matrix, and may change the optimal transport basis. Thus a direct pointwise implementation must reoptimize the transport problem for each trial $y$. A parametric or warm-started transportation solver can reuse the current basis until a cell boundary is crossed, but enumerating all cells in $\mathfrak B_{\tilde x}$ in advance can be combinatorial and need not be cheaper than repeated optimization.

\begin{theorem}[Finite-sample validity of weighted full-conformal LOTCP]
\label{thm:weighted_lotcp_validity}
Suppose the $\{X_i,Y_i\}_{i=1}^{n+1}$ are exchangeable, the score construction is permutation-equivariant at the true response, the reference cloud is independent, and the solution of \eqref{eq:lotcp_transport} is selected permutation-equivariantly. Then
\begin{align*}
    \sP\{Y_{n+1}\in{\widehat\gC}^{\rm LOTCP-rand}
    (X_{n+1},\widetilde X_{n+1})\}
    &\ge1-\alpha,\\
    \sP\{Y_{n+1}\in{\widehat\gC}^{\rm LOTCP-bar}
    (X_{n+1},\widetilde X_{n+1})\}
    &\ge1-\alpha.
\end{align*}
The first probability also averages over the common auxiliary draw $V$.
\end{theorem}

These are finite-sample marginal guarantees, where the probability averages over both the data and the randomized center. 

\subsection{Task-aware metric extension}
\label{subsec:weighted_itacp}

We next record the task-aware extension of the weighted construction. Let $\gD_{\rm cal}=\{(X_i,Y_i)\}_{i=1}^n$ be independent of the data used to learn the metric, and let $U_{1:n+1}$ be an independent final reference cloud. Draw $\widetilde X_{n+1}$ and form the weights $p_i(\tilde x)$ as in \eqref{eq:lotcp_weights}. For a realized center $\tilde x$, use the nearest anchor metric
\[
    {\widehat M(\tilde x)=\widehat M_{a_{j(\tilde x)}}.}
\]
For each trial response $y$, the IMV-OTCP plan $\widehat\pi_{\tilde x}^{\,y,\widehat M}$ solves
\begin{align}
\label{eq:itacp_weighted_metric_ot}
\begin{aligned}
    \widehat\pi_{\tilde x}^{\,y,\widehat M}
    &\in
    \argmin_{\pi\ge0}
    \sum_{i=1}^{n+1}\sum_{j=1}^{n+1}
    \pi_{ij}(S_i^y-U_j)^\top
    \widehat M(\tilde x)(S_i^y-U_j),\\
    \text{subject to}\qquad
    \sum_{j=1}^{n+1}\pi_{ij}
    &=p_i(\tilde x),\quad i\in[n+1],
    \qquad
    \sum_{i=1}^{n+1}\pi_{ij}=\frac{1}{n+1},\quad j\in[n+1].
\end{aligned}
\end{align}
Thus IMV-OTCP differs from LOTCP only through the ground cost: the Euclidean cost in \eqref{eq:lotcp_transport} is replaced by $c_{\widehat M(\tilde x)}$, while the localization weights and target masses remain unchanged.

Accordingly, the two LOTCP acceptance rules carry over without introducing a second fitted transport map. For randomized conditional transport, draw $J_{\tilde x}^{y,\widehat M}$ from the normalized last row of $\widehat\pi_{\tilde x}^{\,y,\widehat M}$, with probabilities $\widehat\pi_{\tilde x,n+1,j}^{\,y,\widehat M}/p_{n+1}(\tilde x)$ for $j\in[n+1]$. The resulting set is
\begin{align*}
    {\widehat\gC}^{\rm IMV-rand}
    (X_{n+1},\widetilde X_{n+1})
    =
    \left\{
        y\in\gY:
        \|U_{J_{\widetilde X_{n+1}}^{y,\widehat M}}\|
        \le
        \widehat q_{1-\alpha}\LRs{\{\|U_i\|\}_{i=1}^{n+1}}
    \right\}.
\end{align*}
For barycentric transport, define $\overline U_{\tilde x,i}^{y,\widehat M}=\frac{1}{p_i(\tilde x)}\sum_{j=1}^{n+1}\widehat\pi_{\tilde x,ij}^{\,y,\widehat M}U_j$ and $R_{\tilde x,i}^{y,\widehat M}=\|\overline U_{\tilde x,i}^{y,\widehat M}\|$.
Let $\widehat q_{1-\alpha}^{\,y,\widehat M}(\tilde x)$ be the weighted $(1-\alpha)$-quantile of these radii under the masses $p_i(\tilde x)$. The barycentric IMV-OTCP set is
\begin{align*}
    {\widehat\gC}^{\rm IMV-bar}
    (X_{n+1},\widetilde X_{n+1})
    =
    \left\{
        y\in\gY:
        R_{\widetilde X_{n+1},n+1}^{y,\widehat M}
        \le
        \widehat q_{1-\alpha}^{\,y,\widehat M}
        (\widetilde X_{n+1})
    \right\}.
\end{align*}

\begin{theorem}[Finite-sample validity of weighted full-conformal IMV-OTCP]
\label{thm:weighted_itacp_validity}
Suppose the assumptions of Theorem~\ref{thm:weighted_lotcp_validity} hold. Assume in addition that the estimation sample is independent of the augmented calibration-test sample, and that $\widehat M(\tilde x)$ is an estimation-sample-measurable function of $\tilde x$ applied identically to all source rows. Then
\begin{align*}
    \sP\{Y_{n+1}\in{\widehat\gC}^{\rm IMV-rand}
    (X_{n+1},\widetilde X_{n+1})\}
    &\ge1-\alpha,\\
    \sP\{Y_{n+1}\in{\widehat\gC}^{\rm IMV-bar}
    (X_{n+1},\widetilde X_{n+1})\}
    &\ge1-\alpha.
\end{align*}
\end{theorem}

\section{Additional experimental results}
\label{app:additional_exp}

\subsection{General task objectives}
\label{subsec:general_task_simulation}
We compare {TA-OTCP} with {SCP} and two decision-focused uncertainty quantification methods of \citet{cortes-gomez2025utilitydirected}, denoted {DFUQ-p} and {DFUQ-g}. {DFUQ-p} incorporates the downstream set loss as a penalty in the conformal score and tunes the penalty weight on a separate estimation sample, whereas {DFUQ-g} constructs a loss-aware ordering through a greedy optimization procedure and subsequently conformalizes the resulting score.
We consider the following two downstream objectives: a quadratic task loss for average RO in Example \ref{example:average_RO} and a smoothed contextual robust-optimization objective in Example \ref{example:smooth_CRO}. All methods use the same residual representation within a setting.

We consider the following two settings. {For a prediction set $\gC$, let $z_\gC(x)$ minimize the corresponding set-induced task criterion over $z$.}
\begin{itemize}[]
\item[(i)] \textit{Integrated quadratic task loss.} Let $X\sim\operatorname{Unif}[-1,1]$, $Y=\mu(X)+\varepsilon$, and $\mu(x)=(x,-x)^\top$. Draw $C\in\{1,2\}$ with probabilities $(2/3,1/3)$ and $D=\sqrt U(\cos\Theta,\sin\Theta)^\top$, where $U\sim\operatorname{Unif}[0,1]$ and $\Theta\sim\operatorname{Unif}[0,2\pi]$ are independent, and set $\varepsilon=m_C+A_CD$, where
\[
m_1=(-0.22,0)^\top,\quad m_2=(0.44,0)^\top,\quad A_1=\begin{pmatrix}0.85&0.20\\0.20&0.18\end{pmatrix},\quad A_2=\begin{pmatrix}0.70&0.16\\0.16&0.16\end{pmatrix}.
\]
For $Q=\operatorname{diag}(1,100)$, the scaled task loss is $\ell_{\rm I}(z,y)=10(z-y)^\top Q(z-y)$. The reported task objective is $\gJ_{\rm I}(\gC)=\E_X[{\int_{\gC(X)}}\ell_{\rm I}\{z_\gC(X),y\}\,dy]$.

\item[(ii)] \textit{Smoothed quadratic CRO.} Let $X\sim N\{(1,1)^\top,2.25I_2\}$, $Y=\mu(X)+\varepsilon$, $\mu(x)=(-x_1-x_2^2,-x_1^2-x_2)^\top$, and $\varepsilon\sim N(0,I_2)$. We use ${\ell_{\rm II}(z,y)}=\frac12(z-y)^\top H(z-y)$, where $z\in\R^2$ and $H=R_{45^\circ}\operatorname{diag}(20\cdot50/51,20/51)R_{45^\circ}^\top$ with $R_{45^\circ}$ the two-dimensional rotation matrix through $45^\circ$. The smoothed contextual robust objective is $\gJ_\tau(\gC)=\E_X[\tau\log{\int_{y\in\gC(X)}\exp\{\ell_{\rm II}(z_\gC(X),y)/\tau\}}\,dy]$ with $\tau=0.25$.

\end{itemize}


\paragraph{Evaluation metrics.}
Besides marginal coverage and set volume, we report the \textit{Task objective} $\widehat\gJ_{\rm I}({\widehat\gC})$ or $\widehat\gJ_{\tau}({\widehat\gC})$, evaluated at the set-induced decision that minimizes the corresponding criterion. Lower values mean that the prediction set is better aligned with the downstream loss. The estimation of the task objective can be found in Appendix \ref{app:task_implementation}.


\begin{figure}[h]
    \centering
    \includegraphics[width=\linewidth]{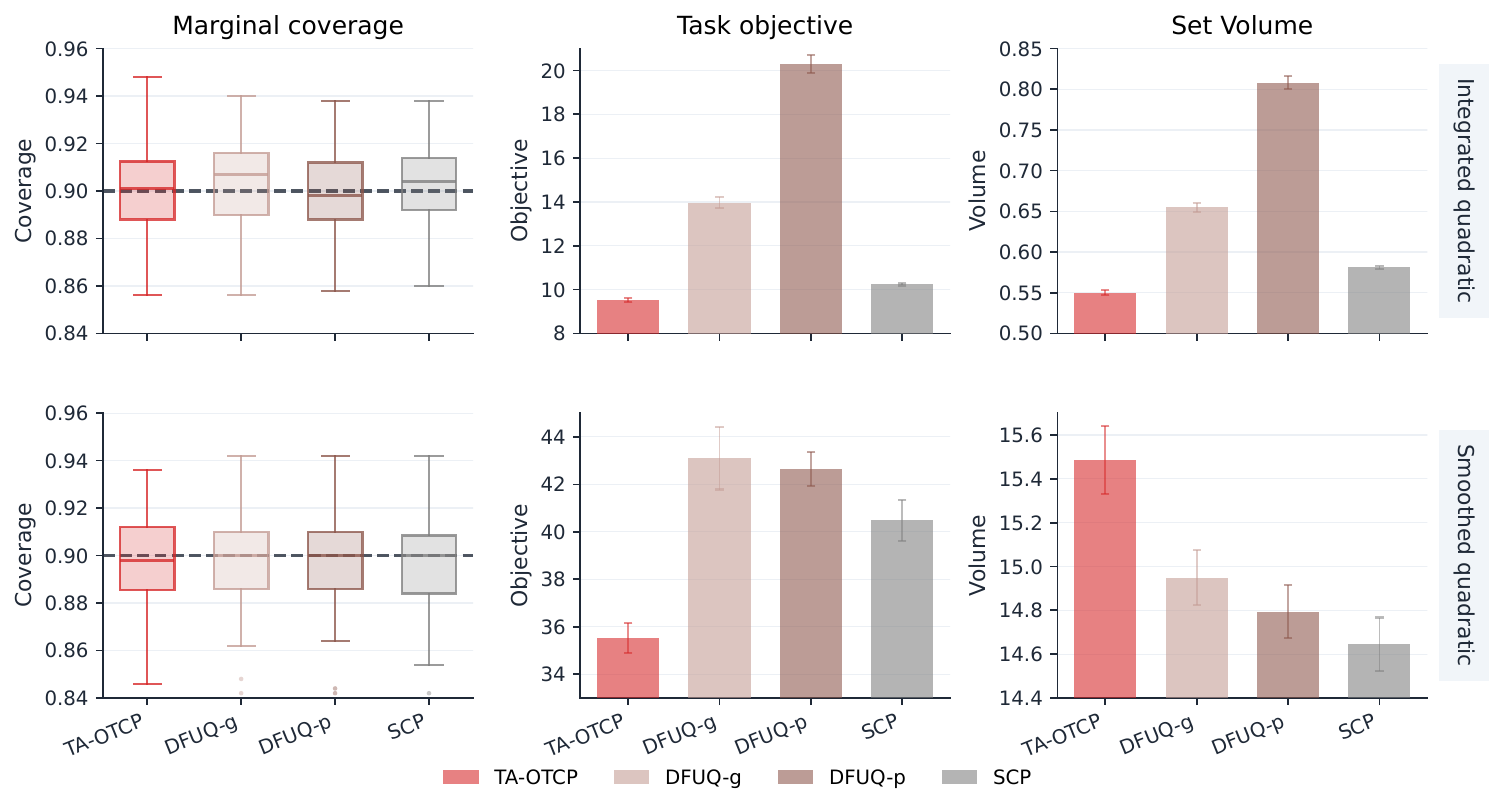}
    \caption{General task-objective experiments. The rows correspond to the integrated and smoothed quadratic settings. The dashed line in the marginal-coverage panels marks the nominal target; lower task objective and volume are preferable.}
    \label{tab:general_task_objectives}
\end{figure}

Figure~\ref{tab:general_task_objectives} shows that all methods remain close to the nominal coverage level. Across both downstream objectives, {TA-OTCP} achieves the smallest task objective, outperforming {SCP} and the two {DFUQ} baselines. The volume results further highlight the distinction between task alignment and set size: under the integrated quadratic objective, {TA-OTCP} also yields the smallest volume, whereas under the smoothed quadratic objective it produces a larger set despite attaining the best task loss. This illustrates that the learned transport geometry is adapting to the downstream objective itself, rather than simply favoring smaller prediction sets.


\subsection{U.S. stock portfolios}
\label{subsec:real_stock}

Following the stock experiment of \citet{chenreddy2022data}, we use daily U.S. equity observations from January 2012 through December 2017, where the response is the vector of percentage changes in adjusted closing prices and the covariates include individual trading volumes and broad market indices. In each replication, we randomly sample 15 stocks from the universe of 64 firms and apply all competing methods to the same selected subset. We use 2012--2014 for training, 2015 for task-aware metric estimation, 2016 for calibration, and 2017 for testing. 

For return vector $y\in\R^{15}$ and long-only portfolio $z\in\gZ:=\{z\ge0:\mathbf 1^\top z=1\}$, let
\begin{align}
    \ell(z,y)=-y^\top z,
    \qquad
    z_C(x)\in\argmin_{z\in\gZ}\sup_{y\in C(x)}\ell(z,y),
    \qquad
    r_C(x)=\sup_{y\in C(x)}\ell\{z_C(x),y\}.\nonumber
\end{align}
We compare against several existing methods at level $1-\alpha=0.9$. {SCP} uses an ellipsoidal score with covariance matrix $\widehat\Sigma$ estimated from the training data, while {CRO} is the ellipsoidal predict-then-calibrate method of \citet{sun2023predict}. We additionally include the end-to-end method of \citet{yeh2024end}, denoted {E2E}, which predicts a covariate-dependent Gaussian ellipsoid and differentiates the realized portfolio loss through both minibatch conformal calibration and the robust portfolio problem during training. Precise definitions are given in Appendix~\ref{app:real_stock}.

\paragraph{Evaluation metrics.}
Besides marginal coverage, we report the \textit{Certificate:}
${n_{\rm test}^{-1}\sum_{i=1}^{n_{\rm test}}}\allowbreak r_{\widehat C}(X_i)$ and the \textit{Decision loss:}
${n_{\rm test}^{-1}\sum_{i=1}^{n_{\rm test}}}(-Y_i^\top z_{\widehat C}(X_i))$.
The former averages the worst-case portfolio loss certified by the calibrated uncertainty set, whereas the latter evaluates the corresponding set-induced portfolio at the realized return. Lower values are better for both metrics.


\begin{figure}[h]
    \centering
    \includegraphics[width=\linewidth]{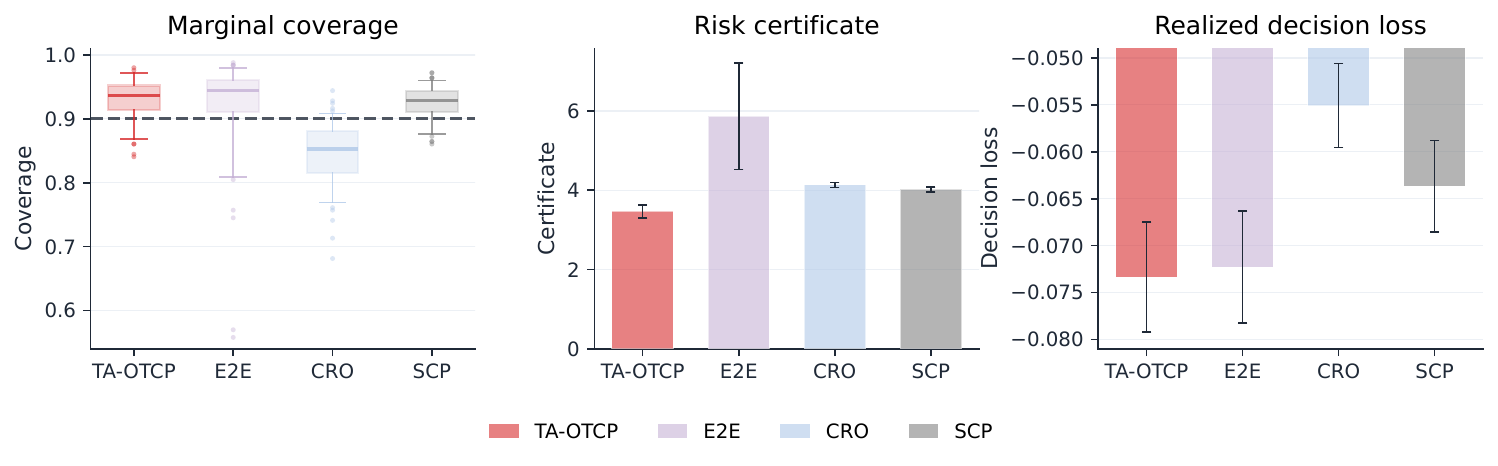}
    \caption{U.S. equity portfolio results for ellipsoidal uncertainty sets. The dashed line is the nominal level $0.9$. Certificate and decision loss are means with normal-approximation $95\%$ intervals (mean $\pm 1.96$ standard errors). Lower is better for the certificate and decision loss.}
    \label{fig:real_stock_results}
\end{figure}

Figure~\ref{fig:real_stock_results} reports the results. All methods exhibit some deviation from the nominal coverage level, which may be due to a certain degree of distribution shift in the data. Among the methods with coverage close to the nominal level, {TA-OTCP} achieves both the smallest mean certificate and the lowest mean decision loss, highlighting the benefit of tailoring the uncertainty-set geometry to the downstream portfolio objective.

\section{Additional experimental details}
\label{app:simulation_protocols}
\label{app:real_data}
Simulations use independent training, calibration, and test samples of sizes $1{,}000/1{,}000/500$, with the training sample also serving as the estimation sample for metric learning. Real-data experiments use separate training, estimation, calibration, and test splits; the exact splitting strategy can be found in Appendix~\ref{app:real_volume}. 
\subsection{Definition of evaluation metrics}
\label{app:eval_metrics}
We provide detailed definitions of the evaluation metrics used in our experiments below.
\begin{itemize}[leftmargin=2.3em]
    \item Worst-case coverage: Let ${\mathcal B_{\rm worst}}$ denote a collection of balls in $\gX$ with $|{\mathcal B_{\rm worst}}|=100$. Specifically, each ball is centered at a randomly selected test point, and its radius is drawn independently from a uniform distribution between $0.04$ and $0.12$ times the width of the covariate support. We report
    \begin{align*}
        \textit{Worst-case}=
    \min_{B\in{\mathcal B_{\rm worst}}}
    \frac{\sum_i\mathbbm{1}(X_i\in B)C_i}
    {\sum_i\mathbbm{1}(X_i\in B)} .
    \end{align*}

    \item $\sqrt{\mathrm{MSCE}}$: For simulations, since $X\sim\mathrm{Unif}(-2,2)$, we partition {the diagnostic support $[-2,2]$ into $40$ equal-width cells} ${\mathcal I_{\rm grid}}$. Because a single replication has only about $12$--$13$ test observations per cell, we pool the cell counts across all $R$ replications and define as:
    \begin{align*}
        \sqrt{\mathrm{MSCE}}_{\rm sim}
        =
        \left[
            \frac{1}{|{\mathcal I_{\rm grid}}|}
            \sum_{J\in{\mathcal I_{\rm grid}}}
            \left\{
                \frac{\sum_{r=1}^R\sum_i\mathbbm{1}\{X_{ri}\in J\}\mathbbm{1}\{Y_{ri}\notin\widehat C(X_{ri})\}}{\sum_{r=1}^R\sum_i\mathbbm{1}\{X_{ri}\in J\}}
                -\alpha
            \right\}^2
        \right]^{1/2}.
    \end{align*} 
    For real data, because regular grids are impractical in high-dimensional covariate spaces, we assess $\sqrt{\mathrm{MSCE}}$ using a mass-controlled random-ball diagnostic. For replication $r$, draw a center index $J_{r}$ randomly from the $n_{\rm test}$ test observations, and independently draw $P_{r}\sim\operatorname{Unif}(0.05,0.20)$. Set $K_{r}=\min\{n_{\rm test},\max(15,\lceil P_{r}n_{\rm test}\rceil)\}$ and let $B_{r}$ contain the $K_{r}$ test points nearest to $X_{J_{r}}$. Hence $B_{r}$ is the empirical $k$-nearest-neighbor ball with random probability mass. Let $\gB_{\rm msce}$ be the set of all $B_{r}$, with $|\gB_{\rm msce}|=100$. We report
    \begin{align*}
            \sqrt{\mathrm{MSCE}}_{{\rm real}}
            =
            \frac{1}{R}\sum_{r=1}^R\left[
                \frac{1}{|\gB_{\rm msce}|}
                \sum_{B_r\in\gB_{\rm msce}}
                \left\{
                    \frac{\sum_{i}\mathbbm{1}\{X_{ri}\in B_r\}\mathbbm{1}\{Y_{ri}\notin\widehat C(X_{ri})\}}{\sum_i\mathbbm{1}\{X_{ri}\in B_r\}}
                    -\alpha
                \right\}^2
            \right]^{1/2}.
    \end{align*} 
\end{itemize}

\subsection{Implementation details of homogeneous experiments}
\label{app:homogeneous_implementation}
Let $X\sim\operatorname{Unif}(-2,2)$ and generate
\begin{align*}
    Y=m(X)+\varepsilon,\qquad
    m(x)=
    \begin{pmatrix}
        1.2x+0.35\sin(\pi x)\\
        -0.8+0.45x^2+0.25\cos(\pi x)
    \end{pmatrix}.
\end{align*}
Let $R(\gamma)$ denote rotation through angle $\gamma$ and write
$e(\theta)=(\cos\theta,\sin\theta)^\top$. For the first two settings, let
$R_0=\{-2\log V\}^{1/2}$ with $V\sim\operatorname{Unif}(0,1)$, draw $\Theta$
with density proportional to $r_j(\theta)^2$, and set
$\varepsilon={A_j}\{R_0r_j(\Theta)e(\Theta)\}+\xi$, where
$\xi\sim N(0,0.015^2I_2)$. We consider the following three residual laws:

\begin{itemize}[leftmargin=2.3em]

\item[(i)] \textit{Five-star.}
\begin{align*}
    r_1(\theta)&=0.70\{1+0.43\cos(5(\theta-0.18))\},\\
    A_1^{(0)}&=R(0.34)\operatorname{diag}(1.55,0.68).
\end{align*}

\item[(ii)] \textit{Clover.}
\begin{align*}
    r_2(\theta)&=0.74\{1+0.40\cos(4(\theta+0.24))\},\\
    A_2^{(0)}&=R(-0.46)\operatorname{diag}(1.42,0.72).
\end{align*}

\item[(iii)] \textit{Exp Fixed.}
\begin{align*}
    r_3(\theta)&=0.58\{1+0.42\cos[3(\theta-0.27)]+0.12\cos(\theta-0.27)\},\\
    A_3&=R(-0.43)\operatorname{diag}(2,1/2).
\end{align*}
Draw $\Theta_3$ with density
$r_3(\theta)^2/\int_{-\pi}^{\pi}r_3(t)^2\,dt$
and independently draw $R_3\sim\operatorname{Gamma}(2,1)$.
The resulting density with respect to area is proportional to
$\exp(-R_3)$, so the population HPD region used as the oracle is available
in exact form.

\end{itemize}

All methods use the same residual $s(x,y)=y-\widehat\mu(x)$, where $\widehat\mu$ is a fifth-degree polynomial ridge regression fit on the training sample. The {SCP} baseline calibrates the scalar Mahalanobis score $\|\widehat\Sigma^{-1/2}s(x,y)\|_2$ with an estimated global covariance $\widehat\Sigma$. The one-matrix {MVCS} baseline jointly learns a global matrix $\widehat\Lambda$ and scalar exponent $\widehat p$ by minimizing its log-volume objective, and then calibrates the score $\|\widehat\Lambda s(x,y)\|_{\widehat p}$. {FOTCP} uses the Euclidean OT cost, while {MV-OTCP} learns a positive definite metric cost $c_M(s,u)=(s-u)^\top M(s-u)$ before final calibration. For {MV-OTCP}, the reference-to-score map is represented by a GradNet \citep{chaudhari2025gradnetot}. Holding the learned metric fixed, final calibration uses the augmented assignment \eqref{eq:final_assignment}.

The {SCP} covariance ridge is $10^{-3}\{\operatorname{tr}(\widehat\Sigma)/2\}I_2$. The {MVCS} matrix and scalar exponent are jointly trained for $200$ epochs using minibatches of size $100$, matrix learning rate $0.004$, exponent learning rate $0.01$, and cosine learning-rate decay, and are frozen before calibration. {MV-OTCP} uses a four-module GradNet of hidden width $24$, $20$ outer metric updates, $500$ map steps in the first update and $200$ thereafter, metric learning rate $0.002$.
The reference law $\gQ$ is $U=R\Theta$, where $R\sim\operatorname{Unif}(0,1)$ and $\Theta$ is uniform on the unit sphere. 
Empirically, we choose the reference cloud $U_{1:n+1}$ as in \citet{ndiaye2025beyond}, that is, let $n+1=n_R n_S+n_o$ and sample $n_S$ points $\theta_1,\dots,\theta_{n_S}$ on the unit sphere, then scale them to form $n_R$ radial shells of the form $\{\frac{r}{n_R}\theta_s:r=1,\dots,n_R\}$ with $n_o$ copies of the origin. In simulations, we use $n_R=20$, $n_S=50$ and $n_o=1$, and is shared by {FOTCP} and {MV-OTCP}.

\subsection{Implementation details of heterogeneous experiments}
\label{app:heterogeneous_implementation}
This appendix provides the coefficient functions used in the heterogeneous simulations of Subsection~\ref{subsec:heterogeneous_simulation}. Let
\begin{align*}
    R(\gamma)
    &=
    \begin{pmatrix}
        \cos\gamma & -\sin\gamma\\
        \sin\gamma & \cos\gamma
    \end{pmatrix},
    \qquad
    m(x)
    =
    \bigl(
        1.2x+0.25\sin(\pi x),\,
        -0.6+0.35x^2
    \bigr)^\top,
\end{align*}
and define
\begin{align*}
    s_{1/2}(x)=\frac{1+\sin(\pi x/2)}{2},
    \qquad
    c_{1/2}(x)=\frac{1+\cos(\pi x/2)}{2},
    \qquad
    s_1(x)=\frac{1+\sin(\pi x)}{2}.
\end{align*}

\paragraph{$X$-varying mixture.}
Let
\begin{align*}
    \{v_1,v_2,v_3,v_4\}
    =
    \bigl\{
        (1,0)^\top,\,
        (0,1)^\top,\,
        (-1,0)^\top,\,
        (0,-1)^\top
    \bigr\},
\end{align*}
and define
\begin{align*}
    p(x)
    &=
    \operatorname{softmax}
    \left[
        1.6\{
            \sin(\pi x/2),
            \cos(\pi x/2),
            -\sin(\pi x/2),
            -\cos(\pi x/2)
        \}
    \right],\\
    \theta_{\rm m}(x)
    &=0.65\pi x,
    \qquad
    {a_{\rm m}(x)}=0.65+0.55s_1(x),
    \qquad
    \tau(x)=0.08+0.12c_{1/2}(x).
\end{align*}

\paragraph{$X$-varying clover.}
The radial boundary and covariate-dependent deformation are given by
\begin{align*}
    b(x)
    &=0.48+0.30s_1(x),
    \qquad
    \delta(x)=0.35\pi x+0.12\sin(\pi x),\\
    q_x(\theta)
    &=1+b(x)\cos\{4(\theta-\delta(x))\},\\
    \rho_{\rm c}(x)
    &=3+\frac{7}{2}\{1+\cos(\pi x)\},
    \qquad
    a_{\rm c}(x)=0.65+0.30\{1+\sin(3\pi x/2+0.3)\},\\
    \gamma_{\rm c}(x)
    &=0.85\pi x+0.25\sin(\pi x),\\
    A_{\rm c}(x)
    &=
    R\{\gamma_{\rm c}(x)\}\,
    a_{\rm c}(x)
    \operatorname{diag}
    \left\{
        \sqrt{\rho_{\rm c}(x)},
        \rho_{\rm c}(x)^{-1/2}
    \right\}.
\end{align*}

\paragraph{Exp transformed.}
Let
\begin{align*}
    \eta=(0.304717,-1.039984)^\top,
    \qquad
    \eta_r=0.750451,
\end{align*}
and define
\begin{align*}
    m(x)
    &=
    2\left[
        \sin(x\eta)
        +\frac{1}{2}\tanh(x^2\eta)
        +(x,0)^\top
    \right],\\
    r(x)
    &=
    \frac{|x|}{2}
    +\frac{(\eta_r x)^2}{10}
    +0.15,
\end{align*}
where the sine and hyperbolic tangent are applied componentwise. The rotation angle is obtained from the anchor locations
\begin{align*}
    a
    =
    (-2.25,-1.55,-1.05,-0.60,-0.20,0.20,0.60,1.05,1.55,2.25)
\end{align*}
and fixed angles
\begin{align*}
    \omega
    =
    (&1.721317,-0.384038,2.253137,1.240100,-2.549859,\\
     &2.988423,1.640789,1.797395,-2.336631,-0.311734),
\end{align*}
drawn once from $\operatorname{Unif}[-\pi,\pi]$ using seed $42$. Define
\begin{align*}
    w_k(x)
    &=
    \frac{(|x-a_k|+10^{-6})^{-4}}
    {\sum_{\ell=1}^{10}(|x-a_\ell|+10^{-6})^{-4}},
    \qquad
    \theta_{\rm e}(x)
    =
    \sum_{k=1}^{10}w_k(x)\omega_k.
\end{align*}
Finally, let $X\sim\operatorname{Unif}[-2,2]$ and
$E=(E_1,E_2)^\top$ with $E_1,E_2\stackrel{\rm iid}{\sim}\operatorname{Exp}(1)$, independently, and generate
\begin{align*}
    Y
    =
    {m(X)}
    +
    r(X)R\{\theta_{\rm e}(X)\}E.
\end{align*}
All methods use the same residual $s(x,y)=y-\widehat\mu(x)$, where $\widehat\mu$ is a fifth-degree polynomial ridge regression fit on an independent training sample. {RLCP} uses $\|s(x,y)\|_2$ as its scalar score with bandwidth $0.2$. For {SCP}, $\widehat\Sigma_x$ is represented by a Cholesky network with one hidden tanh layer of width $16$ and then compute the score $\|\widehat\Sigma_x^{-1/2}s(x,y)\|_2$.
{MVCS} learns an $x$-dependent matrix $\widehat\Lambda(x)$ and one scalar exponent $p$, and applies a global split-conformal quantile to $\|\widehat\Lambda(x)s(x,y)\|_p$ after freezing both fitted quantities. 
For {IMV-OTCP}, we use $15$ empirical-quantile anchors and an RBF kernel with bandwidth of $0.1$.
The reference-to-score and score-to-reference maps are four-module GradNets of hidden width $32$ \citep{chaudhari2025gradnetot}. The first outer iteration uses $600$ map steps and subsequent iterations use $250$; after selecting the best outer iterate, the final score-to-reference map uses $600$ steps and the hybrid loss in \eqref{eq:pointwise_forward_gradnet_loss}. The reference law is the radial law $U=R\Theta$ above, represented by a $20\times50$ grid of equally spaced radial shells and common random directions, without an origin. Each calibration and test point is assigned to its nearest anchor before applying conformalied step.


\subsection{Multivariate-regression volume study}
\label{app:real_volume}

\paragraph{Splitting and common representation.}
Each replication begins with one random permutation. \texttt{Energy} assigns $55\%$ of the observations to training and $15\%$ to each remaining split. The other datasets use $70\%$ for training and $10\%$ for each remaining split. Estimation and calibration samples are capped separately at $1{,}000$ observations, and the full test split is retained. Missing values are replaced by coordinatewise training medians. We then follow \cite{dheur2025unified,braun2025minimum} to fit coordinatewise Gaussian quantile transforms $\widehat Q_X$ and $\widehat Q_Y$ on the training split, using at most $1{,}000$ empirical quantiles, and keep both transforms fixed.

{Let $d_\gX$ denote the original covariate dimension.} To obtain a common covariate representation, let $\widehat P$ be the whitened PCA map fitted to $\widehat Q_X(X)$ on the training split, with
\[
    k=\min\{8,d_\gX,n_{\rm tr}-1\},
    \qquad
    Z=\widehat P\{\widehat Q_X(X)\}\in\R^k.
\]
The same frozen map is applied to estimation, calibration, and test covariates. Thus test observations do not affect either the PCA directions or their whitening scales. All methods receive $Z$ as their covariate, and all localization and random-ball distances are Euclidean distances in this representation. The common center is ridge regression with penalty $1$, fitted on $(Z,\widehat Q_Y(Y))$ from the training split. {Define $S(z,y)=\widehat Q_Y(y)-\widehat\mu(z)$.} Its residual ${S=S(Z,Y)}$ is the common multivariate score input. Set membership and volume are evaluated in these training-fitted response coordinates. Table~\ref{tab:real_volume_data} reports the effective sample sizes.

\begin{table}[H]
\centering
\small
\setlength{\tabcolsep}{4.5pt}
\caption{Real multivariate-regression datasets and effective split sizes. Estimation and calibration are capped at $1{,}000$ observations; the test split is not truncated.}
\label{tab:real_volume_data}
\begin{tabular}{lrrrrrrr}
\toprule
Dataset & $n$ & $d_\gX$ & ${d}$ & $n_{\rm tr}$ & $n_{\rm est}$ & $n_{\rm cal}$ & $n_{\rm test}$ \\
\midrule
\texttt{Energy} & 768 & 8 & 2 & 422 & 115 & 115 & 116 \\
\texttt{Bias} & 7,752 & 22 & 2 & 5,426 & 775 & 775 & 776 \\
\texttt{Taxi} & 61,286 & 6 & 2 & 42,900 & 1,000 & 1,000 & 6,130 \\
\texttt{rf1} & 9,125 & 64 & 8 & 6,387 & 912 & 912 & 914 \\
\texttt{scm1d} & 9,803 & 280 & 16 & 6,862 & 980 & 980 & 981 \\
\bottomrule
\end{tabular}
\end{table}

\paragraph{Method implementations.}
Whenever an ordinary split-conformal threshold is used below, it is the $\lceil(n_{\rm cal}+1)(1-\alpha)\rceil$th order statistic, with the usual $+\infty$ convention when the index exceeds $n_{\rm cal}$. All method-specific hyperparameters were selected in pilot experiments whose random seeds are disjoint from those of the reported replications, and were fixed before formal evaluation.

{SCP} estimates a global covariance $\widehat\Sigma$ from the estimation residuals, adding $10^{-3}\operatorname{tr}(\widehat\Sigma)/{d}$ to its eigenvalues with a numerical floor of $10^{-8}$. It calibrates the Mahalanobis score $\|\widehat\Sigma^{-1/2}S\|_2$, giving the ellipsoid
\[
    \{y:S(z,y)^\top\widehat\Sigma^{-1}S(z,y)\le \widehat q^2\}.
\]

{MVCS} uses the adaptive matrix score $\|\widehat\Lambda(z)S\|_{\widehat p}$. The matrix network has two ReLU hidden layers of width $64$; the network and scalar exponent, initialized at $p=2$, are jointly trained on the estimation residuals for $100$ epochs with batch size $128$, Adam learning rates $0.004$ and $0.01$, respectively, and cosine learning-rate decay. Computation uses double precision for ${d}\ge8$ and single precision otherwise. The fitted matrix network and exponent are frozen before the common split-conformal calibration.


{RLCP} uses the Euclidean residual score $\|S\|_2$. For each test representation $z$, it draws $\widetilde z\sim N(z,h^2I_k)$ and weights calibration point $i$ by
\[
    w_i(\widetilde z)=\exp\{-\|Z_i-\widetilde z\|_2^2/(2h^2)\}.
\]
The randomized threshold is the weighted $(1-\alpha)$ quantile after adding the test covariate with score $+\infty$. Writing $Z_{i,(k)}$ for the $k$th nearest estimation neighbor of $Z_i$, let $k_0=\min\{n_{\rm est}-1,\max(20,\lceil n_{\rm est}/4\rceil)\}$. In the order \texttt{Energy}, \texttt{Bias}, \texttt{Taxi}, \texttt{rf1}, and \texttt{scm1d}, the {RLCP} bandwidth is
\[
    h=(1,0.75,0.75,1,0.75)
    {\operatorname{Median}_{i}}d\{Z_i,Z_{i,(k_0)}\}.
\]

For {IMV-OTCP}, the metric is parameterized by a trace-zero matrix exponential and learned with the implicit-gradient GradNetOT implementation. The common settings are four outer metric updates, two GradNet modules of hidden width $16$, map-step counts $120/60$ for the first/later outer iterations, metric learning rate $0.02$. The score KDE uses the weighted Scott bandwidth specified in Appendix~\ref{subsec:tacp_implementation}. The reference law is $U=R\Theta$ with $R\sim\operatorname{Unif}(0,1)$ and $\Theta$ uniform on the unit sphere; an $n_{\rm est}$-point stratified cloud uses radii $i/n_{\rm est}$ and independent random directions.
Define
\[
    k_h(A)=\min\{n_{\rm est}-1,\max(20,\lceil n_{\rm est}/A\rceil)\},
    \qquad
    h_0(A)={\operatorname{Median}_{i}}
    d\{Z_i,Z_{i,(k_h(A))}\}.
\]
The method fits $A$ anchor-specific metric and transport-map pairs. The anchors are the centers returned by $K$-means on the estimation representations, using ten initializations, and the RBF localization bandwidth is $h=m h_0(A)$. Calibration and test observations are assigned to their nearest anchor; their transported radius is evaluated with that anchor's map, and one global split-conformal threshold is then calibrated from these assigned scores. In the dataset order above, the settings are
\[
\begin{aligned}
    A&=(6,6,8,12,16),\\
    m&=(0.35,0.5,1,0.75,0.75).
\end{aligned}
\]
{These dataset-specific bandwidth and anchor choices were fixed in the pilot experiments with disjoint random seeds described above;} Table~\ref{tab:real_volume_results} uses this same fitted configuration.



\subsection{Implementation details of general task experiments}
\label{app:task_implementation}
For the integrated quadratic setting, $X\sim\operatorname{Unif}[-1,1]$, $\mu(x)=(x,-x)^\top$, and $Y=\mu(X)+m_C+A_CD$, where $C$ has probabilities $(2/3,1/3)$, $D$ is uniform on the unit disk, and
\[
    m_1=(-0.22,0)^\top,
    \quad m_2=(0.44,0)^\top,
    \quad
    A_1=\begin{pmatrix}0.85&0.20\\0.20&0.18\end{pmatrix},
    \quad
    A_2=\begin{pmatrix}0.70&0.16\\0.16&0.16\end{pmatrix}.
\]
The loss is $10(z-y)^\top\operatorname{diag}(1,100)(z-y)$. Numerical integration uses a $61\times61$ residual grid on $[-2.5,2.8]\times[-2,2]$. {TA-OTCP} uses $40$ outer updates, map-step counts $150/75$, metric learning rate $0.006$, implicit damping $0.01$, and $30$ conjugate-gradient steps; the estimation and calibration reference factorizations are $(25,40,0)$ and $(25,40,1)$.

For the smoothed quadratic setting, $X\sim N\{(1,1)^\top,2.25I_2\}$, $\mu(x)=(-x_1-x_2^2,-x_1^2-x_2)^\top$, and $Y=\mu(X)+\varepsilon$ with $\varepsilon\sim N(0,I_2)$. The quadratic curvature matrix is
\[
    {H}
    =
    R_{45^\circ}\operatorname{diag}(20\cdot50/51,20/51)R_{45^\circ}^\top.
\]
The smoothed robust objective uses temperature $0.25$ and a $25\times25$ grid on $[-4,4]^2$. {TA-OTCP} uses six outer updates, map-step counts $300/100$, metric learning rate $0.01$, implicit damping $0.02$, $20$ conjugate-gradient steps, and reference factorizations $(250,4,0)$ and $(250,4,1)$. In both settings, the two finite-grid baselines receive a fresh split-conformal calibration after task-based selection.

\subsection{Implementation details of U.S. stock portfolios}
\label{app:real_stock}

The chronological split is fixed as follows.
\begin{table}[H]
\centering
\small
\caption{Leakage-free chronological split for the portfolio experiment.}
\label{tab:stock_windows}
\begin{tabular}{cccc}
\toprule
Training & Estimation & Calibration & Test \\
\midrule
2012--2014 & 2015 & 2016 & 2017 \\
\bottomrule
\end{tabular}
\end{table}

The common ridge-regression center $\widehat\mu$ and residual covariance $\widehat\Sigma_{\rm tr}$ are estimated from training dates only. Using the symmetric positive-definite square root, the {SCP} score is $ V_i^{\rm SCP}=\lVert\widehat\Sigma_{\rm tr}^{-1/2}S_i\rVert_2$.
Its ordinary split-conformal quantile gives the corresponding ellipsoid. The task-aware metric and parametric score-to-reference map are fitted on 2015 using the empirical version of smoothed-CRO in Example \ref{example:smooth_CRO}. Holding both fixed, {TA-OTCP} calibrates on 2016 the score $ V_i^{\rm TA}=\lVert\widehat T_{\widehat M_{\rm task}}(S_i)\rVert_2$.
For {E2E}, we adapt the ellipsoidal construction of \citet{yeh2024end} to the same four-way chronological split. A neural network predicts a conditional location and Cholesky factor,
\[
    h_\theta(x)=\{\mu_\theta(x),L_\theta(x)\},
    \qquad
    \Sigma_\theta(x)=L_\theta(x)L_\theta(x)^\top,
\]
where the diagonal of $L_\theta(x)$ is passed through softplus. All response standardization is estimated on 2012--2014 and the displayed parameters are mapped back to percentage-return units. The scalar nonconformity score and final uncertainty set are
\begin{align*}
    V_i^{\rm E2E}
    &=\{Y_i-\mu_\theta(X_i)\}^\top
      \Sigma_\theta(X_i)^{-1}
      \{Y_i-\mu_\theta(X_i)\},\\
    \widehat C_{\rm E2E}(x)
    &=\left\{y:
      \{y-\mu_\theta(x)\}^\top\Sigma_\theta(x)^{-1}
      \{y-\mu_\theta(x)\}
      \le \widehat q_{\rm E2E}
      \right\},
\end{align*}
where $\widehat q_{\rm E2E}$ is the ordinary plus-one split-conformal quantile computed on 2016.


We implement Algorithm~2 of \citet{sun2023predict} using 2015 as their preliminary-calibration sample and 2016 as their scalar-adjustment sample. {Write $S_i=Y_i-\widehat\mu(X_i)$. For {CRO}, $\widehat g(x)$ estimates the $(1-\alpha)$-quantile of $\|S_i\|_2$ using the 2015 data, and $\widehat\Sigma_{\rm CRO}$ is the regularized covariance matrix of the corresponding scaled residuals $S_i/\widehat g(X_i)$.}
The scalar calibration score is
\[
    V_i^{\rm CRO}=\frac{
      \{S_i^\top{\widehat\Sigma_{\rm CRO}^{-1}}S_i\}^{1/2}}
      {\widehat g(X_i)}.
\]
Let $\widehat\eta$ be the $\lceil(n_{\rm cal}+1)(1-\alpha)\rceil$th order statistic of the 2016 scores. The resulting set is
\[
    \widehat C_{\rm CRO}(x)
    =\left\{y:\
      \{(y-\widehat\mu(x))^\top{\widehat\Sigma_{\rm CRO}^{-1}}
      (y-\widehat\mu(x))\}^{1/2}
      \le \widehat\eta\widehat g(x)\right\}.
\]

\section{Additional proofs}
\label{app:proofs}

\subsection{Proof of Lemma~\ref{lem:metric_brenier}}
\label{app:proof_metric_brenier}

\begin{proof}
Let $A=M^{1/2}$, $\widetilde\gP=A_\#\gP$, and $\widetilde\gQ=A_\#\gQ$. For any measurable map $T$ with $T_\#\gP=\gQ$, define the transformed map $\widetilde T(\tilde s)=A T(A^{-1}\tilde s)$. Then $\widetilde T_\#\widetilde\gP=\widetilde\gQ$. Moreover, by the change of variables $\tilde s=As$,
\begin{align*}
    \int (s-T(s))^\top M(s-T(s))\,d\gP(s)
    &=
    \int \|A s-A T(s)\|^2\,d\gP(s)\\
    &=
    \int \|\tilde s-\widetilde T(\tilde s)\|^2
    \,d\widetilde\gP(\tilde s).
\end{align*}
Therefore $T$ minimizes the metric-cost Monge problem from $\gP$ to $\gQ$ if and only if $\widetilde T$ minimizes the ordinary squared-Euclidean Monge problem from $\widetilde\gP$ to $\widetilde\gQ$.

Since $\gP$ is absolutely continuous and $A$ is invertible, $\widetilde\gP$ is also absolutely continuous. Brenier's theorem implies that there exists a unique Monge solution
\[
    \widetilde T
    =
    \argmin_{\widetilde T_\#\widetilde\gP=\widetilde\gQ}
    \int \|\tilde s-\widetilde T(\tilde s)\|^2
    \,d\widetilde\gP(\tilde s).
\]
Moreover, there exists a convex function $\psi$ such that
\begin{align*}
    \widetilde T(\tilde s)=\nabla\psi(\tilde s)
    \qquad
    \widetilde\gP\text{-a.e. }\tilde s .
\end{align*}
Substituting $\tilde s=As$ gives $AT(s)=\nabla\psi(As)$. Conversely, if a measurable map $T$ satisfies $T_\#\gP=\gQ$ and $AT(s)=\nabla\psi(As)$ for some convex $\psi$, then $\widetilde T(\tilde s)=AT(A^{-1}\tilde s)=\nabla\psi(\tilde s)$ pushes $\widetilde\gP$ to $\widetilde\gQ$. By Brenier's theorem, $\widetilde T$ is the ordinary squared-Euclidean Monge solution from $\widetilde\gP$ to $\widetilde\gQ$. The reduction above then implies that $T$ is the metric-cost Monge solution from $\gP$ to $\gQ$.

It remains to prove the inverse statement. Suppose now that $\gQ$ is also absolutely continuous with respect to Lebesgue measure. Since $A$ is invertible, $\widetilde\gQ$ is absolutely continuous as well. Let $G$ be the solution of the reverse metric-cost Monge problem from $\gQ$ to $\gP$, and define
\[
    \widetilde G(\tilde u)=A G(A^{-1}\tilde u).
\]
By the same reduction as above, $\widetilde G$ is the unique Brenier map from $\widetilde\gQ$ to $\widetilde\gP$ for the ordinary squared-Euclidean cost. The forward Brenier map $\widetilde T$ induces the optimal coupling
\[
    \widetilde\pi=(\mathrm{Id},\widetilde T)_\#\widetilde\gP
\]
between $\widetilde\gP$ and $\widetilde\gQ$. Since the squared-Euclidean cost is symmetric, the reversed coupling
\[
    \widetilde\pi^{\mathrm{rev}}=(\widetilde T,\mathrm{Id})_\#\widetilde\gP
\]
is an optimal coupling between $\widetilde\gQ$ and $\widetilde\gP$.
Hence this reversed optimal coupling must be induced by $\widetilde G$, and therefore
\[
    \widetilde G\circ \widetilde T(\tilde s)=\tilde s
    \qquad
    \widetilde\gP\text{-a.e. }\tilde s,
    \qquad
    \widetilde T\circ \widetilde G(\tilde u)=\tilde u
    \qquad
    \widetilde\gQ\text{-a.e. }\tilde u .
\]
Returning to the original variables $\tilde s=As$ and $\tilde u=Au$ gives
\[
    G\circ T(s)=s
    \qquad
    \gP\text{-a.e. }s,
    \qquad
    T\circ G(u)=u
    \qquad
    \gQ\text{-a.e. }u .
\]
This proves that $T$ and $G$ are a.e. inverses of each other.
\end{proof}




\subsection{Proof of Theorem~\ref{thm:tacp_validity}}
\label{app:proof_tacp_validity}

\begin{proof}
Let $U_{1:n+1}=\{U_1,\ldots, U_{n+1}\}$.
Condition on the pre-calibration fit $(\widehat\mu,\widehat M)$ and on $U_{1:n+1}$. Set $y=Y_{n+1}$ and write $Z_i=(X_i,Y_i)$. The augmented observations $Z_1,\dots,Z_{n+1}$ are exchangeable. Let ${\gD_{\rm aug}}=\{Z_1,\dots,Z_{n+1}\}$ denote the unordered augmented sample, and fix an ordered representative $z_{1:n+1}$ of this multiset.

For any permutation $\tau$ of $[n+1]$, let $z^\tau_{1:n+1}=(z_{\tau(1)},\dots,z_{\tau(n+1)})$. Because the objective in \eqref{eq:final_assignment} applies the same fixed metric $\widehat M$ to every augmented score and because ties are resolved permutation-equivariantly, the selected assignment satisfies
\begin{align*}
    \widehat\sigma_{\widehat M}(z^\tau_{1:n+1};U_{1:n+1})
    =
    \widehat\sigma_{\widehat M}(z_{1:n+1};U_{1:n+1})\circ\tau .
\end{align*}
Thus relabeling the augmented sample relabels the source index in exactly the same way, while preserving the reference point assigned to each unordered observation.

Conditional on ${\gD_{\rm aug}}$ and $U_{1:n+1}$, exchangeability implies that the distinguished test observation is uniformly distributed over the $n+1$ elements of ${\gD_{\rm aug}}$. Since the assignment is a permutation, the reference point assigned to a uniformly selected source index is uniformly distributed over $U_1,\dots,U_{n+1}$. Therefore
\begin{align*}
    U_{\widehat\sigma_{\widehat M}^{\,Y_{n+1}}(n+1)}
    \stackrel{d}{=} U_J
    \mid {\gD_{\rm aug}}, U_{1:n+1},
    \qquad
    J\sim\operatorname{Unif}\{1,\dots,n+1\}.
\end{align*}
By construction of $\widehat q_{1-\alpha}\LRs{\{\|U_i\|\}_{i=1}^{n+1}}$, at least $\lceil(n+1)(1-\alpha)\rceil$ of these values are no larger than the threshold. Hence
\begin{align*}
    \sP\left\{
        \left\|U_{\widehat\sigma_{\widehat M}^{\,Y_{n+1}}(n+1)}\right\|_2
        \le
        \widehat q_{1-\alpha}\LRs{\{\|U_i\|\}_{i=1}^{n+1}}
        \mid
        {\gD_{\rm aug}},U_{1:n+1}
    \right\}
    \ge
    1-\alpha .
\end{align*}
The event inside the probability is exactly $Y_{n+1}\in{\widehat\gC}^{\rm MV}(X_{n+1})$. Averaging over the conditioning variables gives the lower bound. If the reference radii are distinct, exactly $\lceil(n+1)(1-\alpha)\rceil$ destinations are accepted, giving the stated equality and upper bound.
\end{proof}

\subsection{Proof of Theorem~\ref{thm:tacp_uniform_deviation}}
\label{app:proof_tacp_uniform_deviation}
We first introduce several auxiliary lemmas.

\begin{lemma}[Covering number of regularized positive definite metrics]
\label{lem:spd_covering}
Let
\begin{align*}
    \gM_{\rm sp}(\underline\lambda,\overline\lambda)
    =
    \left\{
        M=M^\top\succ0:
        \lambda_{\min}(M)\ge \underline\lambda,
        \lambda_{\max}(M)\le \overline\lambda
    \right\},
\end{align*}
where $0<\underline\lambda\le\overline\lambda<\infty$. Then, with $m={d}({d}+1)/2$ and $R_\gM=\sqrt{{d}}\,\overline\lambda$,
\begin{align}
\label{eq:spd_covering}
    N(\epsilon,\gM_{\rm sp},\|\cdot\|_F)
    \le
    \left(1+\frac{2R_\gM}{\epsilon}\right)^m,
    \qquad
    \epsilon>0 .
\end{align}
\end{lemma}

\begin{proof}
The vector space of real symmetric ${d}\times {d}$ matrices has dimension $m={d}({d}+1)/2$ under the Frobenius inner product. If $M\in\gM_{\rm sp}(\underline\lambda,\overline\lambda)$, then
\begin{align*}
    \|M\|_F
    =
    \left(\sum_{j=1}^{{d}}\lambda_j^2(M)\right)^{1/2}
    \le
    \sqrt{{d}}\,\overline\lambda
    =
    R_\gM .
\end{align*}
Thus $\gM_{\rm sp}$ is contained in an $m$-dimensional Euclidean ball of radius $R_\gM$. The standard volumetric covering bound for a Euclidean ball gives \eqref{eq:spd_covering}.
\end{proof}

\begin{assumption}[Primitive global objective regularity]
\label{ass:tacp_objective_primitives}
\begin{enumerate}[label=(\roman*),nosep,leftmargin=*]
\item On ${\bigcup_{M\in\gM_0}G_M(\sB_{1-\alpha})}$, the score density is $L_{f_P}$-Lipschitz and satisfies $f_P\ge f_{P,\min}>0$.
\item The maps $G_M$ are differentiable almost everywhere on $\sB_{1-\alpha}$ and satisfy the Monge--Amp\`ere identity. For every $\epsilon>0$,
\[
    \sup_{M_c\in\gM_0}
    \E_U\left[
        \sup_{M\in\gM_0(M_c,\epsilon)}
        \|G_M(U)-G_{M_c}(U)\|
        \one\{U\in\sB_{1-\alpha}\}
    \right]
    \le L_G\epsilon.
\]
\end{enumerate}
\end{assumption}

\begin{lemma}[Verification of global objective regularity]
\label{lem:brenier_objective_modulus}
Under Assumption~\ref{ass:tacp_objective_primitives}, define
\begin{align*}
    B_\zeta
    =
    \frac{1}{f_{P,\min}},
    \qquad
    L_\zeta
    =
    \frac{L_{f_P}L_G}{f_{P,\min}^2}.
\end{align*}
Then $\sup_{M\in\gM_0}|\zeta_M(U)|\le B_\zeta$ almost surely and, for every $\epsilon>0$,
\begin{align}
    \sup_{M_c\in\gM_0}
    \E\left[
        \sup_{\substack{M\in\gM_0:\\ \|M-M_c\|_F\le\epsilon}}
        |\zeta_M(U)-\zeta_{M_c}(U)|
    \right]
    \le
    L_\zeta\epsilon .\nonumber
\end{align}
Consequently, Assumption~\ref{ass:tacp_metric_learning}(i) holds with these constants.
\end{lemma}

\begin{proof}
The metric Brenier representation gives
\begin{align*}
    \nabla_uG_M(u)
    =
    M^{-1/2}\nabla^2g_M(M^{1/2}u)M^{1/2},
    \qquad
    |\det \nabla_uG_M(u)|
    =
    \det\{\nabla^2g_M(M^{1/2}u)\}.
\end{align*}
Since $(G_M)_\#\gQ=\gP$, the Monge--Amp\`ere identity gives, for $\gQ$-a.e. $u$,
\begin{align}
    f_Q(u)
    =
    f_P(G_M(u))|\det \nabla_uG_M(u)|
    =
    f_P(G_M(u))
    \det\{\nabla^2g_M(M^{1/2}u)\}.\nonumber
\end{align}
Therefore the Jacobian factor in \eqref{eq:task_loss_hessian_form} can be rewritten as
\begin{align}
\label{eq:zeta_density_form}
    \zeta_M(u)
    =
    \one\{u\in\sB_{1-\alpha}\}
    \frac{1}{f_P(G_M(u))}.
\end{align}
Then the envelope bound follows immediately from \eqref{eq:zeta_density_form} and Assumption~\ref{ass:tacp_objective_primitives}(i):
\begin{align*}
    |\zeta_M(U)|
    \le
    1/f_{P,\min}
    =
    B_\zeta .
\end{align*}
For the modulus, fix $M_c\in\gM_0$ and $M\in\gM_0$ with $\|M-M_c\|_F\le\epsilon$. On $\{U\in\sB_{1-\alpha}\}$, write $S=G_M(U)$ and $S_c=G_{M_c}(U)$. Using \eqref{eq:zeta_density_form},
\begin{align*}
    |\zeta_M(U)-\zeta_{M_c}(U)|
    &\le
    \one\{U\in\sB_{1-\alpha}\}
    \frac{L_{f_P}}{f_{P,\min}^2}\|S-S_c\|.
\end{align*}
Taking the supremum over all such $M$, then expectation over $U$, and applying Assumption~\ref{ass:tacp_objective_primitives}(ii) proves the result.
\end{proof}

\begin{proof}[Proof of Theorem~\ref{thm:tacp_uniform_deviation}]
\textbf{Step 1: Reduction to a finite metric net.}
It is enough to control the infeasible empirical process
\begin{align*}
    \sup_{M\in\gM_0}|\gJ_N(M)-\gJ(M)|,
\end{align*}
because the definition of $\mathcal E_N^{\rm obj}=\sup_{M\in\gM_0}|\widehat\gJ(M)-\gJ_N(M)|$ and the triangle inequality give
\begin{align}
    \label{eq:objective_deviation_reduction}
    \sup_{M\in\gM_0}
    |\widehat\gJ(M)-\gJ(M)|
    \le
    \mathcal E_N^{\rm obj}+
    \sup_{M\in\gM_0}|\gJ_N(M)-\gJ(M)|.
\end{align}
Fix $\epsilon>0$ and let $\mathcal N_\epsilon=\{M_1,\dots,M_K\}$ be an $\epsilon$-net of $\gM_0$ under $\|\cdot\|_F$. By Lemma~\ref{lem:spd_covering} and the spectral bounds defining $\gM_0$,
\begin{align}
\label{eq:tacp_covering_number_proof}
    K
    \le
    \left(1+\frac{2R_\gM}{\epsilon}\right)^m .
\end{align}
For each center $M_j$, define the local oscillation envelope
\begin{align*}
    H_j(U)
    =
    \sup_{\substack{M\in\gM_0:\\ \|M-M_j\|_F\le\epsilon}}
    |\zeta_M(U)-\zeta_{M_j}(U)|.
\end{align*}
For any $M\in\gM_0$, there exists $M_j$ such that $\|M-M_j\|_F\le\epsilon$ by the definition of $\mathcal N_\epsilon$. Then
\begin{align}
    |\gJ_N(M)-\gJ(M)|
    &\le
    |\gJ_N(M_j)-\gJ(M_j)|
    +
    |\gJ_N(M)-\gJ_N(M_j)|
    +
    |\gJ(M)-\gJ(M_j)|\nonumber\\
    &=
    |\gJ_N(M_j)-\gJ(M_j)|
    +
    \left|
        \frac{1}{N}\sum_{i=1}^N
        \{\zeta_M(U_i)-\zeta_{M_j}(U_i)\}
    \right|\nonumber\\
    &\quad+
    \left|
        \E\{\zeta_M(U)-\zeta_{M_j}(U)\}
    \right|\nonumber\\
    &\le
    |\gJ_N(M_j)-\gJ(M_j)|
    +
    \frac{1}{N}\sum_{i=1}^N
    |\zeta_M(U_i)-\zeta_{M_j}(U_i)|
    \nonumber\\
    &\quad+
    \E|\zeta_M(U)-\zeta_{M_j}(U)|\nonumber\\
    &\le
    |\gJ_N(M_j)-\gJ(M_j)|
    +
    \frac{1}{N}\sum_{i=1}^N H_j(U_i)
    +
    \E H_j .\nonumber
\end{align}
Assumption~\ref{ass:tacp_metric_learning}(i) gives $0\le H_j\le 2B_\zeta$ a.s. and $\E H_j\le L_\zeta\epsilon$. Consequently,
\begin{align}
\label{eq:net_reduction_tacp}
    \sup_{M\in\gM_0}
    |\gJ_N(M)-\gJ(M)|
    \le
    \max_{1\le j\le K}|\gJ_N(M_j)-\gJ(M_j)|
    +
    \max_{1\le j\le K}\frac{1}{N}\sum_{i=1}^N H_j(U_i)
    +
    L_\zeta\epsilon .
\end{align}
\textbf{Step 2: Concentration on the net.}
For a fixed $j$, the variables $\zeta_{M_j}(U_i)$ are independent and bounded by $B_\zeta$. Hoeffding's inequality implies
\begin{align*}
    \sP\left\{
        |\gJ_N(M_j)-\gJ(M_j)|>t
    \right\}
    \le
    2\exp\left(-\frac{Nt^2}{2B_\zeta^2}\right).
\end{align*}
Taking a union bound over $j\in[K]$ gives, with probability at least $1-\delta/2$,
\begin{align*}
    \max_{1\le j\le K}
    |\gJ_N(M_j)-\gJ(M_j)|
    \le
    B_\zeta\sqrt{
        \frac{2\{\log(4K)+\log(1/\delta)\}}{N}
    } .
\end{align*}
Similarly, for each $j$, Hoeffding's inequality applied to $H_j\in[0,2B_\zeta]$ gives
\begin{align*}
    \sP\left\{
        \frac{1}{N}\sum_{i=1}^N H_j(U_i)
        >
        \E H_j+t
    \right\}
    \le
    \exp\left(-\frac{Nt^2}{2B_\zeta^2}\right).
\end{align*}
Another union bound gives, with probability at least $1-\delta/2$,
\begin{align*}
    \max_{1\le j\le K}
    \frac{1}{N}\sum_{i=1}^N H_j(U_i)
    \le
    L_\zeta\epsilon
    +
    B_\zeta\sqrt{
        \frac{2\{\log(2K)+\log(1/\delta)\}}{N}
    } .
\end{align*}
\textbf{Step 3: Combine above results.}
Combining the last two displays with \eqref{eq:net_reduction_tacp}, and using \eqref{eq:tacp_covering_number_proof}, yields
\begin{align*}
    \sup_{M\in\gM_0}|\gJ_N(M)-\gJ(M)|
    \le
    C B_\zeta
    \sqrt{
        \frac{
            m\log(1+2R_\gM/\epsilon)+\log(2/\delta)
        }{N}
    }
    +
    2L_\zeta\epsilon .
\end{align*}
If $L_\zeta>0$, choose $\epsilon=B_\zeta/(L_\zeta N)$; then $2L_\zeta\epsilon=2B_\zeta/N$ is absorbed into the displayed constant after replacing the logarithm by $\log\{e+2R_\gM L_\zeta N/B_\zeta\}$. By Eq.~\eqref{eq:objective_deviation_reduction}, this gives the uniform-deviation bound in the theorem.

\textbf{Step 4: Argmin consistency.}
Let
\(
    \Delta_N=\sup_{M\in\gM_0}|\widehat\gJ(M)-\gJ(M)|.
\)
Taking the tail probability $\eta_N^{\rm tail}=N^{-1}$ shows that the empirical-process term is $o_{\sP}(1)$ for fixed ${d}$. Assumption~\ref{ass:tacp_metric_learning}(ii) therefore implies $\Delta_N=o_{\sP}(1)$. Since $\widehat M$ minimizes $\widehat\gJ$ over $\gM_0$ and $M^\star\in\gM_0$,
\begin{align*}
    \gJ(\widehat M)-\gJ(M^\star)
    &\le
    \{\gJ(\widehat M)-\widehat\gJ(\widehat M)\}
    +
    \{\widehat\gJ(M^\star)-\gJ(M^\star)\}\\
    &\le
    2\Delta_N .
\end{align*}
For any $\epsilon>0$, define
\[
    \operatorname{gap}(\epsilon)
    :=
    \inf_{\substack{M\in\gM_0:\ \|M-M^\star\|_F\ge\epsilon}}
    \{\gJ(M)-\gJ(M^\star)\}.
\]
Assumption~\ref{ass:tacp_metric_learning}(iii) gives $\operatorname{gap}(\epsilon)>0$, and hence
\[
    \sP\{\|\widehat M-M^\star\|_F\ge\epsilon\}
    \le
    \sP\{2\Delta_N\ge\operatorname{gap}(\epsilon)\}
    \to0.
\]
Thus $\|\widehat M-M^\star\|_F\overset{\sP}{\to}0$. Since $\widehat M,M^\star\in\gM_0$, $\|\widehat M-M^\star\|_F\le2\sqrt{{d}}\,\overline\lambda$ almost surely; boundedness together with convergence in probability gives $\E\|\widehat M-M^\star\|_F\to0$.
\end{proof}

\subsection{Proof of Theorem~\ref{thm:tacp_assignment_set_error}}
\label{app:proof_tacp_assignment_set_error}
\begin{proof}

\textbf{Step 1: Control set disagreement.}
For the true augmented sample, write $S_i=Y_i-\widehat\mu(X_i)$ for $i\in[n+1]$. For a fixed metric $M$, let $\widehat T_{n+1,M}(S_i)=U_{\widehat\sigma_M^{Y_{n+1}}(i)}$, $i\in[n+1]$ be the empirical Monge map solution of \eqref{eq:final_assignment} with $y=Y_{n+1}$ and $\widehat M$ replaced by $M$. Define
\begin{align*}
    \widehat p_{\triangle,n+1}
    =
    \frac{1}{n+1}
    \sum_{i=1}^{n+1}
    \one\left\{
    \one\{\|T_{M^\star}(S_i)\|\le q_\star\}
    \ne
    \one\{\|\widehat T_{n+1,\widehat M}(S_i)\|
    \le \widehat q_{1-\alpha}\LRs{\{\|U_i\|\}_{i=1}^{n+1}}\}
    \right\}.
\end{align*}
For each $i$, set
\begin{align*}
    e_i
    =
    \left|
    \|\widehat T_{n+1,\widehat M}(S_i)\|
    -
    \|T_{M^\star}(S_i)\|
    \right|.
\end{align*}
Write
\[
    \widehat q_n
    =
    \widehat q_{1-\alpha}\LRs{\{\|U_i\|\}_{i=1}^{n+1}},
    \qquad
    \delta_n=|\widehat q_n-q_\star|.
\]
If the two indicators in the definition of $\widehat p_{\triangle,n+1}$ differ, then
\begin{align}
    \big|\|T_{M^\star}(S_i)\|-q_\star\big|
    \le
    e_i+\delta_n .\nonumber
\end{align}
Therefore, for any fixed $\tau>0$,
\begin{align}
    \label{eq:hat_p_upper}
    \widehat p_{\triangle,n+1}
    &\le 
    \frac{1}{n+1}
    \sum_{i=1}^{n+1}
    \one\{\big|\|T_{M^\star}(S_i)\|-q_\star\big|\le e_i+\delta_n\}\nonumber\\
    &\le
    \frac{1}{n+1}
    \sum_{i=1}^{n+1}
    \one\{\big|\|T_{M^\star}(S_i)\|-q_\star\big|\le \tau\}+
    \frac{1}{n+1}
    \sum_{i=1}^{n+1}
    \one\{e_i>\tau/2\}
    +\one\{\delta_n>\tau/2\}.
\end{align}
\textbf{Step 2: Control the average map error.}
Let
\begin{align}
    A_{n+1}(M)
    =
    \left\{
    \frac{1}{n+1}\sum_{i=1}^{n+1}
    \|\widehat T_{n+1,M}(S_i)-T_M(S_i)\|^2
    \right\}^{1/2}.\nonumber
\end{align}
For a fixed $M\in\gM_0$, set $Z_i=MS_i$. For every permutation $\sigma$,
\begin{align*}
    \sum_{i=1}^{n+1}c_M(S_i,U_{\sigma(i)})
    &=C_{M,S,U}-2\sum_{i=1}^{n+1}(MS_i)^\top U_{\sigma(i)},\\
    \sum_{i=1}^{n+1}\|Z_i-U_{\sigma(i)}\|^2
    &=C'_{M,S,U}-2\sum_{i=1}^{n+1}(MS_i)^\top U_{\sigma(i)},
\end{align*}
where neither $C_{M,S,U}$ nor $C'_{M,S,U}$ depends on $\sigma$. Hence the $c_M$-assignment of $S_{1:n+1}$ to $U_{1:n+1}$ is exactly the Euclidean discrete--discrete optimal transport assignment from $Z_{1:n+1}$ to $U_{1:n+1}$. The same identity at the population level shows that the Euclidean optimal transport map from $M_\#\gP$ to $\gQ$ is
\[
    z\longmapsto T_M(M^{-1}z).
\]
{Assumption~\ref{ass:tacp_assignment_stability}(ii) and $\|M^{-1}\|_{\rm op}\le\underline\lambda^{-1}$ imply that these transformed maps are Lipschitz uniformly over $M\in\gM_0$.}

Let $\widehat\mu_{n+1,M}=(n+1)^{-1}\sum_{i=1}^{n+1}\delta_{Z_i}$, and $\widehat\nu_{n+1}=(n+1)^{-1}\sum_{i=1}^{n+1}\delta_{U_i}$. The selected permutation induces an optimal coupling between $\widehat\mu_{n+1,M}$ and $\widehat\nu_{n+1}$, and its barycentric projection at $Z_i$ is $U_{\widehat\sigma_M^{Y_{n+1}}(i)}$. {The supports of $M_\#\gP$, $M\in\gM_0$, are contained in a common compact set, and the transformed population maps have a common Lipschitz bound. Hence the constant in Corollary~2.3 of \citet{deb2021rates} can be chosen uniformly over $M$, giving}
\begin{align}
\label{eq:deb_fixed_metric_assignment}
    {
    \sup_{M\in\gM_0}\E\{A_{n+1}(M)^2\}
    \le C r_{{d}}^{(n+1,n+1)}
    \longrightarrow0,}
\end{align}
where $r_{{d}}^{(n+1,n+1)}\to0$ is their discrete--discrete empirical transport rate. {Consequently, Jensen's inequality gives}
\begin{align}
\label{eq:deb_fixed_metric_l1}
    {
    \sup_{M\in\gM_0}\E A_{n+1}(M)\longrightarrow0.}
\end{align}

We can now bound the second term of \eqref{eq:hat_p_upper} by
\begin{align*}
    \frac{1}{n+1}\sum_{i=1}^{n+1}
    \one\{e_i>\tau/2\}
    &\overset{(i)}{\le}
    \frac{2}{\tau}
    \frac{1}{n+1}\sum_{i=1}^{n+1}e_i\\
    &\overset{(ii)}{\le}
    \frac{2}{\tau}
    \frac{1}{n+1}\sum_{i=1}^{n+1}
    \|\widehat T_{n+1,\widehat M}(S_i)-T_{M^\star}(S_i)\|\\
    &\overset{(iii)}{\le}
    \frac{2}{\tau}
    \left\{
        A_{n+1}(\widehat M)
        +
        \frac{1}{n+1}\sum_{i=1}^{n+1}
        \|T_{\widehat M}(S_i)-T_{M^\star}(S_i)\|
    \right\},
\end{align*}
where (i) uses $\one\{e_i>\tau/2\}\le 2e_i/\tau$, (ii) is the reverse triangle inequality, and (iii) uses the triangle inequality and Cauchy--Schwarz.

\textbf{Step 3: Bound the expected symmetric difference.}
Define
\[
    \pi_{n,N}
    =
    \sP\{Y_{n+1}\in\gC_\star^{\rm MV}(X_{n+1})\triangle
    {\widehat\gC}^{\rm MV}(X_{n+1})\}.
\]
Write $\omega_\star(t)=\sP\{\big|\|T_{M^\star}(S)\|-q_\star\big|\le t\}$ for a fresh residual $S\sim\gP$. By exchangeability of the augmented calibration-test sample and permutation equivariance of the assignment,
\begin{align*}
    \pi_{n,N}
    &=\E[\widehat p_{\triangle,n+1}]\\
    &\le
    \omega_\star(\tau)+
    \frac{2}{\tau}
    \left[
        \E \LRl{A_{n+1}(\widehat M)}
        +
        \E\left\{
        \frac{1}{n+1}\sum_{i=1}^{n+1}
        \|T_{\widehat M}(S_i)-T_{M^\star}(S_i)\|
        \right\}
    \right]
    +\sP(\delta_n>\tau/2).
\end{align*}
Since $\widehat M$ is trained on the independent dataset $\gD_{\rm est}$, condition on $\widehat M$, Assumption~\ref{ass:tacp_assignment_stability}(iii) yields
\begin{align*}
    \E\left[
    \frac{1}{n+1}\sum_{i=1}^{n+1}
    \|T_{\widehat M}(S_i)-T_{M^\star}(S_i)\|
    \,\middle|\,\widehat M
    \right]
    \le
    L_T\|\widehat M-M^\star\|_F .
\end{align*}
Because $(T_{M^\star})_\#\gP=\gQ$, this modulus is determined by the radial law of the reference distribution. Assumption~\ref{ass:tacp_assignment_stability}(i) implies that this radial law is continuous at $q_\star$, so $\omega_\star(t)\to0$ as $t\to0$.
Then, for any fixed $\tau>0$,
\begin{align}
\label{eq:tacp_l1_symmetric_difference_bound}
    \pi_{n,N}
    \le
    \omega_\star(\tau)
    +
    \frac{2}{\tau}
    \left\{
        \E A_{n+1}(\widehat M)
        +
        L_T\E\|\widehat M-M^\star\|_F
    \right\}
    +\sP(\delta_n>\tau/2).
\end{align}
{Independence of $\widehat M$ from the augmented calibration and reference samples, together with \eqref{eq:deb_fixed_metric_l1}, gives}
\begin{align}
\label{eq:uniform_deb_assignment}
    {
    \E A_{n+1}(\widehat M)
    \le \sup_{M\in\gM_0}\E A_{n+1}(M)
    \longrightarrow0
    \qquad\text{as }n\to\infty,}
\end{align}
{uniformly in $N$. By Theorem~\ref{thm:tacp_uniform_deviation}, $\E\|\widehat M-M^\star\|_F\to0$ as $N\to\infty$.}

\textbf{Step 4: Establish cutoff consistency.}
To control the remaining cutoff term, let 
\[
    \widehat F_{n+1}(t)
    =
    \frac{1}{n+1}\sum_{i=1}^{n+1}\one\{\|U_i\|\le t\}.
\]
By the theorem's i.i.d. reference condition, the variables $\|U_i\|$ are i.i.d. with distribution function $F_Q$. The Glivenko--Cantelli theorem gives
\[
    \sup_{t\in\R}|\widehat F_{n+1}(t)-F_Q(t)|\longrightarrow0
    \qquad\text{almost surely}.
\]
Since $F_Q$ is continuously differentiable near $q_\star$ and $F_Q'(q_\star)>0$, it is strictly increasing there; empirical-quantile consistency therefore gives $\widehat q_n\to q_\star$ almost surely and hence $\delta_n\to0$ in probability. {Combining this fact with \eqref{eq:tacp_l1_symmetric_difference_bound}, \eqref{eq:uniform_deb_assignment}, and Theorem~\ref{thm:tacp_uniform_deviation}, for every fixed $\tau>0$,}
\[
    {
    \limsup_{\min\{N,n\}\to\infty}\pi_{n,N}
    \le \omega_\star(\tau).}
\]
{Finally let $\tau\downarrow0$. The continuity of $F_Q$ at $q_\star$ gives $\omega_\star(t)\to0$, proving the claimed joint-limit result.}
\end{proof}

\subsection{Proof of Theorem~\ref{thm:itacp_validity}}
\label{app:proof_itacp_validity}

\begin{proof}
Let $\gH_N$ be the $\sigma$-algebra generated by $\gD_{\rm est}$, reference draws, and pretrained predictor $\widehat\mu$. 
Conditional on $\gH_N$, {the score function $(x,y)\mapsto\widehat R_N(x,y)$ is fixed.} By the theorem's independence and exchangeability assumptions,
$V_i:={\widehat R_N(X_i,Y_i)}$, $i\in[n+1]$, are exchangeable. Put $k=\lceil(n+1)(1-\alpha)\rceil$, and let $q_{\rm cal}$ be the $k$th order statistic of $\{V_1,\ldots,V_n,\infty\}$ and $q_{\rm aug}$ the $k$th order statistic of $\{V_1,\ldots,V_{n+1}\}$. Since $V_{n+1}\le\infty$, $q_{\rm aug}\le q_{\rm cal}$ for every realization. Hence
\begin{align*}
    \sP\{V_{n+1}\le q_{\rm cal}\mid\gH_N\}
    &\ge\sP\{V_{n+1}\le q_{\rm aug}\mid\gH_N\}\\
    &=\E\left[
        \frac{1}{n+1}\sum_{i=1}^{n+1}\one\{V_i\le q_{\rm aug}\}
        \,\middle|\,\gH_N
    \right]\\
    &\ge\frac{k}{n+1}\ge1-\alpha.
\end{align*}
The first event is exactly
$\{Y_{n+1}\in{\widehat\gC}^{\rm IMV}(X_{n+1})\}$ by
\eqref{eq:itacp_set}. Taking expectations over $\gH_N$ proves the claim.
\end{proof}

\subsection{Proof of Theorem~\ref{thm:localized_metric_consistency}}
\label{app:proof_localized_metric_consistency}

\begin{proof}
\textbf{Step 1: Reduce the local empirical process to a finite metric net.}
Fix $x$. It is enough to control
\[
    A_{N,x}^{\rm obj}
    :=\sup_{M\in\gM_0}
      |\gJ_{N,x}(M)-\gJ_x(M)|,
\]
since the triangle inequality give
\begin{align}
    \sup_{M\in\gM_0}
    |\widehat\gJ_x(M,{\widehat\theta_x(M)})-\gJ_x(M)|
    \le
    \mathcal E_N^{\rm obj}(x)+ {A_{N,x}^{\rm obj}}.\nonumber
\end{align}
Let $m={d}({d}+1)/2$ and $R_{\gM}=\sqrt{{d}}\,\overline\lambda$.  For $\varepsilon>0$, let $\{M_1,\ldots,M_K\}$ be an $\varepsilon$-net of $\gM_0$ in Frobenius norm. Since $\gM_0$ is contained in the radius-$R_{\gM}$ ball in the $m$-dimensional space of symmetric matrices, Lemma~\ref{lem:spd_covering} gives
\[
    K\le
    \left(1+\frac{2R_{\gM}}{\varepsilon}\right)^m.
\]
For every net point, define the local oscillation envelope
\[
    H_j(u)
    :=
    \sup_{M\in\gM_0(M_j,\varepsilon)}
    |\zeta_{M,x}(u)-\zeta_{M_j,x}(u)|.
\]
Assumption~\ref{ass:imv_metric_learning}(i) implies
\[
    0\le H_j(U)\le2B_x,
    \qquad
    \E_U H_j(U)\le L_x\varepsilon.
\]

For each $M\in\gM_0$, choose $j(M)$ such that
$\|M-M_{j(M)}\|_F\le\varepsilon$. Then
\begin{align*}
    |\gJ_{N,x}(M)-\gJ_x(M)|
    &\le
    |\gJ_{N,x}(M_{j(M)})-\gJ_x(M_{j(M)})|\\
    &\quad+|\gJ_{N,x}(M)-\gJ_{N,x}(M_{j(M)})|
      +|\gJ_x(M)-\gJ_x(M_{j(M)})|\\
    &\le
    |\gJ_{N,x}(M_{j(M)})-\gJ_x(M_{j(M)})|\\
    &\quad+
    \frac1N\sum_{i=1}^N H_{j(M)}(U_i)
    +\E_U H_{j(M)}(U).
\end{align*}
Consequently,
\begin{align}
\label{eq:local_empirical_net_reduction}
    A_{N,x}^{\rm obj}
    &\le
    \max_{1\le j\le K}|\gJ_{N,x}(M_j)-\gJ_x(M_j)|
    +\max_{1\le j\le K}\frac1N\sum_{i=1}^N H_j(U_i)
    +L_x\varepsilon.
\end{align}

\textbf{Step 2: Apply concentration on the net.}
Because $\zeta_{M_j,x}(U)\in[-B_x,B_x]$ and $H_j(U)\in[0,2B_x]$, Hoeffding's inequality and a union bound give, for every $t>0$,
\begin{align*}
    \sP\left\{
        \max_{1\le j\le K}
        |\gJ_{N,x}(M_j)-\gJ_x(M_j)|>t
    \right\}
    &\le
    2K\exp\left(-\frac{Nt^2}{2B_x^2}\right),\\
    \sP\left\{
        \max_{1\le j\le K}
        \left\{\frac1N\sum_{i=1}^N H_j(U_i)-\E_U H_j(U)\right\}>t
    \right\}
    &\le
    K\exp\left(-\frac{Nt^2}{2B_x^2}\right).
\end{align*}
Taking
\[
    t
    =B_x
    \sqrt{\frac{2\{\log(3K)+\log(1/\delta)\}}{N}}
\]
and applying another union bound in
\eqref{eq:local_empirical_net_reduction} shows that, with probability at
least $1-\delta$,
\begin{align}
\label{eq:local_empirical_process_bound}
    A_{N,x}^{\rm obj}
    &\le
    C B_x
    \sqrt{\frac{
        m\log(1+2R_{\gM}/\varepsilon)+\log(2/\delta)}{N}}
    +2L_x\varepsilon
\end{align}
for a universal constant $C$.

\textbf{Step 3: Choose the net resolution.}
Suppose first that $L_x>0$ and take $\varepsilon_N=B_x/(L_xN)$. Then $2L_x\varepsilon_N=2B_x/N$, $m\le {d}^2$, and $R_{\gM}=\sqrt{{d}}\,\overline\lambda$. Enlarging $C$ in \eqref{eq:local_empirical_process_bound} therefore gives
\[
    A_{N,x}^{\rm obj}
    \le C B_x
    \left[
        \frac{{d}^2}{N}
        \log\left\{e+\frac{2\sqrt{{d}}\,\overline\lambda L_xN}{B_x}\right\}
        +\frac{\log(2/\delta)}{N}
    \right]^{1/2}.
\]

\textbf{Step 4: Conclude pointwise metric consistency.}
{Taking the failure probability $N^{-1}$ for $N\ge2$ in the preceding empirical-process bound gives a deterministic right-hand side converging to zero for fixed ${d}$.} Therefore
\[
    A_{N,x}^{\rm obj}=o_{\sP}(1).
\]
Assumption~\ref{ass:imv_metric_learning}(ii) now gives
\[
    D_{N,x}
    :=\sup_{M\in\gM_0}
      |\widehat\gJ_x\{M,\widehat\theta_x(M)\}-\gJ_x(M)|
    =o_{\sP}(1).
\]
Finally, empirical optimality of $\widehat M_x$ and
$M_x^\star\in\gM_0$ imply
\begin{align*}
    \gJ_x(\widehat M_x)
    &\le
    \widehat\gJ_x\{\widehat M_x,
        \widehat\theta_x(\widehat M_x)\}+D_{N,x}\\
    &\le
    \widehat\gJ_x\{M_x^\star,
        \widehat\theta_x(M_x^\star)\}+D_{N,x}\\
    &\le
    \gJ_x(M_x^\star)+2D_{N,x}.
\end{align*}
For every $\varepsilon>0$, define
\[
    \operatorname{gap}_x(\varepsilon)
    :=
    \inf_{\substack{M\in\gM_0:\ \|M-M_x^\star\|_F\ge\varepsilon}}
    \{\gJ_x(M)-\gJ_x(M_x^\star)\}.
\]
Assumption~\ref{ass:imv_metric_learning}(iii) gives $\operatorname{gap}_x(\varepsilon)>0$, and therefore
\begin{align*}
    \{\|\widehat M_x-M_x^\star\|_F\ge\varepsilon\}
    &\subseteq
    \{\operatorname{gap}_x(\varepsilon)\le2D_{N,x}\}.
\end{align*}
Consequently,
\begin{align*}
    \sP\{\|\widehat M_x-M_x^\star\|_F\ge\varepsilon\}
    &\le
    \sP\{2D_{N,x}\ge\operatorname{gap}_x(\varepsilon)\}
    \longrightarrow0.
\end{align*}
\end{proof}

\subsection{Proof of Theorem~\ref{thm:itacp_conditional_consistency}}
\label{app:proof_itacp_conditional_consistency}

\begin{proof}
\textbf{Step 1: Control the average radius error.}
Let
\[
    R_i^\star=R^\star(X_i,Y_i),
    \qquad
    \widehat R_i=\widehat R_N(X_i,Y_i),
    \qquad i\in[n].
\]
Define the average radius error
\[
    A_{n,N}
    :=\frac1n\sum_{i=1}^n|\widehat R_i-R_i^\star|.
\]
By Assumption~\ref{ass:imv_radius_nuisance}, for almost every $x$,
$\mathcal E_N^{\rm rad}(x)\to0$ in probability and
$0\le\mathcal E_N^{\rm rad}(x)\le B_{\rm rad}(x)$. Bounded convergence in probability, followed by dominated convergence over $P_X$, therefore gives
\[
    \E\left\{
        \int \mathcal E_N^{\rm rad}(x)\,dP_X(x)
    \right\}
    \longrightarrow0.
\]
Conditional on $\gH_N$, the calibration observations are i.i.d. and
\begin{align*}
    \E(A_{n,N}\mid\gH_N)
    &=\frac1n\sum_{i=1}^n
      \E\left(
        |\widehat R_N(X_i,Y_i)-R^\star(X_i,Y_i)|
        \mid\gH_N
      \right)\\
    &=\int \mathcal E_N^{\rm rad}(x)\,dP_X(x).
\end{align*}
Consequently, for every $\varepsilon>0$, Markov's inequality gives
\[
    \sP(A_{n,N}>\varepsilon)
    \le
    \frac{1}{\varepsilon}
    \E\left\{
        \int \mathcal E_N^{\rm rad}(x)\,dP_X(x)
    \right\}
    \longrightarrow0,
\]
and hence $A_{n,N}=o_{\sP}(1)$.

\textbf{Step 2: Compare the fitted and oracle empirical CDFs.}
For $t\in\R$, define the two empirical distribution functions
\[
    F_n^\star(t)
    :=\frac1n\sum_{i=1}^n\one\{R_i^\star\le t\},
    \qquad
    \widehat F_{n,N}(t)
    :=\frac1n\sum_{i=1}^n\one\{\widehat R_i\le t\}.
\]
For $\eta>0$, also define the fraction of radius errors exceeding $\eta$:
\[
    D_{n,N}(\eta)
    :=\frac1n\sum_{i=1}^n
      \one\{|\widehat R_i-R_i^\star|>\eta\}.
\]
Since $\eta\one\{|\widehat R_i-R_i^\star|>\eta\}\le|\widehat R_i-R_i^\star|$, we have
\[
    D_{n,N}(\eta)\le\frac{A_{n,N}}{\eta}=o_{\sP}(1)
\]
for every fixed $\eta>0$.  Moreover, pointwise in the calibration sample,
\begin{align*}
    &\one\{R_i^\star\le t-\eta\}
    -\one\{|\widehat R_i-R_i^\star|>\eta\}
    \le\one\{\widehat R_i\le t\},\\
    &\one\{\widehat R_i\le t\}
    \le
    \one\{R_i^\star\le t+\eta\}
    +\one\{|\widehat R_i-R_i^\star|>\eta\}.
\end{align*}
After averaging over $i$,
\begin{align*}
    F_n^\star(t-\eta)-D_{n,N}(\eta)
    \le\widehat F_{n,N}(t)
    \le F_n^\star(t+\eta)+D_{n,N}(\eta).
\end{align*}

\textbf{Step 3: Control the oracle empirical CDF.}
For $P_X$-almost every $X_i$, the oracle transport satisfies $T_{M_{X_i}^\star,X_i}\{s(X_i,Y_i)\}\mid X_i\sim\gQ$. Consequently,
\begin{align*}
    \sP(R_i^\star\le t\mid X_i,\gH_N)
    &=\sP_{U\sim\gQ}(\|U\|\le t)
    =F_Q(t),\\
    \sP(R_i^\star\le t\mid\gH_N)
    &=F_Q(t).
\end{align*}
\begin{sloppypar}
Thus, conditionally on $\gH_N$, $R_1^\star,\ldots,R_n^\star$ are i.i.d. with distribution function $F_Q$. The Dvoretzky--\allowbreak Kiefer--\allowbreak Wolfowitz inequality gives, for every $z>0$,
\end{sloppypar}
\[
    \sP\left\{
        \sup_{t\in\R}|F_n^\star(t)-F_Q(t)|>z
        \;\middle|\;\gH_N
    \right\}
    \le2e^{-2nz^2}.
\]
Therefore,
\[
    \sup_{t\in\R}|F_n^\star(t)-F_Q(t)|=o_{\sP}(1).
\]

\textbf{Step 4: Prove consistency of the calibration cutoff.}
Let $ k_n:=\left\lceil(n+1)(1-\alpha)\right\rceil$ with $\tau_n:=\frac{k_n}{n}$. For all sufficiently large $n$, $k_n\le n$, and $\tau_n\longrightarrow1-\alpha$.
Let ${\widehat q_{n,N}}$ denote the empirical cutoff in \eqref{eq:itacp_set}. By the $+\infty$ convention,
\begin{align*}
    {\widehat q_{n,N}}
    &=\inf\{t:\widehat F_{n,N}(t)\ge\tau_n\}.
\end{align*}
Fix $\varepsilon>0$ and put $\eta=\varepsilon/2$.  Strict increase of $F_Q$ at $q_\star$ gives
\begin{align*}
    g_-(\varepsilon)
    &:=(1-\alpha)-F_Q(q_\star-\varepsilon/2)>0,\\
    g_+(\varepsilon)
    &:=F_Q(q_\star+\varepsilon/2)-(1-\alpha)>0.
\end{align*}
The empirical-CDF sandwich at $t=q_\star-\varepsilon$ gives
\begin{align*}
    \widehat F_{n,N}(q_\star-\varepsilon)
    &\le
    F_n^\star(q_\star-\varepsilon/2)
    +D_{n,N}(\varepsilon/2)\\
    &\le
    F_Q(q_\star-\varepsilon/2)
    +\sup_t|F_n^\star(t)-F_Q(t)|
    +D_{n,N}(\varepsilon/2)\\
    &=1-\alpha-g_-(\varepsilon)+o_{\sP}(1)
    <\tau_n
\end{align*}
with probability tending to one.  Similarly,
\begin{align*}
    \widehat F_{n,N}(q_\star+\varepsilon)
    &\ge
    F_n^\star(q_\star+\varepsilon/2)
    -D_{n,N}(\varepsilon/2)\\
    &\ge
    F_Q(q_\star+\varepsilon/2)
    -\sup_t|F_n^\star(t)-F_Q(t)|
    -D_{n,N}(\varepsilon/2)\\
    &=1-\alpha+g_+(\varepsilon)+o_{\sP}(1)
    >\tau_n
\end{align*}
with probability tending to one.  Hence
\begin{align*}
    \sP\left\{
        q_\star-\varepsilon
        <{\widehat q_{n,N}}
        \le q_\star+\varepsilon
    \right\}
    \longrightarrow1,
\end{align*}
which proves
\[
    {\widehat q_{n,N}}
    \overset{\sP}{\longrightarrow}q_\star.
\]

\textbf{Step 5: Prove conditional set convergence.}
Now fix an interior $x$ satisfying the pointwise conditions and condition on $X_{n+1}=x$. The test response $Y_{n+1}$ is independent of the calibration sample and $\gH_N$ under this conditioning. Abbreviate
\[
    R_x^\star:=R^\star(x,Y_{n+1}),
    \qquad
    \widehat R_{N,x}:=\widehat R_N(x,Y_{n+1}).
\]
For every $\varepsilon>0$, the following pointwise indicator inequality holds:
\begin{align*}
    &\left|
        \one\{\widehat R_{N,x}\le{\widehat q_{n,N}}\}
        -\one\{R_x^\star\le q_\star\}
    \right|\\
    &\quad\le
    \one\{|R_x^\star-q_\star|\le2\varepsilon\}
    +\one\{|\widehat R_{N,x}-R_x^\star|>\varepsilon\}
    +\one\{|{\widehat q_{n,N}}-q_\star|>\varepsilon\}.
\end{align*}
Indeed, on the complement of the three events, if $R_x^\star<q_\star-2\varepsilon$, then
\[
    \widehat R_{N,x}
    \le R_x^\star+\varepsilon
    <q_\star-\varepsilon
    \le{\widehat q_{n,N}},
\]
whereas if $R_x^\star>q_\star+2\varepsilon$, then
\[
    \widehat R_{N,x}
    \ge R_x^\star-\varepsilon
    >q_\star+\varepsilon
    \ge{\widehat q_{n,N}}.
\]
Taking expectations in the indicator inequality conditional on $X=x$ gives
\begin{align*}
    &\sP\left\{
        Y_{n+1}\in\gC_\star^{\rm IMV}(x)
        \mathbin\triangle
        {\widehat\gC}^{\rm IMV}(x)
        \;\middle|\;X_{n+1}=x
    \right\}\\
    &\quad\le
    F_Q(q_\star+2\varepsilon)
    -F_Q((q_\star-2\varepsilon)-)\\
    &\qquad+
    \E\left[
        \min\left\{1,\frac{\mathcal E_N^{\rm rad}(x)}{\varepsilon}\right\}
    \right]
    +\sP(|{\widehat q_{n,N}}-q_\star|>\varepsilon).
\end{align*}
Here conditional Markov's inequality gives
\[
    \sP\left(
        |\widehat R_{N,x}-R_x^\star|>\varepsilon
        \mid\gH_N,X_{n+1}=x
    \right)
    \le\min\left\{1,\frac{\mathcal E_N^{\rm rad}(x)}{\varepsilon}\right\}.
\]
Since $\mathcal E_N^{\rm rad}(x)=o_{\sP}(1)$ by Assumption~\ref{ass:imv_radius_nuisance}, bounded convergence in probability implies
\[
    \E\left[
        \min\left\{1,\frac{\mathcal E_N^{\rm rad}(x)}{\varepsilon}\right\}
    \right]\longrightarrow0.
\]
The quantile term also tends to zero.  Thus
\begin{align*}
    \limsup_{N,n\to\infty}
    &\sP\left\{
        Y_{n+1}\in\gC_\star^{\rm IMV}(x)
        \mathbin\triangle
        {\widehat\gC}^{\rm IMV}(x)
        \;\middle|\;X_{n+1}=x
    \right\}\\
    &\le
    F_Q(q_\star+2\varepsilon)
    -F_Q((q_\star-2\varepsilon)-).
\end{align*}
Letting $\varepsilon\downarrow0$ and using continuity of $F_Q$ at $q_\star$ proves the symmetric-difference convergence.

\textbf{Step 6: Deduce asymptotic conditional coverage.}
Finally,
\[
    \sP\{Y_{n+1}\in\gC_\star^{\rm IMV}(x)\mid X_{n+1}=x\}
    =\sP(R_x^\star\le q_\star\mid X_{n+1}=x)
    =F_Q(q_\star)=1-\alpha.
\]
For any two events $A$ and $B$,
\[
    |\sP(A)-\sP(B)|\le\sP(A\mathbin\triangle B).
\]
Applying this inequality conditionally on $X_{n+1}=x$ gives
\begin{align*}
    &\left|
        \sP\{Y_{n+1}\in{\widehat\gC}^{\rm IMV}(x)\mid X_{n+1}=x\}
        -(1-\alpha)
    \right|\\
    &\quad\le
    \sP\left\{
        Y_{n+1}\in\gC_\star^{\rm IMV}(x)
        \mathbin\triangle
        {\widehat\gC}^{\rm IMV}(x)
        \;\middle|\;X_{n+1}=x
    \right\}
    \longrightarrow0.
\end{align*}
\end{proof}

\subsection{Proof of Proposition~\ref{prop:lotcp_polyhedra}}
\label{app:proof_lotcp_polyhedra}

\begin{proof}
The feasible set $\Pi_{\tilde x}$ is a bounded transportation polytope, and its objective is linear in $\pi$. Hence an optimal plan can be selected from the finite collection of basic feasible plans indexed by $\mathfrak B_{\tilde x}$. If the trial score is $s$, the objective value at the plan $\pi^b$ is
\begin{align*}
    L_{b,M}(s)
    &=
    \sum_{i=1}^{n}\sum_{j=1}^{n+1}
        \pi^b_{ij}c_M(S_i,U_j)
    +\sum_{j=1}^{n+1}\pi^b_{n+1,j}c_M(s,U_j)\\
    &=p_{n+1}s^\top Ms-2s^\top M m_b+\kappa_{b,M},
\end{align*}
where the last equality uses $\sum_j\pi^b_{n+1,j}=p_{n+1}$. Consequently, $L_{b,M}(s)\le L_{b',M}(s)$ if and only if
\begin{align*}
    2s^\top M(m_{b'}-m_b)
    \le \kappa_{b',M}-\kappa_{b,M}.
\end{align*}
The common quadratic term has canceled. Intersecting these halfspaces over $b'\in\mathfrak B_{\tilde x}$ gives precisely $\gR_{b,M}(\tilde x)$, so its closure is a convex polyhedron.

Within the relative interior of this cell, the selected plan is $\pi^b$. Therefore $J_b(V)$ and the randomized acceptance decision are fixed throughout the cell. Likewise, all barycenters $\overline U_i^b$, their weighted quantile $q_b$, and the barycentric acceptance decision are fixed throughout the cell. Taking the union over the accepting bases gives the two displayed representations. On a hyperplane where several bases are optimal, a deterministic tie-breaking rule assigns the boundary to the selected cell. Translation by $\widehat\mu(X_{n+1})$ then gives the response-space representation.
\end{proof}

\subsection{Proof of Theorem~\ref{thm:weighted_lotcp_validity}}
\label{app:proof_weighted_lotcp_validity}

\begin{proof}
Set $y=Y_{n+1}$, $Z_i=(X_i,Y_i)$, and let ${\gD_{\rm aug}}=\{Z_1,\ldots,Z_{n+1}\}$ be the unordered augmented multiset.
Choose an arbitrary ordering of ${\gD_{\rm aug}}$, and let $I\in[n+1]$ denote the index of the occurrence in this ordering corresponding to the original test observation $Z_{n+1}$. Let $\mathcal G=\sigma\!({\gD_{\rm aug}},\widetilde X_{n+1},U_{1:n+1})$ be the $\sigma$-algebra generated by the data. Conditional on $\widetilde X_{n+1}=\tilde x$, exchangeability and Bayes' rule give
\begin{align*}
    \sP(I=i\mid\mathcal G)
    &=\frac{H_h(\tilde x\mid x_i)}
            {\sum_{k=1}^{n+1}H_h(\tilde x\mid x_k)}
      =p_i(\tilde x).
\end{align*}
By permutation equivariance, the fitted scores and plan are fixed in this ordering. Write $p_i=p_i(\tilde x)$ and $\widehat\pi_{ij}$ for the plan at the true response which satisfies
\[
    \sum_j\widehat\pi_{ij}=p_i,
    \qquad
    \sum_i\widehat\pi_{ij}=\frac{1}{n+1}.
\]

For the randomized rule, only rows with $p_i>0$ can be selected. The common uniform draw used for inverse-cdf sampling yields
\begin{align*}
    \sP(J=j\mid I=i,\mathcal G)
    =\frac{\widehat\pi_{ij}}{p_i},\qquad p_i>0.
\end{align*}
Hence,
\begin{align*}
    \sP(J=j\mid\mathcal G)
    =\sum_{i:p_i>0}p_i\frac{\widehat\pi_{ij}}{p_i}
     =\sum_i\widehat\pi_{ij}
     =\frac{1}{n+1}.
\end{align*}
With $\widehat q_U=\widehat q_{1-\alpha}
\LRs{\{\|U_j\|\}_{j=1}^{n+1}}$,
\begin{align*}
    \sP\{\|U_J\|\le\widehat q_U\mid\mathcal G\}
    &=\sum_{j=1}^{n+1}\sP(J=j\mid\mathcal G)
       \one\{\|U_j\|\le\widehat q_U\}\\
    &=\frac{1}{n+1}\sum_{j=1}^{n+1}
       \one\{\|U_j\|\le\widehat q_U\}
     \ge1-\alpha.
\end{align*}

For the barycentric rule, let $R_i=R_{\tilde x,i}^{Y_{n+1}}$ for $p_i>0$, and let $\widehat q=\widehat q_{1-\alpha}^{\,Y_{n+1}}(\tilde x)$. Then
\begin{align*}
    \sP\{R_I\le\widehat q\mid\mathcal G\}
    &=\sum_{i:p_i>0}\sP(I=i\mid\mathcal G)
       \one\{R_i\le\widehat q\}\\
    &=\sum_{i:p_i>0}p_i\one\{R_i\le\widehat q\}
     \ge1-\alpha.
\end{align*}
Taking expectations over $\mathcal G$ proves both
claims.
\end{proof}

\subsection{Proof of Theorem~\ref{thm:weighted_itacp_validity}}
\label{app:proof_weighted_itacp_validity}

\begin{proof}

$\widehat q_U=\widehat q_{1-\alpha}
\LRs{\{\|U_j\|\}_{j=1}^{n+1}}$.
Since ${\gD_{\rm est}}$ is independent of the augmented sample,
conditioning additionally on ${\gD_{\rm est}}$ does not change the
conditional distribution of the test index. Hence
$$
    \sP(I=i\mid\mathcal G,{\gD_{\rm est}})
    =
    \sP(I=i\mid\mathcal G)
    =
    p_i(\tilde x).
$$
Moreover, conditional on $(\mathcal G,{\gD_{\rm est}})$, $\widehat M(\tilde x)$ and the corresponding transport plan $\widehat\pi$ are fixed. By permutation equivariance of the construction, the arbitrary ordering of the augmented observations is immaterial, and the transport plan satisfies
\begin{align*}
    \sum_j\widehat\pi_{ij}&=p_i(\tilde x),
    \qquad
    \sum_i\widehat\pi_{ij}=\frac{1}{n+1}.
\end{align*}
For the randomized metric rule,
\begin{align*}
    \sP(J^{\widehat M}=j\mid\mathcal G,{\gD_{\rm est}})
    &=\sum_{i:p_i>0}p_i\frac{\widehat\pi_{ij}}{p_i}
      =\frac{1}{n+1},\\
    \sP\{\|U_{J^{\widehat M}}\|\le\widehat q_U
       \mid\mathcal G,{\gD_{\rm est}}\}
    &=\frac{1}{n+1}\sum_{j=1}^{n+1}
       \one\{\|U_j\|\le\widehat q_U\}
      \ge1-\alpha.
\end{align*}
For the barycentric metric rule, write
$R_i^{\widehat M}=R_{\tilde x,i}^{Y_{n+1},\widehat M}$ and
$\widehat q^{\widehat M}
=\widehat q_{1-\alpha}^{\,Y_{n+1},\widehat M}(\tilde x)$. Then
\begin{align*}
    \sP\{R_I^{\widehat M}\le\widehat q^{\widehat M}
       \mid\mathcal G,{\gD_{\rm est}}\}
    &=\sum_{i:p_i>0}p_i\one\{R_i^{\widehat M}
                       \le\widehat q^{\widehat M}\}
      \ge1-\alpha.
\end{align*}
Taking expectations over $(\mathcal G,{\gD_{\rm est}})$ proves both
claims.
\end{proof}

\end{document}